\documentclass[letterpaper]{article} 
\usepackage{aaai23}  
\usepackage{times}  
\usepackage{helvet}  
\usepackage{courier}  
\usepackage[hyphens]{url}  
\usepackage{graphicx} 
\usepackage{natbib}  
\usepackage{caption} 
\usepackage{algorithm}
\usepackage{algorithmic}
\usepackage{amsmath}
\usepackage{amsthm}

\newtheorem{theorem}{Theorem}[section]

\newtheorem{lemma}[theorem]{Lemma}

\usepackage{xcolor}

\usepackage{newfloat}
\usepackage{listings}
\DeclareCaptionStyle{ruled}{labelfont=normalfont,labelsep=colon,strut=off} 
\floatstyle{ruled}
\newfloat{listing}{tb}{lst}{}
\floatname{listing}{Listing}
\title{A Risk-Sensitive Approach to Policy Optimization}

\author{
    Jared Markowitz\textsuperscript{\rm 1},
    Ryan W. Gardner\textsuperscript{\rm 1},
    Ashley Llorens\textsuperscript{\rm 2},
    Raman Arora\textsuperscript{\rm 3},
    I-Jeng Wang\textsuperscript{\rm 1}
}

\affiliations{
    \textsuperscript{\rm 1}Johns Hopkins University Applied Physics Laboratory\\
    \textsuperscript{\rm 2} Microsoft Corporation\\
    \textsuperscript{\rm 3}Johns Hopkins University\\
    Jared.Markowitz@jhuapl.edu

}

\begin{document}

\maketitle

\begin{abstract}
Standard deep reinforcement learning (DRL) aims to maximize expected reward, considering collected experiences equally in formulating a policy. This differs from human decision-making, where gains and losses are valued differently and outlying outcomes are given increased consideration.  It also fails to capitalize on opportunities to improve safety and/or performance through the incorporation of distributional context. Several approaches to distributional DRL have been investigated, with one popular strategy being to evaluate the projected distribution of returns for possible actions.  We propose a more direct approach whereby risk-sensitive objectives, specified in terms of the cumulative distribution function (CDF) of the distribution of full-episode rewards, are optimized. This approach allows for outcomes to be weighed based on relative quality, can be used for both continuous and discrete action spaces, and may naturally be applied in both constrained and unconstrained settings.  We show how to compute an asymptotically consistent estimate of the policy gradient for a broad class of risk-sensitive objectives via sampling, subsequently incorporating variance reduction and regularization measures to facilitate effective on-policy learning.  We then demonstrate that the use of moderately ``pessimistic'' risk profiles, which emphasize scenarios where the agent performs poorly, leads to enhanced exploration and a continual focus on addressing deficiencies.  We test the approach using different risk profiles in six OpenAI Safety Gym environments, comparing to state of the art on-policy methods.  Without cost constraints, we find that pessimistic risk profiles can be used to reduce cost while improving total reward accumulation.  With cost constraints, they are seen to provide higher positive rewards than risk-neutral approaches at the prescribed allowable cost.

\end{abstract}

\section{Introduction}
While deep reinforcement learning (DRL) has been used to master an impressive array of simulated tasks in controlled settings, it has not yet been widely adopted for high-stakes, real-world applications. One reason for this gap is its lack of safety assurances.  Endowing artificial agents with a distributional perspective, potentially used in conjunction with cost constraints, should make their decision-making more robust.  This could in turn lead to increased trust from humans and increased real-world adoption.
 
In reinforcement learning (RL), \emph{risk} arises due to uncertainty around the possible outcomes of an agent's actions.  It is a result of randomness in the operating environment, mismatch between training and test conditions, and the stochasticity of the policy. \emph{Risk-sensitive} policies, or those that consider a quantity other than the mean reward over the distribution of possible outcomes, offer the potential for added robustness under uncertain and dynamic conditions. There is an evolving landscape of algorithmic paradigms for handling risk in RL, including distributional methods \citep{BeDaMu17,DaOsSiMu18,DaRoBeMu18, d4pg, fei21} and constraint-based approaches adapted from optimal control~\citep{BHATNAGAR2010760, AcHeTaAb17, Chow2019, RaAcAm19, TeMaMa19, ZhFaYaWa20, zhang20}.  Within this landscape, learning approaches that optimize distributional measures offer the ability to express design preferences over the full distribution of potential outcomes.  Constraint-based approaches allow the level of \emph{average} cost incurred by an agent to be adjusted, typically through the use of dual methods.

In the following, we introduce a novel method for estimating the policy gradient of a broad class of risk-sensitive objectives, applicable in both the unconstrained and constrained settings.  The approach allows agents to be trained with different risk profiles, based on full episode outcomes.  It can be used to mimic general human decision-making \citep{TvKa92}, but is found to be most effective when implementing one particular human learning strategy: emphasizing improvement on tasks where one is deficient.  We elucidate the mechanisms behind performance gains associated with this strategy and evaluate its effectiveness for several stochastic continuous control problems.

\section{Related Work}\label{related}

The presence of stochasticity in Markov Decision Processes (MDPs) can lead to variability in agent outcomes.  Randomness can come from different sources- for instance the initial state of the environment, noise in transition dynamics, and sampling from a stochastic policy.  Distributional RL methods allow for consideration of outcome variability in formulating policies, and have primarily been explored from a value-based perspective.  For example, Q-value distributions have been explicitly modeled through categorical techniques \citep{BeDaMu17} and quantile regression \citep{DaOsSiMu18}, leading to improved value predictions and overall performance. Recent works utilize distribution modeling in the actor-critic setting to enable application to continuous action spaces, again demonstrating improved performance over baseline approaches \citep{d4pg, MaXiZh20, ZhBeSiWaKo21, DuGuLiReSuCh21}. In value-based approaches, risk-sensitivity criteria are applied at run time as a nonlinear warping of the estimated Q-value distribution.

Policy optimization with a distributional objective offers additional promise for risk-sensitive RL. Some existing methods are limited to a specific class of learning objective, such as the set of concave risk measures that permit a globally-optimal solution \citep{ZhFaYaWa20, taCoherent2015}.  Others allow a broader class of measures but are more restrictive in the class of policies that can be represented \citep{LAJiFuMaSz16, UMD2018, prahsanth_fu_book, jaimungal22}. Our contribution is a risk-sensitive policy gradient approach that offers both significant flexibility in the choice of learning objective and the ability to learn policies parameterized by a deep neural network.  The unconstrained version of the algorithm resembles Proximal Policy Optimization (PPO; \citet{ScWoDhRaKl17}) and is similarly widely applicable.


Various measures have been considered in the context of risk-sensitive RL, including exponential utility \citep{Pr64}, percentile performance criteria \citep{WuLi99}, value-at-risk \citep{Le45}, conditional value-at-risk \citep{RoUr00}, and prospect theory \citep{KaTv79}. In this work, we consider a class of risk-sensitivity measures motivated by Cumulative Prospect Theory (CPT) \citep{TvKa92}.  CPT uniquely models two key aspects of human decision-making: (1) a utility function $u$, computed relative to a reference point and inducing more risk-averse behavior in the presence of gains than losses as well as (2) a weight function $w$ that prioritizes outlying events. Specific forms of $u$ and $w$ are given in \cite{TvKa92}; while we evaluate these specific choices we also consider the much broader class of measures possible with different choices.  In this work, we typically take $u$ to be the reward provided by the environment and evaluate the effect of adjusting $w$.

Constrained reinforcement learning addresses safety concerns \cite{JMLR:garcia15a} explicitly via methods including Lagrangian constraints \cite{BHATNAGAR2010760,RaAcAm19, TeMaMa19, zhang20, UMD2018, PaSaChLuCaRi2019} and constraint coefficients \cite{AcHeTaAb17}. In this work, we pair our risk-sensitive policy gradient estimate with Reward Constrained Policy Optimization (RCPO; \citet{TeMaMa19}) in order to achieve higher accumulation of non-cost rewards than possible with a risk-neutral objective.


\section{Risk-Sensitive Policy Optimization}\label{math}

In this section we formalize the class of distributional objectives to be considered, derive a sampling-based approximation of its policy gradient, enact variance reduction and regularization on this estimate, and use the result to produce practical learning algorithms for both the unconstrained and constrained settings.

\subsection{Preliminaries: Problem and Notation}\label{prelim}
Standard deep reinforcement learning seeks to maximize the expected reward of an agent acting in an MDP. That is, it maximizes the objective
\begin{equation}
J(\theta) =  E_{\tau \sim p_\theta(\tau)} \bigg[ \sum_t r(\mathbf{s}_t, \mathbf{a}_t) \bigg].
\label{rl}
\end{equation}
Here $p_\theta(\tau)$ is the distribution over trajectories $\tau \equiv \mathbf{s}_1, \mathbf{a}_1, \ldots,  \mathbf{s}_T, \mathbf{a}_T $ induced by a policy parameterized by $\theta$;  $\mathbf{s}_t$, $\mathbf{a}_t$, and $r(\mathbf{s}_t, \mathbf{a}_t)$ denote the state, action, and reward at time $t$, respectively. To 
allow a mapping from reward to utility and outcomes to be weighed based on their relative quality, we instead consider the risk-sensitive objective
\begin{equation}
J_{rs}(\theta) =  \int_{-\infty}^{+\infty} u(r(\tau)) \frac{d}{dr(\tau)}\bigg(w(P_{\theta}(r(\tau))\bigg)dr(\tau),
\label{cdf-rl}
\end{equation}
where $u(r(\tau))$ is the utility associated with full-trajectory reward $r(\tau) \equiv \sum_tr(\mathbf{s}_{t}, \mathbf{a}_t)$ and $w$ is a piecewise differentiable weighting function of the CDF of $r(
\tau)$; $ P_{\theta}(r(\tau)) = \int_{-\infty}^{r(\tau)}p_{\theta}(r')dr'$. We assume that the temporal allocation of utility-- whether mapped from reward throughout an episode or provided only at the end-- is additionally specified.

Equation \ref{cdf-rl} is inspired by CPT~\citep{TvKa92}, which includes a pair of integrals of this form.  It was chosen for its generality; by using different utility functions $u$ and/or weight functions $w$ one may represent all of the risk measures mentioned above, all of the risk measures evaluated by \citet{DaOsSiMu18}, and many more.  The form (\ref{cdf-rl}) reduces to (\ref{rl}) when $u$ and $w$ are both the identity mapping.  It accommodates ``cutoff'' risk measures (including CVaR) through the use of piecewise weight functions.  While designed for the episodic setting, the objective (\ref{cdf-rl}) may be considered for infinite horizons through the use of appropriately long windows.

\subsection{Risk-Sensitive Policy Gradient}

To optimize the objective (\ref{cdf-rl}), we first derive an approximation to its gradient with respect to the policy parameters $\theta$. Working toward a representation that can be sampled, we assert the independence of the reward on $\theta$ and use the chain rule to write
\begin{equation}
\small
{\nabla_\theta J_{rs}(\theta) = \int_{-\infty}^{\infty} u(r(\tau)) \frac{d}{dr(\tau)}\bigg( w'(P_\theta(r(\tau))) \nabla_\theta P_\theta(r(\tau)) \bigg)dr(\tau),}
 \label{cr}
\end{equation}
where $w'$ is the derivative of $w$ with respect to $P_\theta(r(\tau))$. The gradient of the CDF may be written as:
\begin{equation}
\begin{split}
\nabla_\theta P_\theta(r(\tau)) &= \nabla_\theta \int_{-\infty}^{r(\tau)} p_\theta(r')dr'\\
&= \nabla_\theta \int_{\tau'} H(r(\tau) - r(\tau')) p_\theta(\tau')d\tau'  \\
&= \int_{\tau'} H(r(\tau) - r(\tau')) \nabla_\theta p_\theta(\tau')d\tau' \\
&= \int_{\tau'} H(r(\tau) - r(\tau')) p_\theta(\tau') \nabla_\theta \log p_\theta(\tau') d\tau'. 
\label{grad-cdf1}
\end{split}
\end{equation}
Here we have used the integral representation of $P_\theta(r(\tau))$, the Heaviside step function $H$ to select all trajectories with total reward $\le r(\tau)$, and the independence of reward on $\theta$.  In the following, we also use the complementary expression
\begin{equation}
\begin{split}
&\nabla_\theta P_\theta(r(\tau))  = \nabla_\theta \bigg(1-\int_{r(\tau)}^\infty p_\theta(r')dr'\bigg)\\
&= -\int_{\tau'} H(r(\tau') - r(\tau)) p_\theta(\tau') \nabla_\theta \log p_\theta(\tau') d\tau'. 
\label{grad-cdf2}
\end{split}
\end{equation}
Either form, or a combination of the two, may be substituted into (\ref{cr}) and the result sampled over $N$ trajectories by first ordering trajectories $i = 1 \ldots N$ by increasing reward $r(\tau)$. Implicitly, this assumes that the collected full-episode rewards are representative of the true distribution.  Then
\begin{equation}
\begin{split}
    \nabla_\theta J_{rs}(\theta) 
    \approx \sum_{i=1}^N u(r(\tau_i))
    \bigg(&w'\bigg(\frac{i}{N}\bigg)\nabla_\theta P_\theta(r(\tau_i)) \\
    - &w'\bigg(\frac{i-1}{N}\bigg)\nabla_\theta P_\theta(r(\tau_{i-1})) \bigg),
    \label{sample1}
\end{split}
\end{equation}
where we have discretized $d/dr(\tau)$ and the term $w'(0)\nabla_\theta P_\theta(r(\tau_0)) \equiv 0$. Such ordering produces an asymptotically consistent estimate of the CPT value \cite{LAJiFuMaSz16}. $\nabla_\theta P_\theta(r(\tau_i))$ may be sampled in one of two ways, based on either (\ref{grad-cdf1}) or (\ref{grad-cdf2}):
\begin{equation}
\begin{split}
\nabla_\theta P_\theta(r(\tau_i))
&\approx \frac{1}{N}\sum_{j=1}^i\sum_{t=1}^{T_j}\nabla_\theta \log \pi_\theta(\mathbf{a}_{j,t}|\mathbf{s}_{j,t}) \\
&\approx -\frac{1}{N}\sum_{j=i+1}^N\sum_{t=1}^{T_j}\nabla_\theta \log \pi_\theta(\mathbf{a}_{j,t}|\mathbf{s}_{j,t}).\label{sample2}
\end{split}
\end{equation}
The expression (\ref{sample1}) may be used to train a policy that optimizes the distributional objective (\ref{cdf-rl}) in a manner similar to REINFORCE ~\citep{Wi92}.  

\subsection{Variance Reduction and Regularization}\label{var_red_main}

Reducing the variance of sample-based gradient estimates enables faster learning.  Here we take several steps to reduce the variance of (\ref{sample1}), similar to what has been done with the policy gradient estimate of REINFORCE ~\cite{Wi92}. First, note that cross-trajectory terms of the form  $f(\tau_i, \mathbf{a}_{j,t},\mathbf{s}_{j,t}) \equiv u(r(\tau_i))\nabla_\theta\log\pi_\theta(\mathbf{a}_{j,t}|\mathbf{s}_{j,t})$, while nonzero, do not contribute to the gradient estimate in expectation when $i \ne j$.  A proof of this assertion (relevant to our approach but not REINFORCE), is given in Appendix A.1.  Using (\ref{grad-cdf1}) for the first term of (\ref{sample1}) and (\ref{grad-cdf2}) for the second allows us to write
\begin{equation}
    \begin{split}
    \nabla_\theta J_{rs}(\theta) \approx 
    \sum_{i=1}^N &u(r(\tau_i))\:*\\
    \bigg( &w'\bigg(\frac{i}{N}\bigg)\frac{1}{N} \sum_{j=1}^i\sum_{t=1}^{T_j}\nabla_\theta \log \pi_\theta(\mathbf{a}_{j,t}|\mathbf{s}_{j,t})\\
    + &w'\bigg(\frac{i-1}{N}\bigg) \frac{1}{N} \sum_{j=i}^N \sum_{t=1}^{T_j}\nabla_\theta \log \pi_\theta(\mathbf{a}_{j,t}|\mathbf{s}_{j,t})\bigg).
\end{split}\label{eq10}
\end{equation}
Removing cross-trajectory terms gives
\begin{equation}
\begin{split}
&\nabla_\theta J_{rs}(\theta) \approx \frac{1}{N} \sum_{i=1}^N u(r(\tau_i))\:* \\
&\bigg( w'\bigg(\frac{i}{N}\bigg) + w'\bigg(\frac{i-1}{N}\bigg)\bigg)
    \sum_{t=1}^{T_i}\nabla_\theta \log \pi_\theta(\mathbf{a}_{i,t}|\mathbf{s}_{i,t}).
\label{gradJ}
\end{split}
\end{equation}
  Note that the weight coefficients $(w'(\frac{i}{N}) + w'(\frac{i-1}{N}))$ should be normalized over each batch. The expression (\ref{gradJ}) is equal to (\ref{sample1}) in expectation, but with reduced variance (justification in Appendix A.1).  It has a clear intuition -- trajectories are assigned utilities based on their rewards and their contributions to the policy gradient are scaled by the derivative of the weight function, just as they are in CPT \citep{TvKa92}.  

Standard variance reduction techniques may be applied to this simplified form. Without further assumption or introduction of additional bias, a static baseline $b$ may be employed:
\begin{equation}
\begin{split}
&\nabla_\theta J_{rs}(\theta) \approx \frac{1}{N} \sum_{i=1}^N \bigg(u(r(\tau_i)) - b \bigg)\:* \\
&\bigg(w'\bigg(\frac{i}{N}\bigg) +  w'\bigg(\frac{i-1}{N} \bigg)\bigg) \sum_{t=1}^{T_i}\nabla_\theta \log \pi_\theta(\mathbf{a}_{i,t}| \mathbf{s}_{i,t}).
\label{grad_baseline}
\end{split}
\end{equation}
Justification for this assertion is given in Appendix A.2.  Learning may be further expedited by considering utilities on a per-step basis.  Given our assumptions that utility is a function of only reward and that its temporal allocation is given with the objective, we may further reduce the variance of (\ref{gradJ}) through the incorporation of utility-to-go and a state-dependent baseline $V_\phi(\mathbf{s}_{i,t})$:
\begin{equation}
\begin{split}
&\nabla_\theta J_{rs}(\theta) \approx \frac{1}{N} \sum_{i=1}^N \bigg[w'\bigg(\frac{i}{N}\bigg) + w'\bigg(\frac{i-1}{N}\bigg)\bigg]\:* \\
&\sum_{t=1}^{T_i}\nabla_\theta \log \pi_\theta(\mathbf{a}_{i,t}| \mathbf{s}_{i,t})\bigg[ \sum_{t'=t}^{T_i} u(\mathbf{s}_{i, t'}, \mathbf{a}_{i, t'}) - V_\phi(\mathbf{s}_{i, t}) \bigg]
\label{rtg}
\end{split}
\end{equation}
Here $u(\mathbf{s}_{i, t'}, \mathbf{a}_{i, t'})$ is the per-step utility. The value function $V_\phi(\mathbf{s}_{i,t})$ is parameterized by $\phi$ and may be trained via regression to minimize
\begin{equation}
    \mathcal{L}(\phi) = \sum_{i, t} \bigg( V_\phi(\mathbf{s}_{i, t}) - \sum_{t'=t}^{T_i}u(\mathbf{s}_{i, t'}, \mathbf{a}_{i, t'})  \bigg)^2.
\end{equation}
A standard argument, similar to the approach taken in \citep{spinning_up}, can be used to show that the incorporation of utility-to-go does not change the expected value of (\ref{gradJ}).  The use of a state-dependent baseline also does not introduce additional bias (see Appendix A.2).

Finally, discount factors, bootstrapping, and trust regions may be used to provide additional variance reduction and regularization (see Appendix A.2).  These measures may introduce additional bias to the policy gradient estimate, but typically lead to more sample-efficient learning. In our experiments, we evaluate the use of generalized advantage estimation (GAE; \cite{ScMoLeJoAb16}) based on utility-to-go as well as clipping-based regularization similar to Proximal Policy Optimization \cite{ScWoDhRaKl17}.  Incorporating these in our policy gradient estimate yields
\begin{equation}
\small
\begin{split}
\nabla_\theta J_{rs}(\theta) \approx \frac{1}{N} &\sum_{i=1}^N  \bigg(w'\bigg(\frac{i}{N}\bigg) +w'\bigg(\frac{i-1}{N}\bigg) \bigg)\:* \\
&\sum_{t=1}^{T_i}\nabla_\theta L_\text{clip}\bigg(\log \pi_\theta(\mathbf{a}_{i,t}| \mathbf{s}_{i,t}), A_u^\pi(\mathbf{s}_{i,t}, \mathbf{a}_{i,t})\bigg),
\label{tr}
\end{split}
\end{equation}
where $A_u^\pi(\mathbf{s}_{i,t}, \mathbf{a}_{i,t})$ is the standard GAE except with per-step utilities in place of rewards.  Trust regions are implemented similarly to PPO, pessimistically clipping policy updates to be within a multiplicative factor of $1 \pm \epsilon$ of the existing policy:
\begin{equation}
\small
\begin{split}
&L_\text{clip} = \min\bigg(\log \pi_\theta(\mathbf{a}_{i,t}| \mathbf{s}_{i,t})A_u^\pi(\mathbf{s}_{i,t}, \mathbf{a}_{i,t}),\\ 
&\log\bigg(\text{clip} \bigg(\frac{\pi_\theta(\mathbf{a}_{i,t}| \mathbf{s}_{i,t})}{\pi_{\theta_\text{old}}(\mathbf{a}_{i,t}| \mathbf{s}_{i,t})}, 1 \pm \epsilon\bigg)\pi_{\theta_\text{old}}(\mathbf{a}_{i,t}| \mathbf{s}_{i,t})\bigg)A_u^\pi(\mathbf{s}_{i,t}, \mathbf{a}_{i,t})\bigg).
\label{clip}
\end{split}
\end{equation}
The form of this clipping differs slightly from that of PPO, due to the difference between our policy gradient and the gradient of the objective used by PPO.  In practice, we found our form to consistently perform better (compare blue and green traces in Figures 3 and 4 and see \citet{copg} for further discussion).  As in PPO, our clipping can be used to perform multiple policy updates with the same batch of data, significantly improving sample efficiency.  When following this route, we apply early stopping based on the Kullback-Leibler divergence ($D_{\text{KL}}$) between old and new policies, as in \cite{RaAcAm19}. 

Finally we note that, in the case of policy distributions with infinite support, the ``clipped action policy gradient'' correction of \cite{pmlr-v80-fujita18a} should be used to properly handle finite control bounds.  This was done for all methods (baselines included) in our experiments.

\subsection{Application in Constrained Settings}
The policy gradient estimate derived above may also be used to maximize a risk-sensitive objective subject to a constraint.  Constrained Markov Decision Processes (CMDPs) have positive rewards $r(\mathbf{s}, \mathbf{a})$ and costs $c(\mathbf{s}, \mathbf{a})$ defined for each time step, as well as an overall constraint $C(\tau) = F(c(\mathbf{s}_1, \mathbf{a}_1), \ldots, c(\mathbf{s}_T, \mathbf{a}_T))$ defined over the whole trajectory.  The associated learning problem is to find
\begin{equation}
    \max_\theta J_R(\theta) \text{ s.t. } J_C(\theta) \le d,
\end{equation}
where $J_R(\theta)$ is the objective based on positive reward, $J_C(\theta) = E_{\tau \sim p_\theta(\tau)}C(\tau)$, and $d$ is a fixed threshold. Reward Constrained Policy Optimization (RCPO; \cite{TeMaMa19}) is a recent method for learning CMDP policies.  RCPO learns to scale the weight of cost terms relative to rewards by solving a dual problem, treating the scaling factor as a Lagrange Multiplier. 

A constrained, risk-sensitive learner may be formulated by using our risk-sensitive policy gradient estimate in a formulation resembling RCPO. Our motivation for doing this is twofold.  First, it allows for an acceptable cost limit to be set ahead of time, removing the need to manually tune the relative weights of positive and negative reward terms.  Second, it leverages the observed tendency of our ``pessimistic'' agents to accumulate lower costs at similar or higher positive reward levels than standard approaches.
    
\subsection{Learning Algorithm}
The above policy gradient estimate may be used to maximize distributional objectives of the form (\ref{cdf-rl}), with or without constraints.  The constrained method, Constrained, Risk-Sensitive Proximal (CRiSP) policy optimization, is given in Algorithm \ref{alg:crisp}.  Note that the Adam optimizer \citep{KiBa17} was used for all parameter sets.  To perform unconstrained learning, one would simply fix $\lambda$ and combine the positive and negative reward terms into a single function $r + c$ (thereby allowing a single value function).  The unconstrained method is given explicitly in Appendix A.3.

\begin{algorithm}[t]
 \caption{CRiSP Policy Optimization\label{alg:crisp}}
 \begin{algorithmic}[1]
\REQUIRE Policy: initial parameters $\theta_0$, learning rate $\alpha_\theta$, updates per batch $M_\theta$
\REQUIRE Value functions: initial parameters for utility, cost value functions $\phi_{u,0}$, $\phi_{c,0}$, steps per update $M_{\phi_u}$, $M_{\phi_c}$
\REQUIRE Penalty: initial value $\lambda_0 \ge 0$, learning rate $\alpha_\lambda$
\REQUIRE Stopping threshold $D_{\text{KL, stop}}$, discount factor $\gamma$
\FOR{$k = 0, 1, 2, \ldots$}
\STATE Collect set of episodes $\mathcal{D}_k = \{ \tau_i \}$ by running policy $\pi(\theta_k)$ in the environment
\STATE Update penalty $\lambda_{k+1}= \lambda_k + \alpha_\lambda (J_C(\theta) - d)$, using cost constraint $J_C(\theta)$ and limit $d$ (1 step)
\STATE Compute discounted  utilities-to-go: \\
\begin{center}
$\hat{u}(\mathbf{s}_{i, t}, \mathbf{a}_{i, t}) = \sum_{t'=t}^{T_i}\gamma^{t'-t}u(\mathbf{s}_{i, t'}, \mathbf{a}_{i, t'})$
\end{center}
\STATE Fit utility value function ($M_{\phi_u}$ steps):\\
\resizebox{0.9\hsize}{!}{$\phi_{u, k+1} = \text{arg} \min_{\phi_u} \frac{1}{\sum_iT_i} \sum_{i,t} \bigg( V_{\phi_u}(\mathbf{s}_{i, t}) - \hat{u}(\mathbf{s}_{i, t}, \mathbf{a}_{i, t}))  \bigg) ^2$}

\STATE Compute discounted cost-to-go: \\
\begin{center}
 $\hat{c}(\mathbf{s}_{i, t}, \mathbf{a}_{i, t}) = \sum_{t'=t}^{T_i}\gamma^{t'-t}c(\mathbf{s}_{i, t'}, \mathbf{a}_{i, t'})$
 \end{center}
 
\STATE Fit cost value function ($M_{\phi_c}$ steps):

\resizebox{.9\hsize}{!}{$\phi_{c, k+1} = \text{arg}\min_{\phi_c} \frac{1}{\sum_iT_i} \sum_{i,t} \bigg( V_{\phi_c}(\mathbf{s}_{i, t}) - \hat{c}(\mathbf{s}_{i, t}, \mathbf{a}_{i, t}))  \bigg) ^2$}

\STATE Update effective utility in batch: \\
\begin{center}
$u(\mathbf{s}_t, \mathbf{a}_t) \leftarrow u(\mathbf{s}_t, \mathbf{a}_t) - \lambda c(\mathbf{s}_t, \mathbf{a}_t)$
\end{center}
\STATE Update utility advantage estimates
$A_u^\pi(\mathbf{s}, \mathbf{a})$ using $V_\phi(\mathbf{s}) = V_{\phi_u}(\mathbf{s}) - \lambda V_{\phi_c}(\mathbf{s})$
\STATE Compute weight coefficients based on ordered full-episode rewards
\STATE Update policy using clipped-action policy gradient correction over $M_\theta$ steps with KL-based early stopping (threshold $D_{\text{KL, stop}}$):\\

$\hspace{10pt}\theta_{k+1} = \text{arg}\max_\theta \bigg( \frac{1}{N} \sum_{i=1}^N  \left(w'(\frac{i}{N})+w'(\frac{i-1}{N})\right)\:* $ \\
 $\hspace{20pt}\sum_{t=1}^{T_i} L_{\text{clip}}(\log \pi_\theta(\mathbf{a}_{i,t}| \mathbf{s}_{i,t}), A_u^\pi(\mathbf{s}_{i,t}, \mathbf{a}_{i,t}))\bigg)$

\ENDFOR
\end{algorithmic}
\end{algorithm}

Algorithm \ref{alg:crisp} differs from conventional methods in the requirement to collect full episodes of data in each batch. This is unnecessary if outcomes can be defined over partial episodes, an assumption that is often viable (for instance with the Atari suite \citep{BeNaVeBo13}) and matches human decision-making. 

\section{Experiments}\label{experiments}
We used the OpenAI Safety Gym \citep{RaAcAm19} to evaluate our approach.  Safety Gym is a configurable suite of continuous, multidimensional control tasks wherein different types of robots must navigate through obstacles with different dynamics to perform different tasks. By including both positive and negative reward terms, it allows evaluation of how agents handle risk and constraints.  Safety Gym is also highly stochastic: the locations of the goals and obstacles are randomized, leading to outcome variability and requiring a generalized strategy.

Safety Gym logs adverse events but does not include them in the reward function. For unconstrained experiments, we assigned each logged adverse event a fixed, negative reward.  The coefficient for cost events was learned in constrained experiments. To highlight performance variability, we focused on the most obstacle-rich (level 2) publicly available environments. Avoiding the longer compute time of the ``Doggo'' robot, we evaluated the ``Point'' and ``Car'' robots on each task (“Goal”, “Button”, and “Push”). Additional details are available in Appendix A.4.

In all experiments, we evaluated five random seeds and matched the hyperparameters used in the baselines accompanying Safety Gym as closely as possible.  The neural networks used to model both policy and value were multilayer perceptrons (MLPs), with two hidden layers of 256 units each and $\tanh$ activations.  The policy networks output the mean values of a multivariate gaussian with diagonal covariance.  Control variances were optimized but independent of state.  All variance reduction and regularization measures were used throughout (see Appendix A.5 for experimental justification).

One drawback of our approach is the additional computational expense of its sorting of full-episode rewards (30 per batch here).  In practice we found this to lead to only minor slowdowns.  On average, our method trained $13\%$ slower than PPO in unconstrained trials and $16\%$ slower in constrained trials.

\subsection{Differing Objectives}\label{diff_obj_text}
Agent performance was first explored under four different distributional objectives. In addition to expected reward and CPT (configured to match the original form of \citet{TvKa92} and as given in Appendix A.5), we optimized for pessimistic ($\eta = 0.5$) and optimistic ($\eta=-0.5$) versions of the distortion risk measure proposed in \cite{Wa2000}.  This measure is defined as $w(p) = \Phi(\Phi^{-1}(p) + \eta)$, where $\Phi$ and $\Phi^{-1}$ are the standard normal cumulative distribution function and its inverse. While  this form is convenient, the ``Pow'' metric in \cite{DaOsSiMu18} or any other set of similarly shaped $w$ curves should produce a similar effect.  Note that only the CPT objective used a non-identity mapping from reward to utility.  The four weight functions and their corresponding coefficients in (\ref{gradJ}) are shown in Figure \ref{weight_fig}.  The effects of varying $\eta$ on weight functions and coefficients are displayed more fully in Appendix A.9.

\begin{figure}
    \centering
    \includegraphics[width=0.234\textwidth]{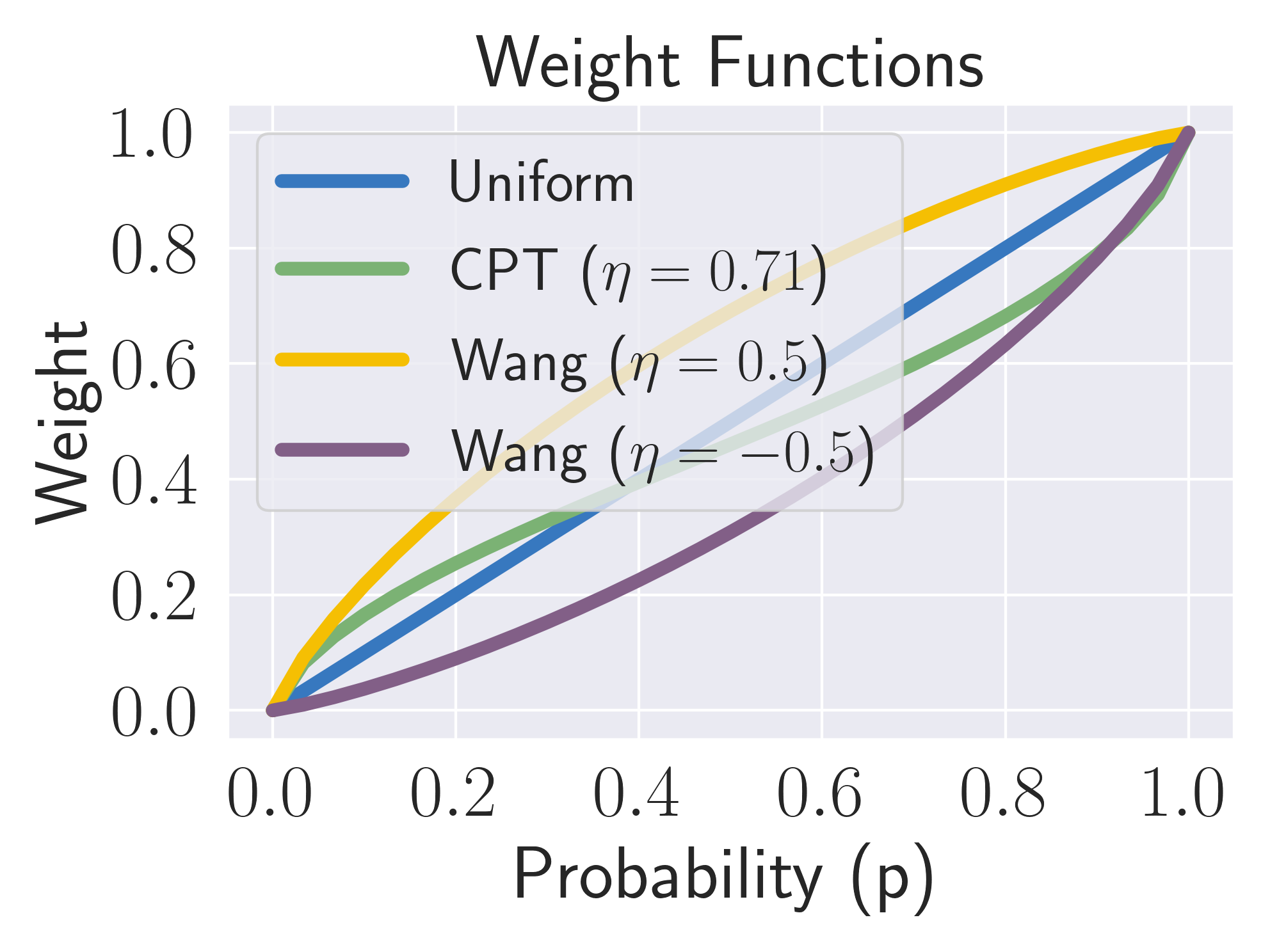}
    \includegraphics[width=0.234\textwidth]{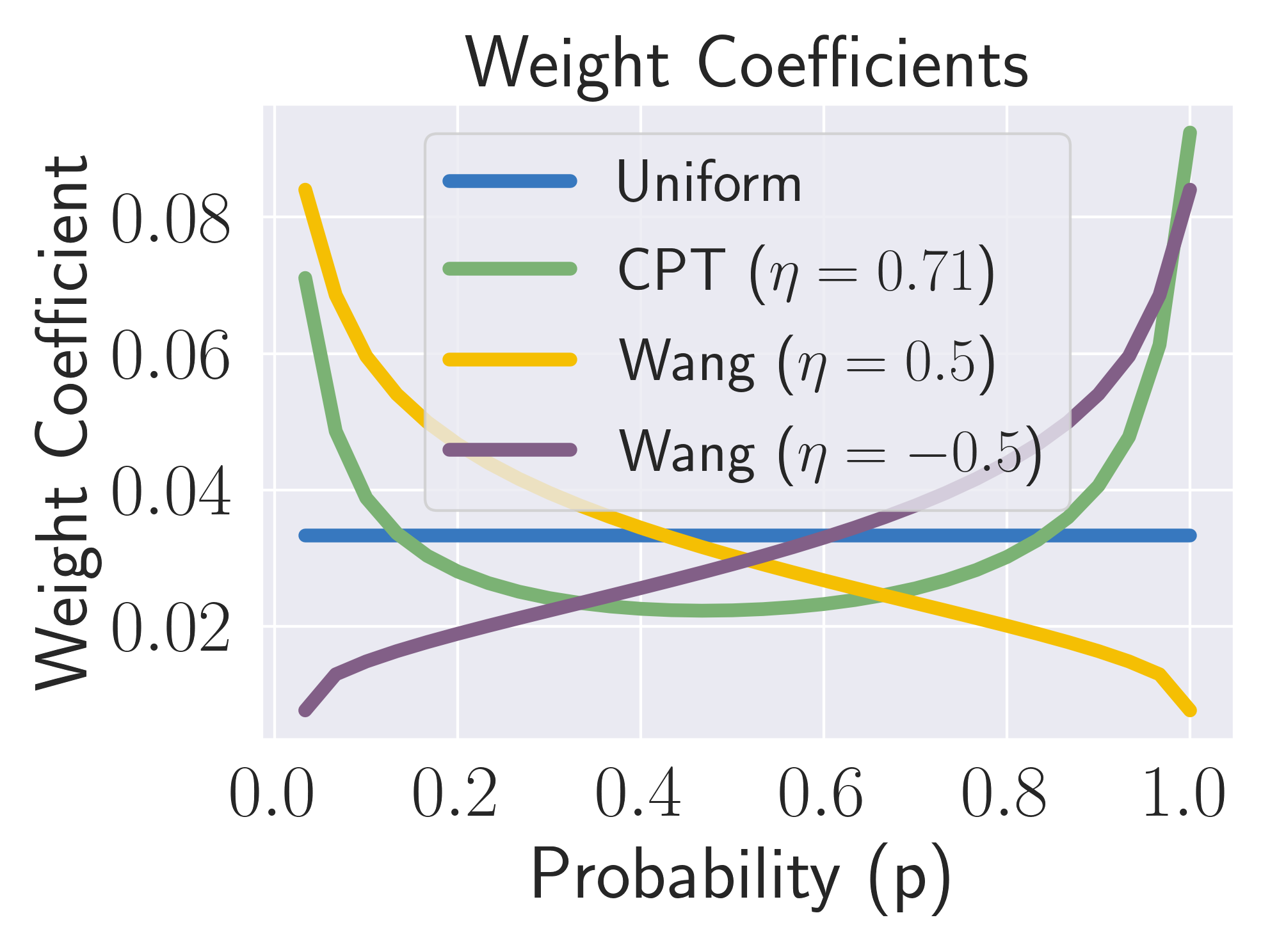}
    \caption{Example weight functions and their resulting coefficients in the policy gradient estimate~(\ref{gradJ}).}
    \label{weight_fig}
\end{figure}

\begin{figure*}[t]
    \centering
    \includegraphics[width=.245\textwidth]{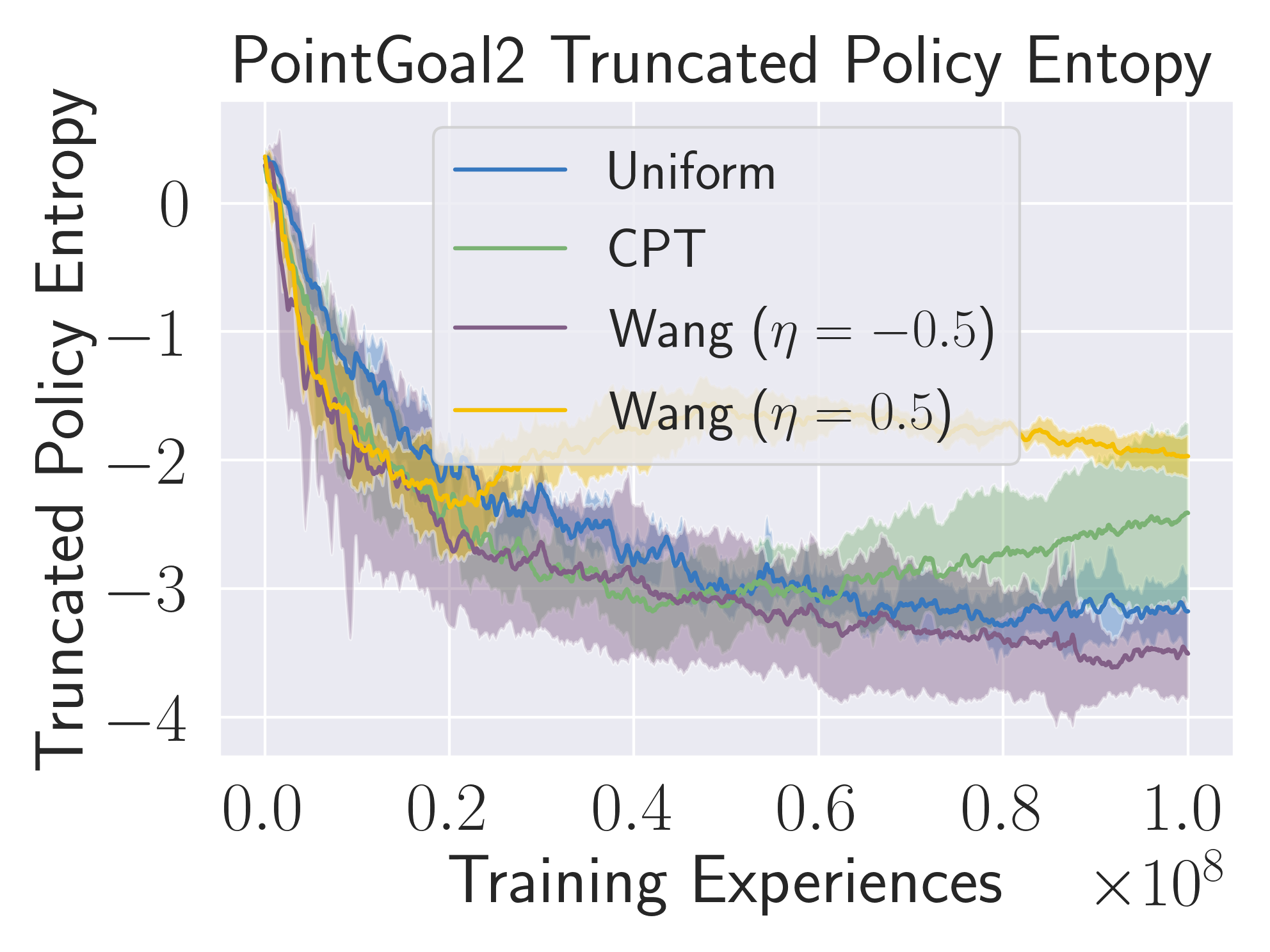}
    \includegraphics[width=.245\textwidth]{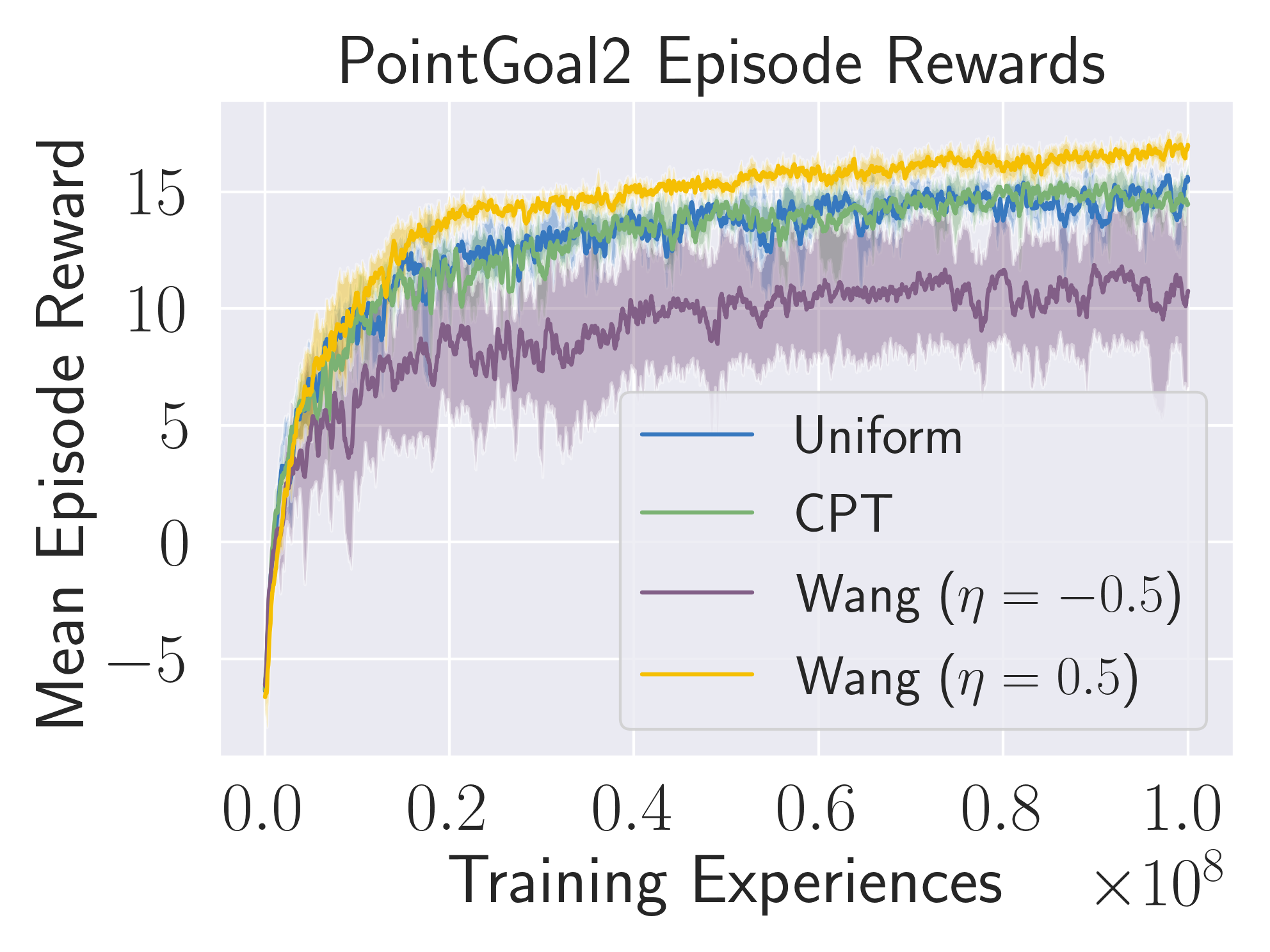}
    \includegraphics[width=.245\textwidth]{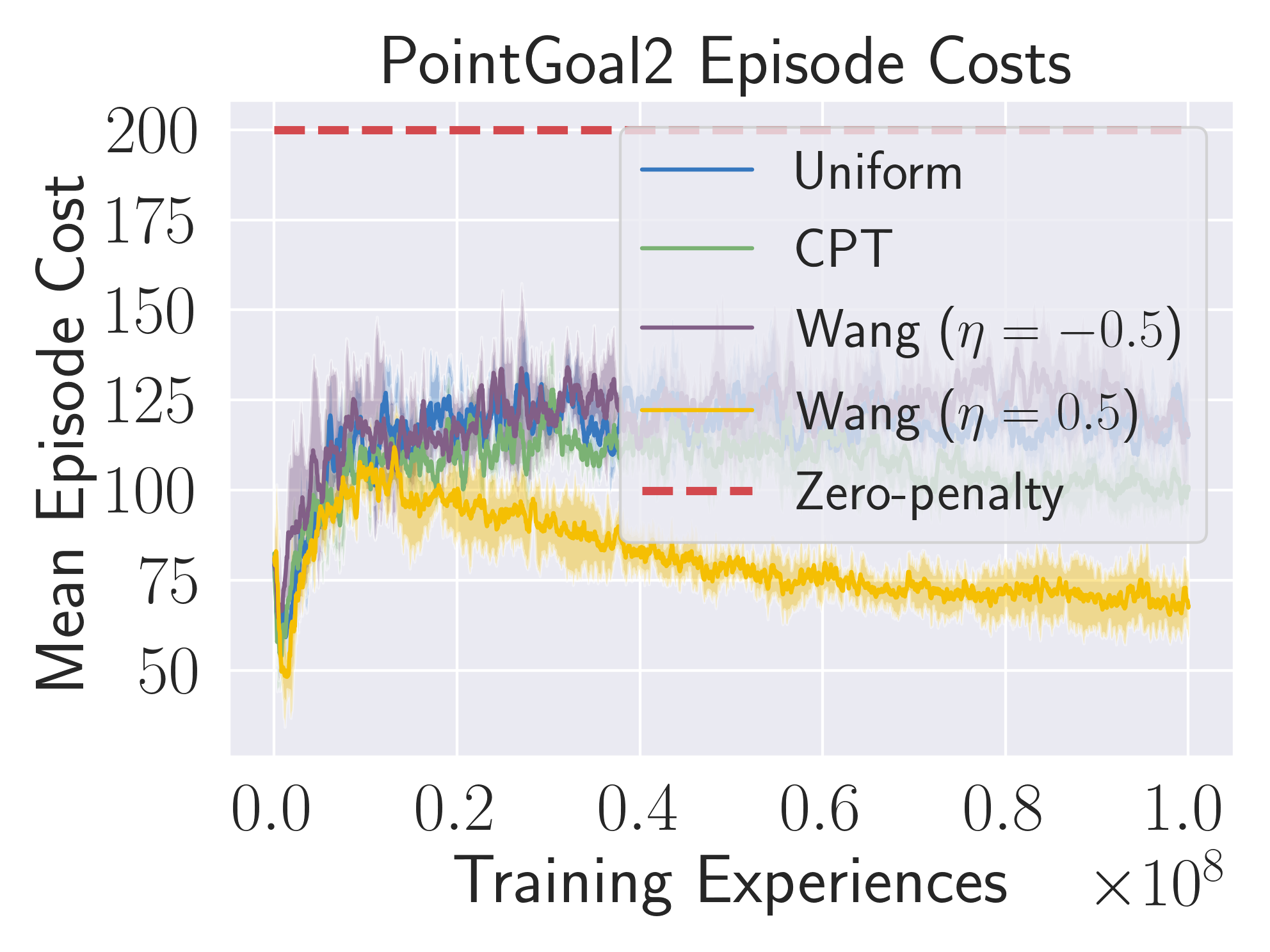}
    \includegraphics[width=.245\textwidth]{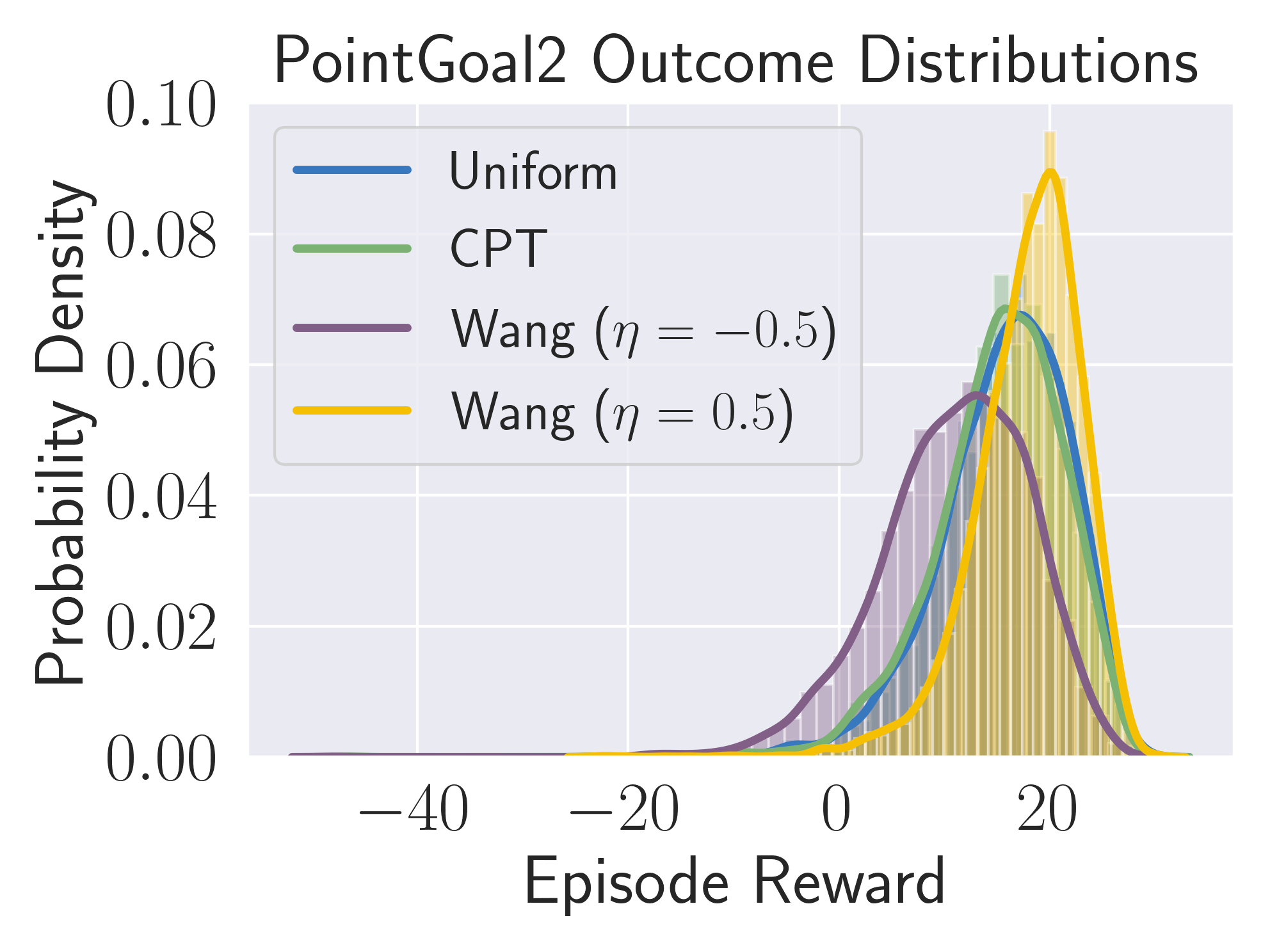}

    \includegraphics[width=.245\textwidth]{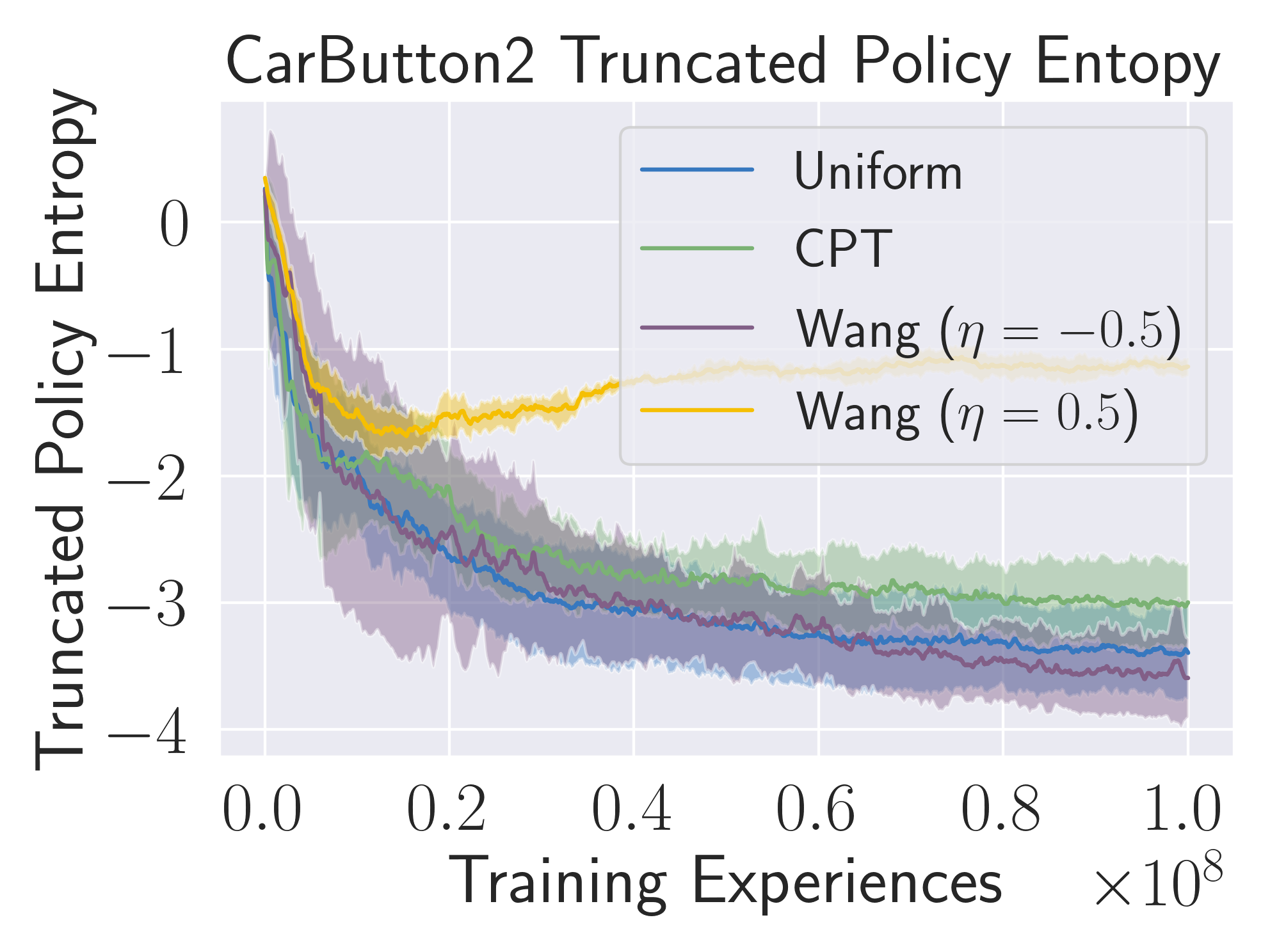}
    \includegraphics[width=.245\textwidth]{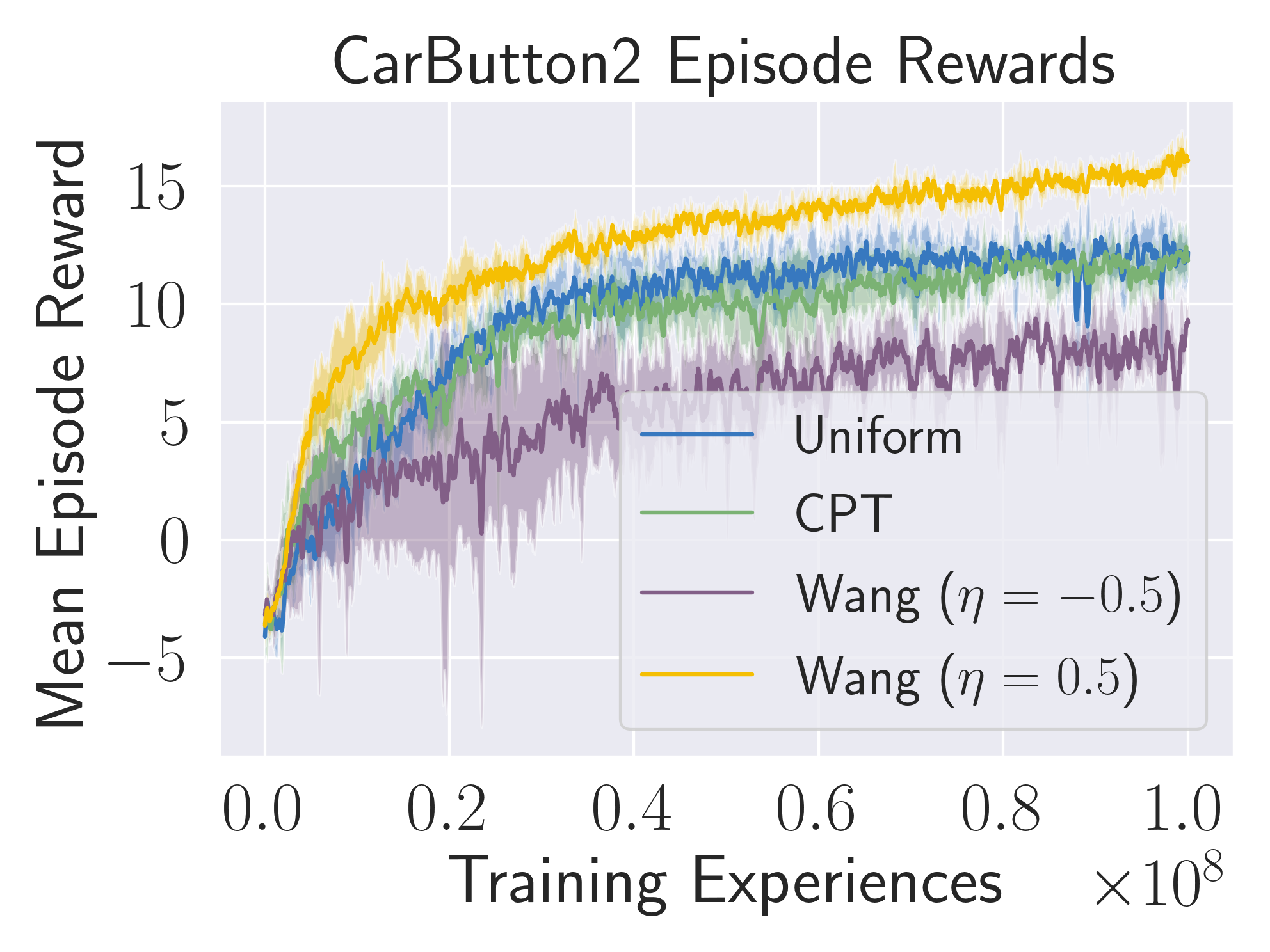}
    \includegraphics[width=.245\textwidth]{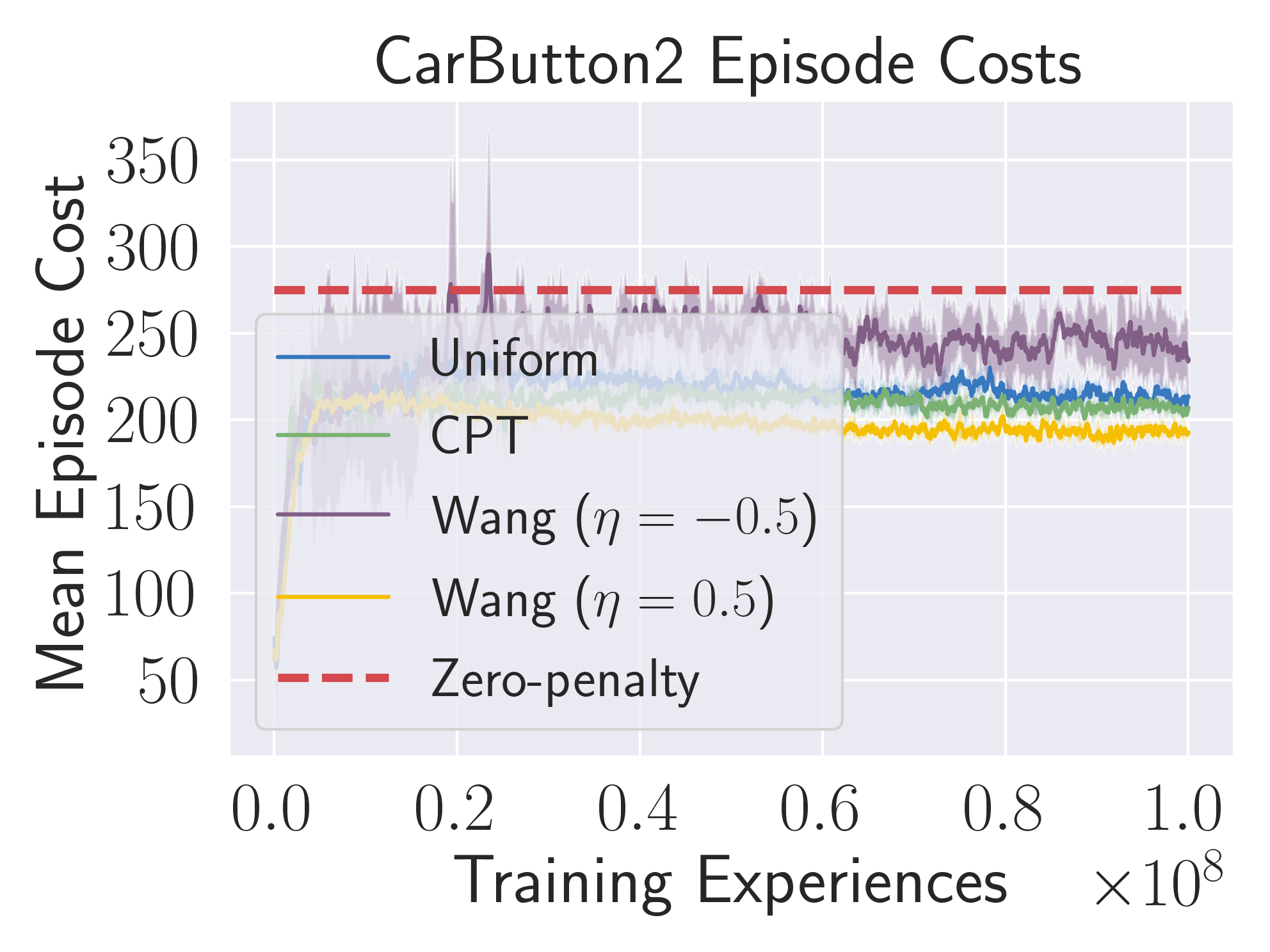}
    \includegraphics[width=.245\textwidth]{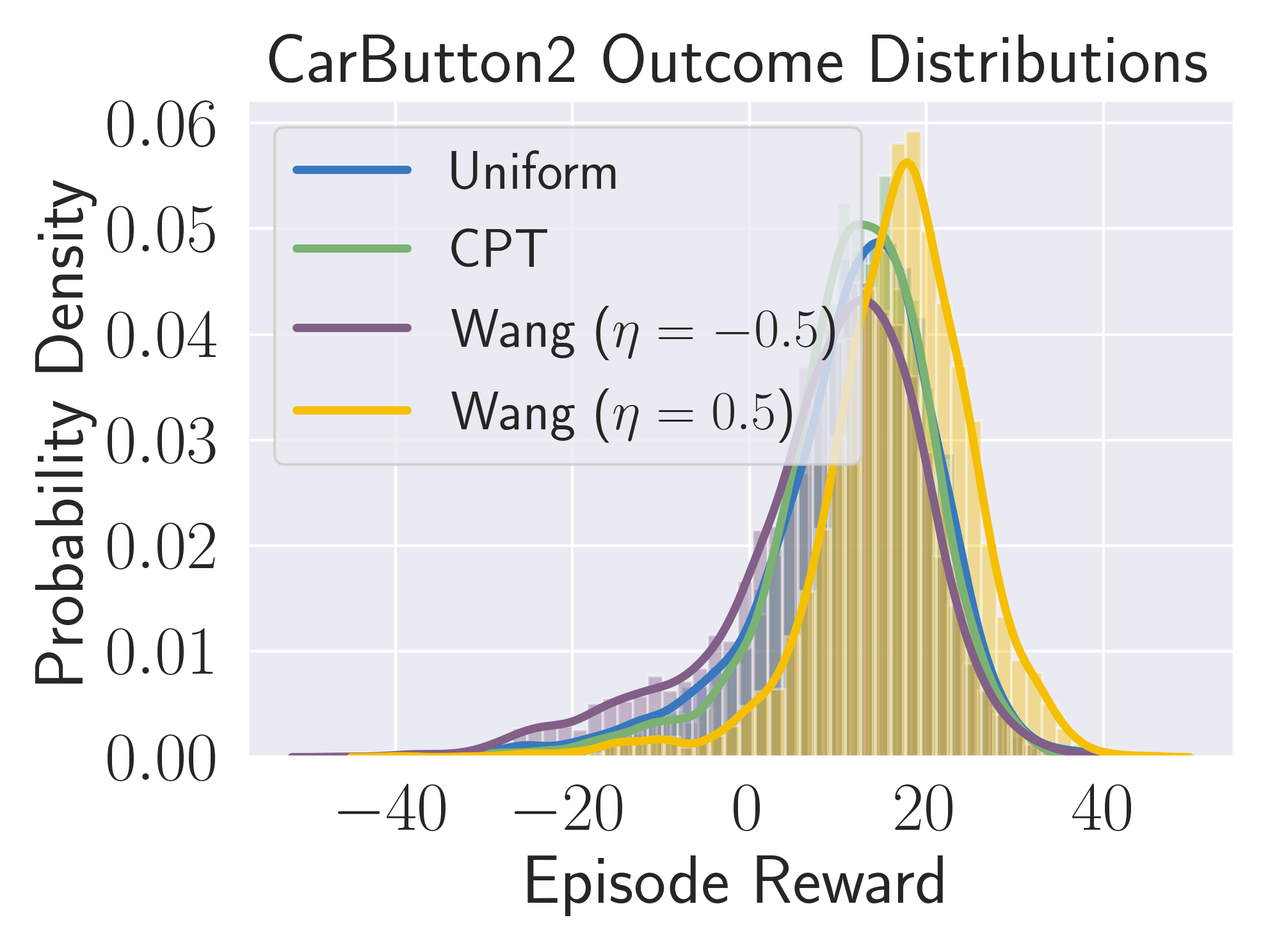}

    \includegraphics[width=.245\textwidth]{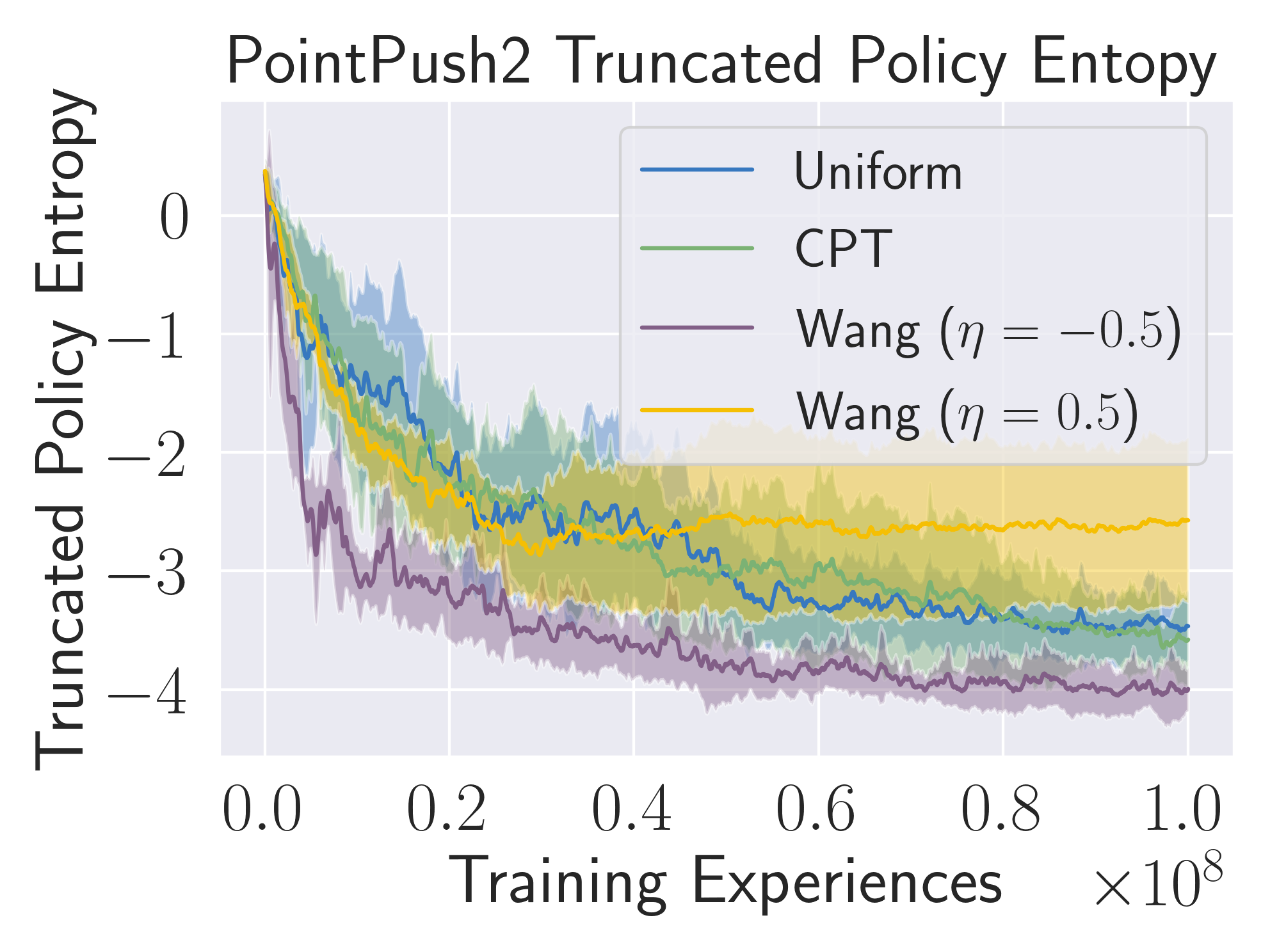}
    \includegraphics[width=.245\textwidth]{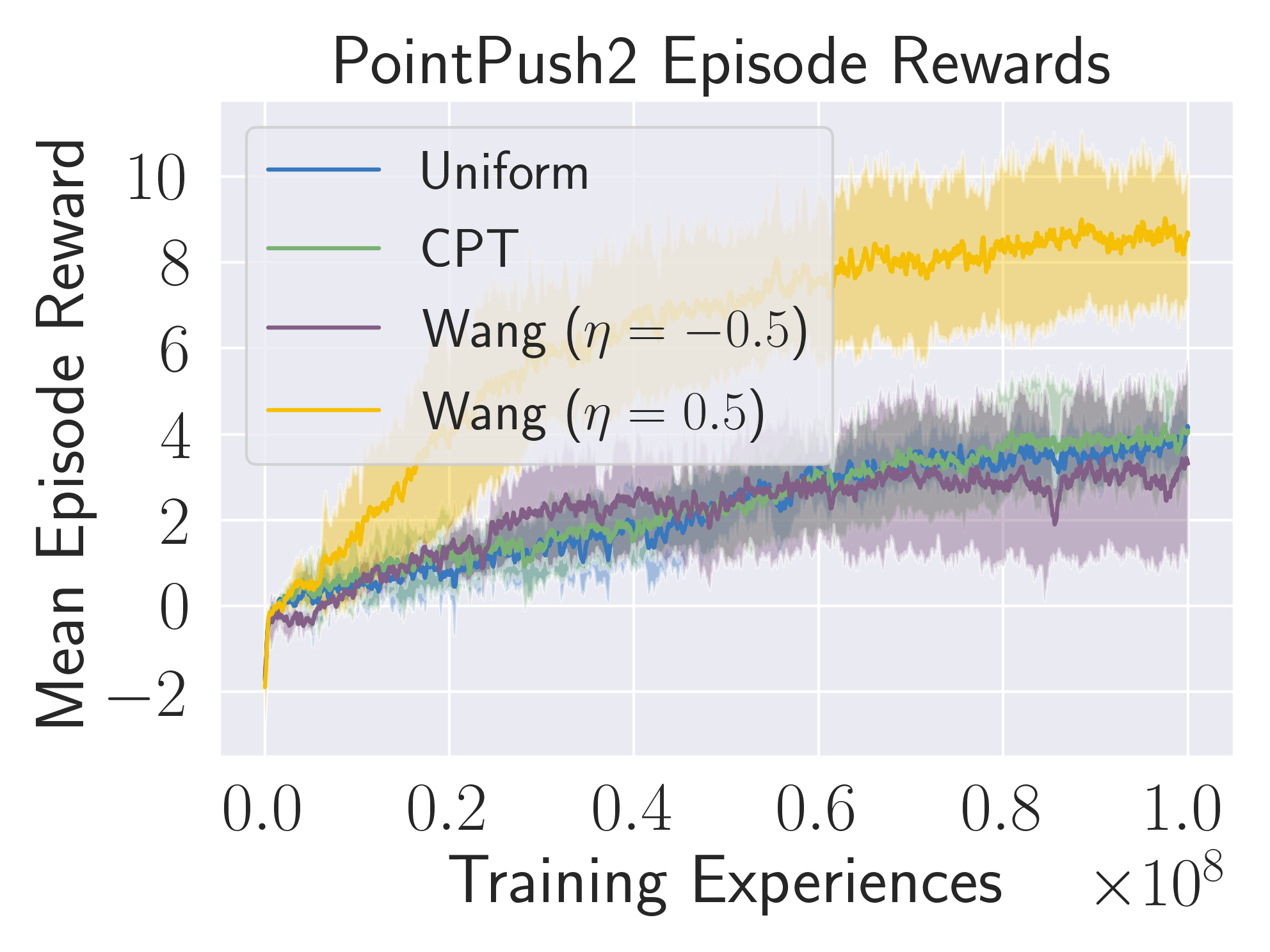}
    \includegraphics[width=.245\textwidth]{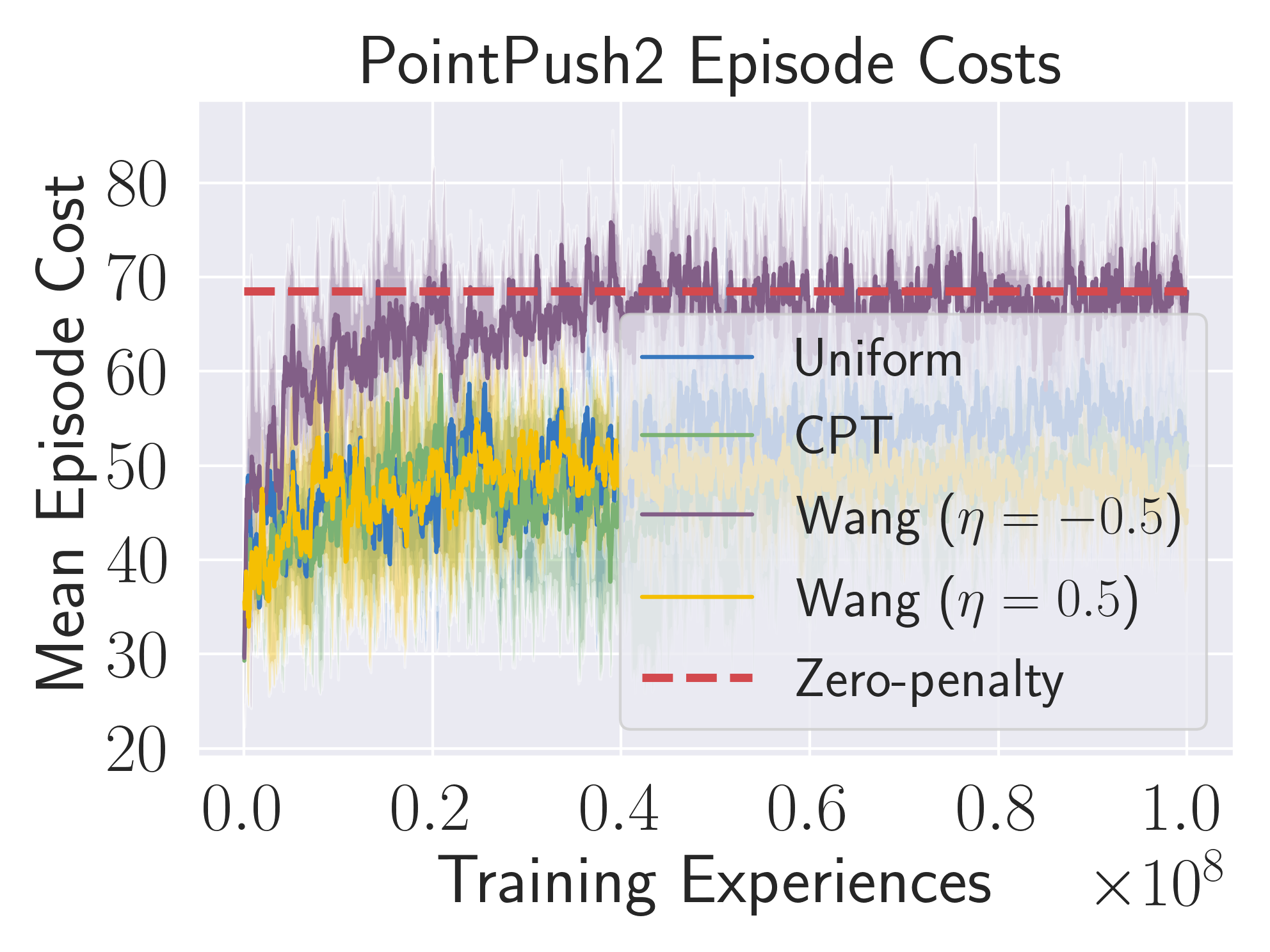}
    \includegraphics[width=.245\textwidth]{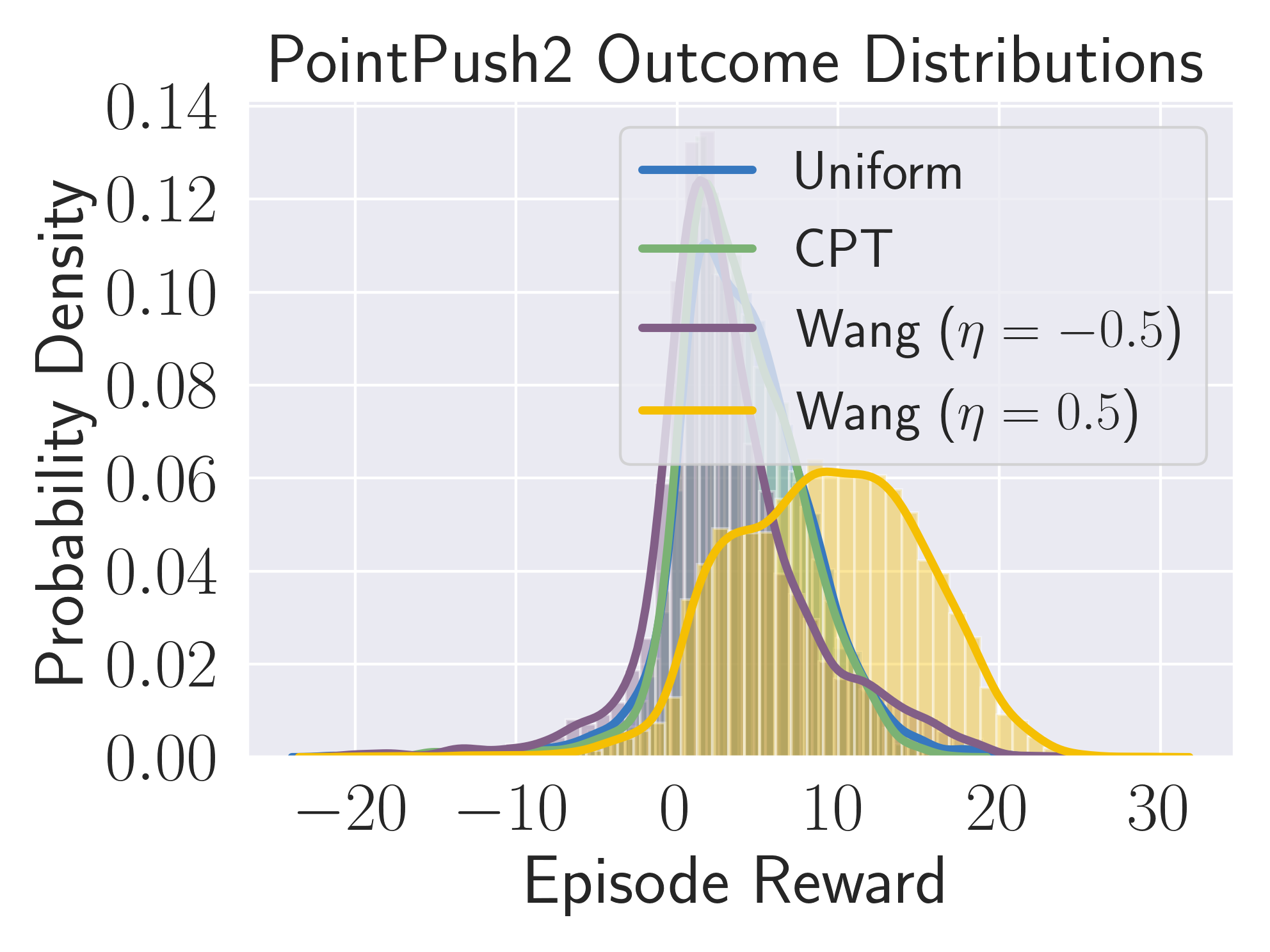}

    \caption{Impact of different distributional objectives in three environments.  Each row represents an environment.  First column: pessimistic weightings ($\eta > 0$; yellow) lead to higher policy entropy, aiding exploration.  Second and third columns: pessimistic weightings lead to higher overall reward (positive reward minus cost) with lower cost contributions than other weightings. Fourth column: testing distributions across seeds (sampling turned off).}
    \label{training}.
\end{figure*}

\subsection{What Should We Expect?}\label{expect}
Before presenting results, it is instructive to inspect the form (\ref{tr}) to set expectations on the impact of different weightings.  The standard policy gradient of \cite{Wi92} can be thought of as maximum likelihood estimation weighted by advantage over observed trajectories.  Over time, it shifts probability mass toward states with positive advantage and away from states with negative advantage.  Weighing gradient contributions differently based on episode outcome adjusts this migration.  In particular, emphasizing experiences from low-reward episodes (pessimistic weighting) should on average lead to stronger contributions from low-advantage terms.  This prioritizes pushing probability mass away from problematic parts of state space, increasing policy entropy.  Conversely, emphasizing experiences from high-reward episodes (optimistic weighting) should prioritize pushing probability mass toward advantageous parts of state space, more rapidly decreasing policy entropy.  One should therefore expect pessimistic weightings to more thoroughly explore the state space, potentially avoiding premature convergence to suboptimal policies.

As a simple quantitative example, consider the contribution to the gradient of the variance parameters of state-independent Gaussian policies for trajectory $i$, to control dimension $j$, at time step $t$:
\begin{equation}
    \begin{split}
    \nabla_{\theta_{\mathbf{\sigma}_j}} &\log \pi_\theta(\mathbf{a}_{i,t}|\mathbf{s}_{i,t})A(\mathbf{s}_{i,t,},\mathbf{a}_{i,t}) = \\ 
    &\frac{1}{\mathbf{\sigma}_j}\left(\frac{(\mathbf{a}_{i,t,j}-\mathbf{\mu}_{i,t,j})^2}{\mathbf{\sigma}_j^2} - 1\right) A(\mathbf{s}_{i,t,},\mathbf{a}_{i,t}).
    \end{split}
\end{equation}
Here we use the chain rule and the fact that $\mathbf{\sigma}$ is a single parameter in each dimension. Note that the first multiplicative term in the $\sigma$ derivative is always positive and the second is negative about $68\%$ of the time (and more negative when sampling near $\mu$).  As learning occurs, sampling close to $\mu$ should correlate with higher advantages.  These trends combine to produce the decrease in $\sigma$ observed in standard learning.  When pessimistic weightings emphasize terms with negative advantage, however, the decrease in $\sigma$ may be slowed or even reversed.

\subsection{Impact of Pessimistic Weightings}\label{pessimistic_text}
 Figure \ref{training} shows the impact of different weightings when training in three environments.  As expected, we observe that policy variance is higher when using pessimistic weightings.  This typically leads to higher truncated (considering control bounds) entropy, depending on how often learned actions are near the boundaries.  Higher levels of total rewards and lower cost levels were achieved with pessimistic weightings, both in training and testing (i.e., with sampling turned off). Note that in this and subsequent plots, the ``zero penalty'' lines reflect the cost levels that standard PPO and TRPO agents reach when unaware of cost \cite{RaAcAm19}. Finally, as shown in Appendix A.6, pessimistic weightings produced higher levels of all metrics in all environments in testing.
 
 
\begin{figure*}[t]
    \centering
    \includegraphics[width=0.245\textwidth]{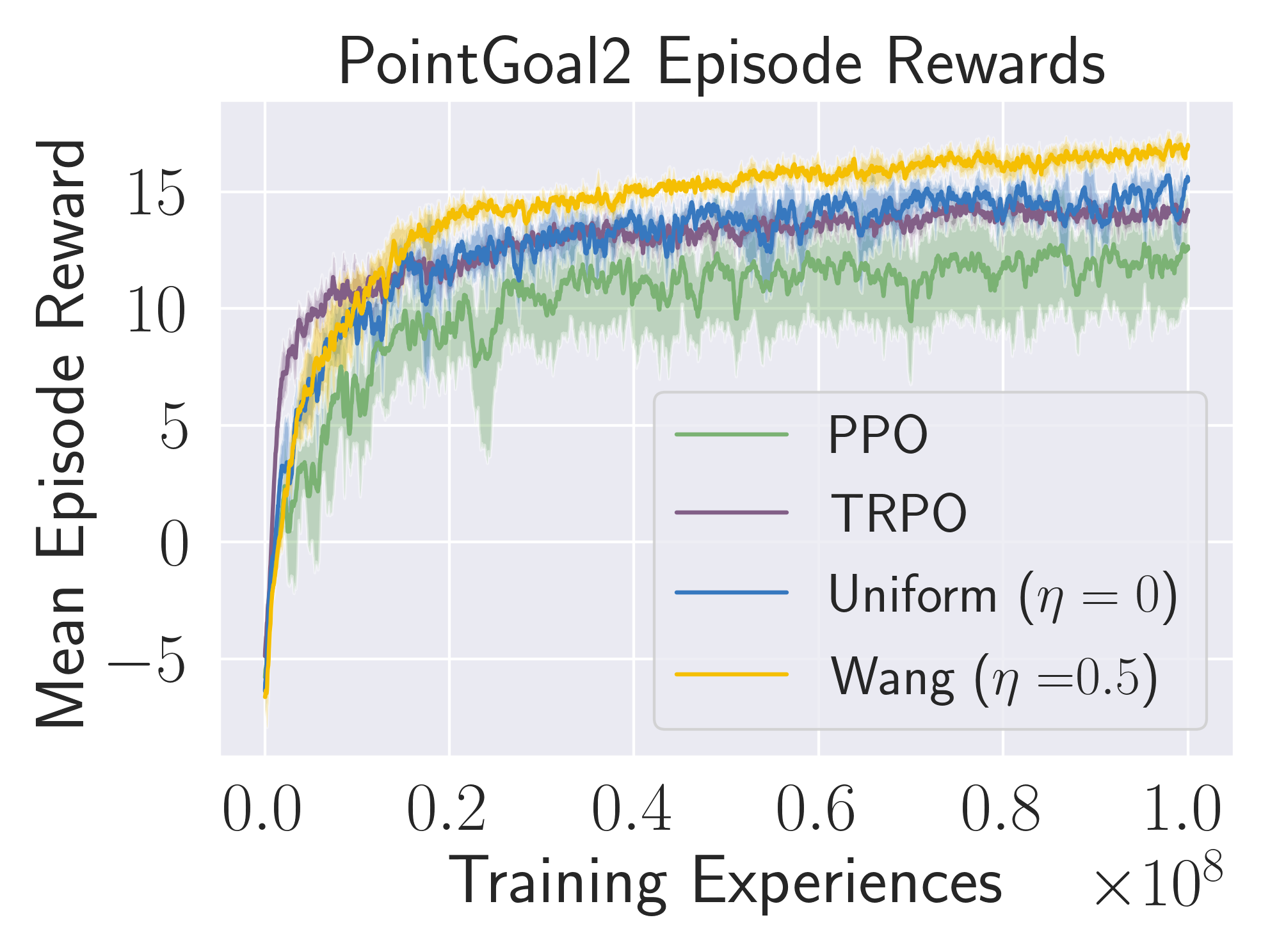}
    \includegraphics[width=0.245\textwidth]{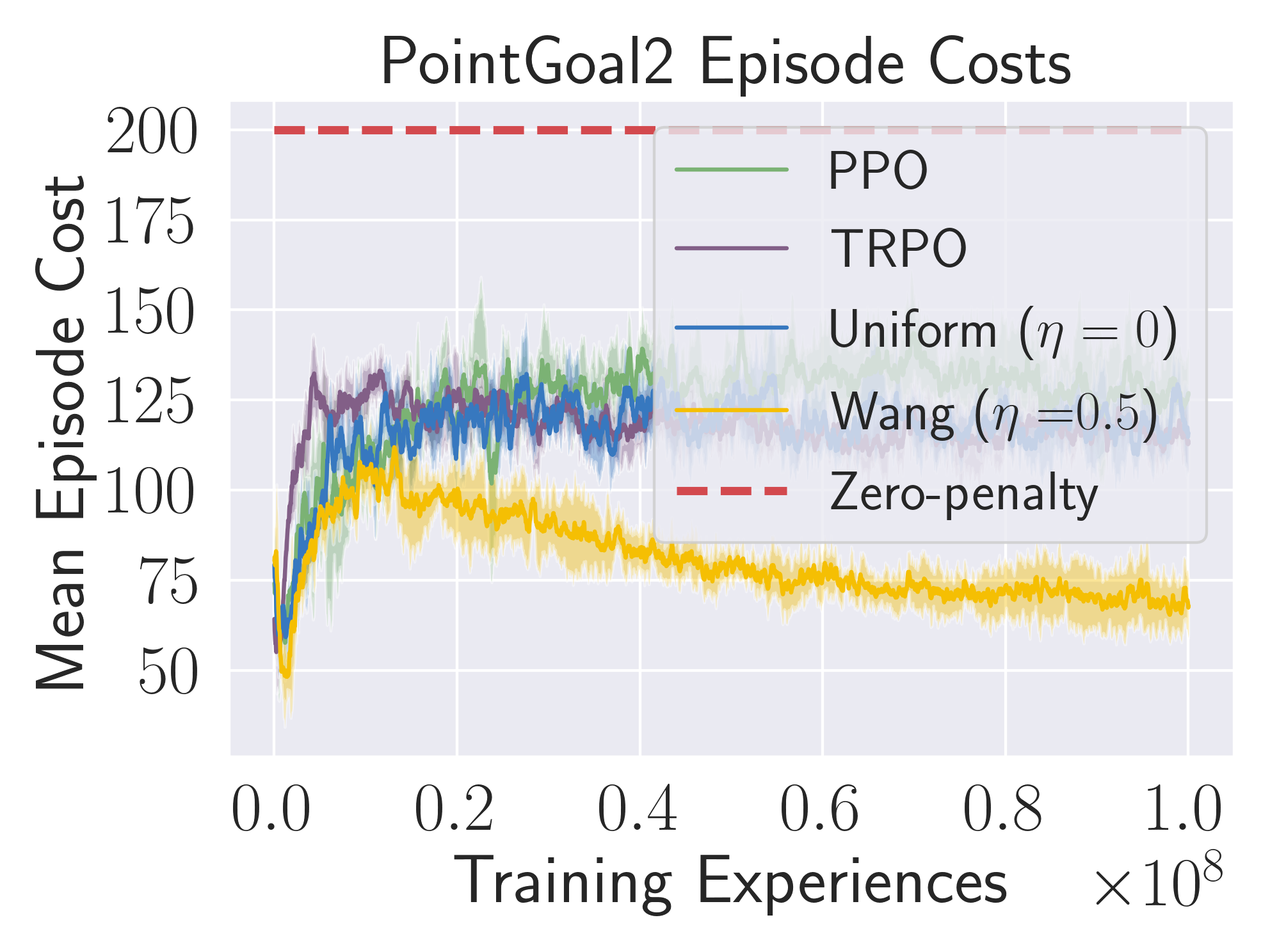}
    \includegraphics[width=0.245\textwidth]{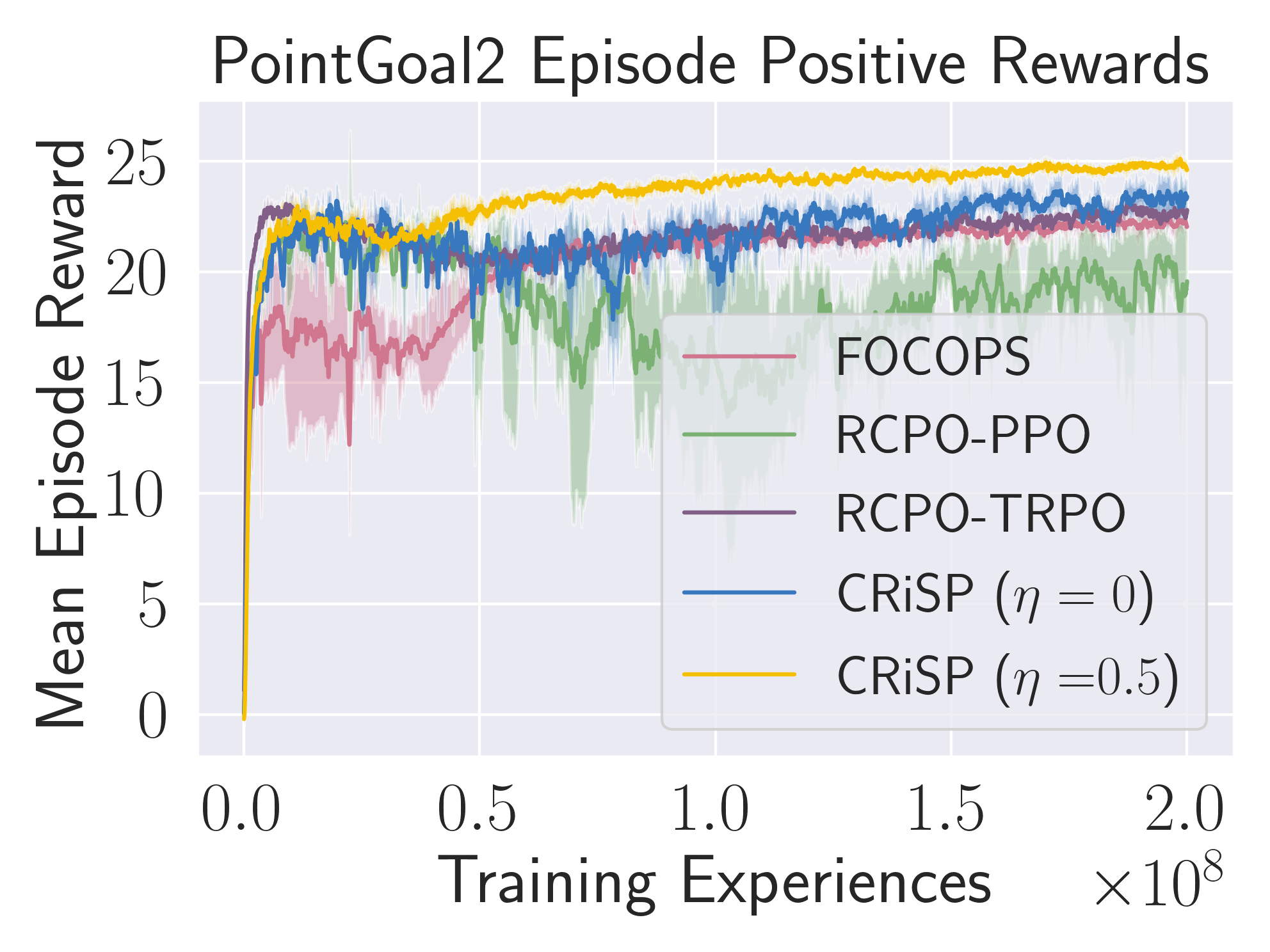}
    \includegraphics[width=0.245\textwidth]{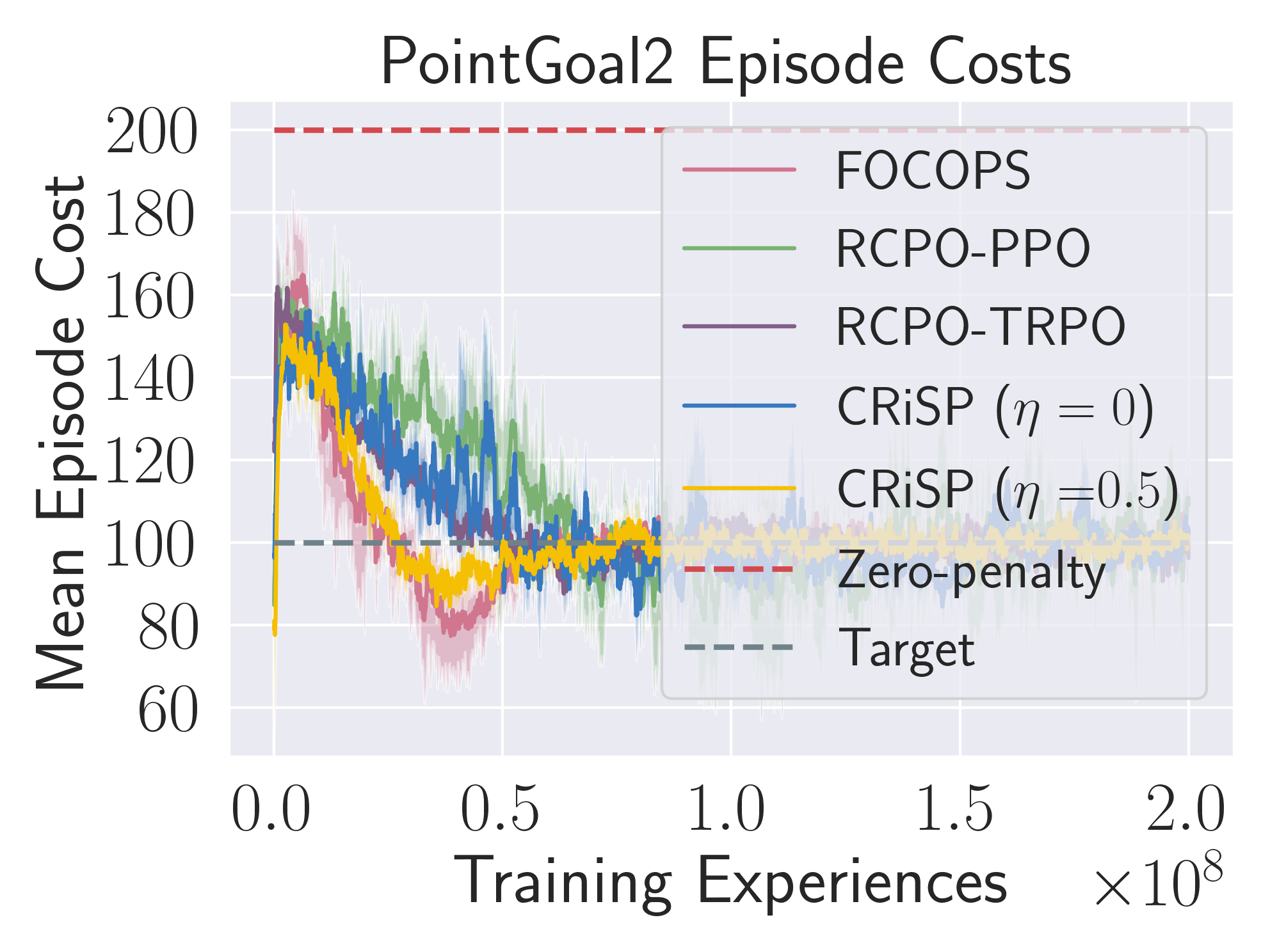}
    \includegraphics[width=0.245\textwidth]{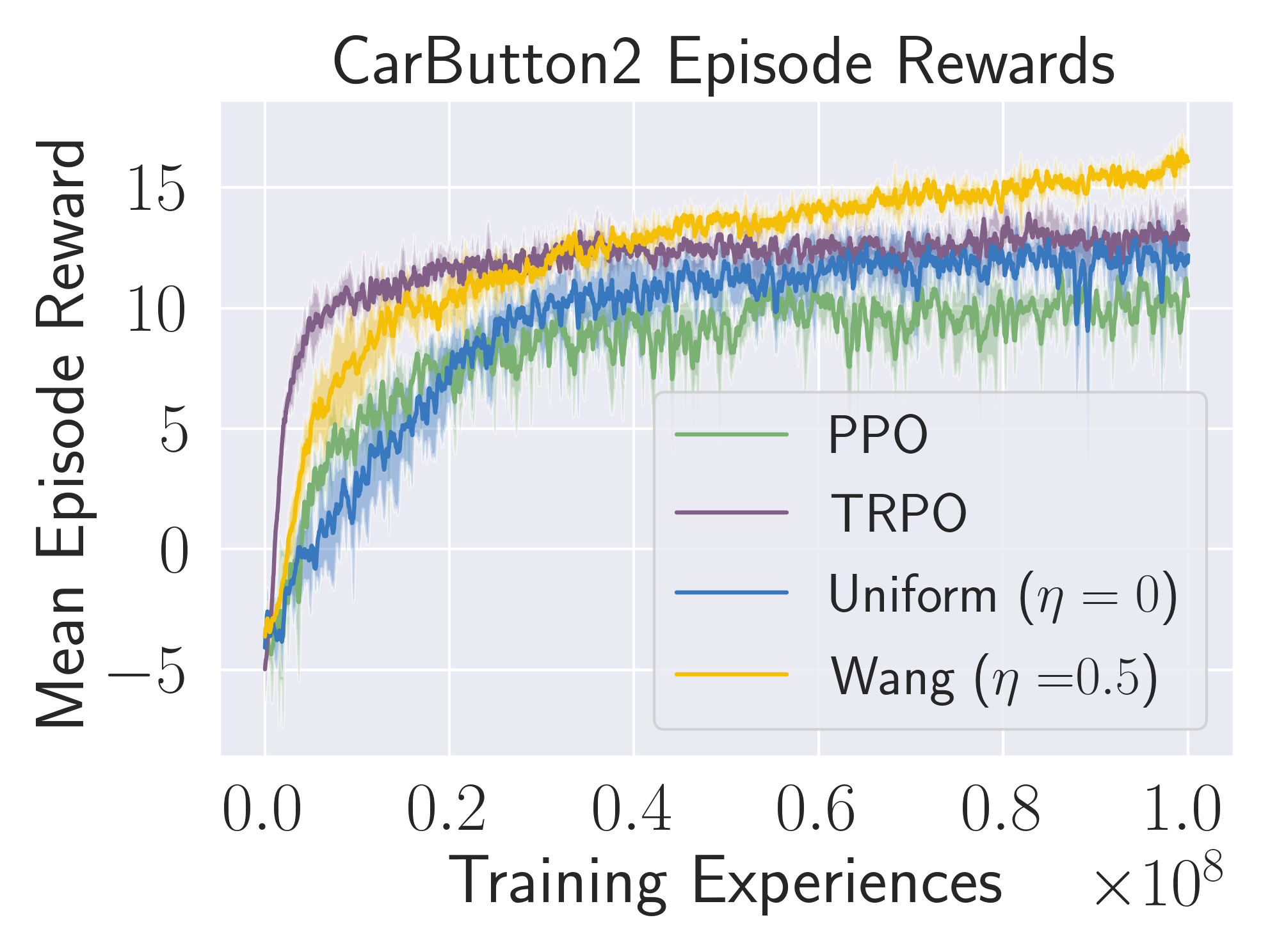}
    \includegraphics[width=0.245\textwidth]{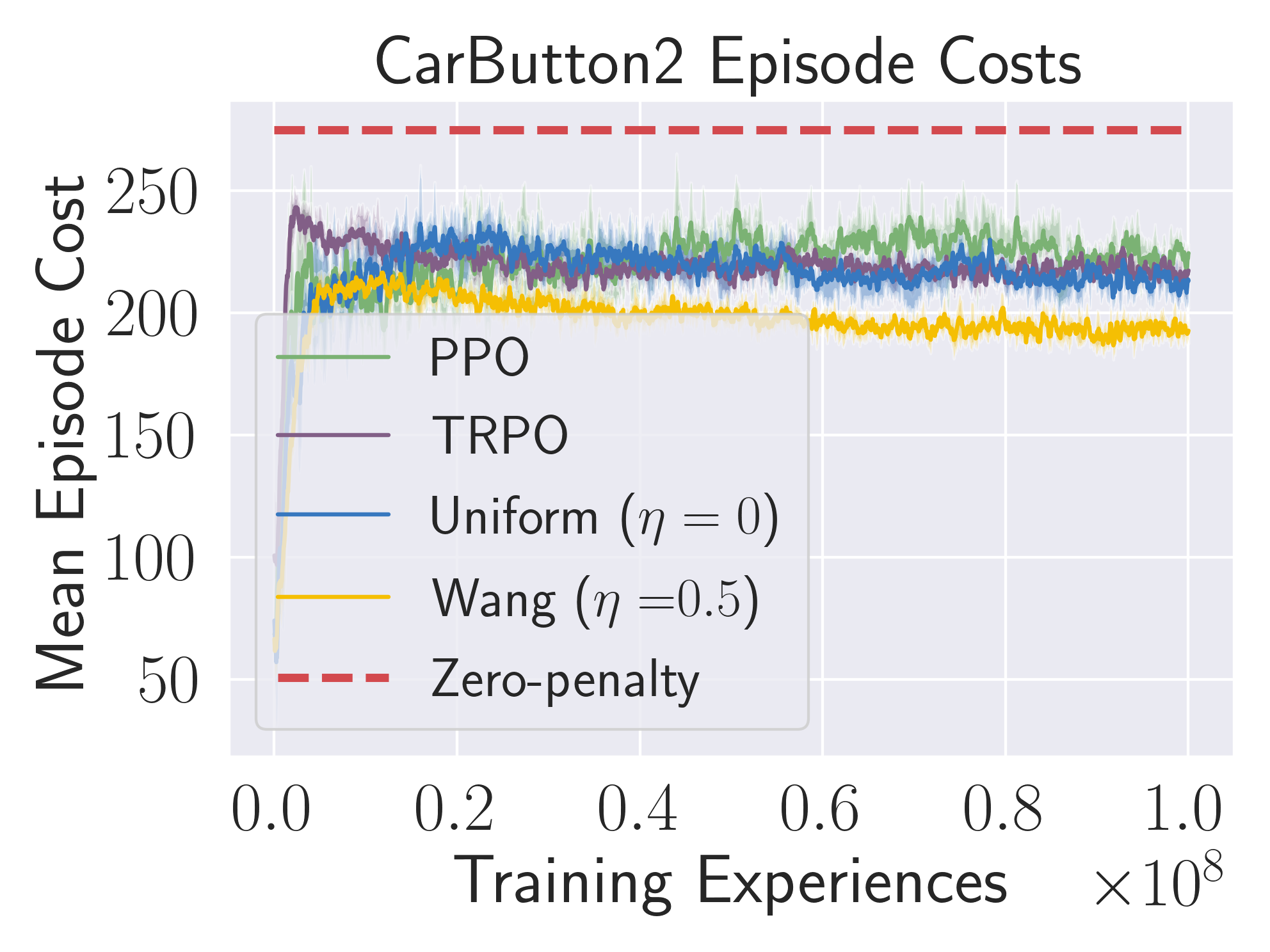}
    \includegraphics[width=0.245\textwidth]{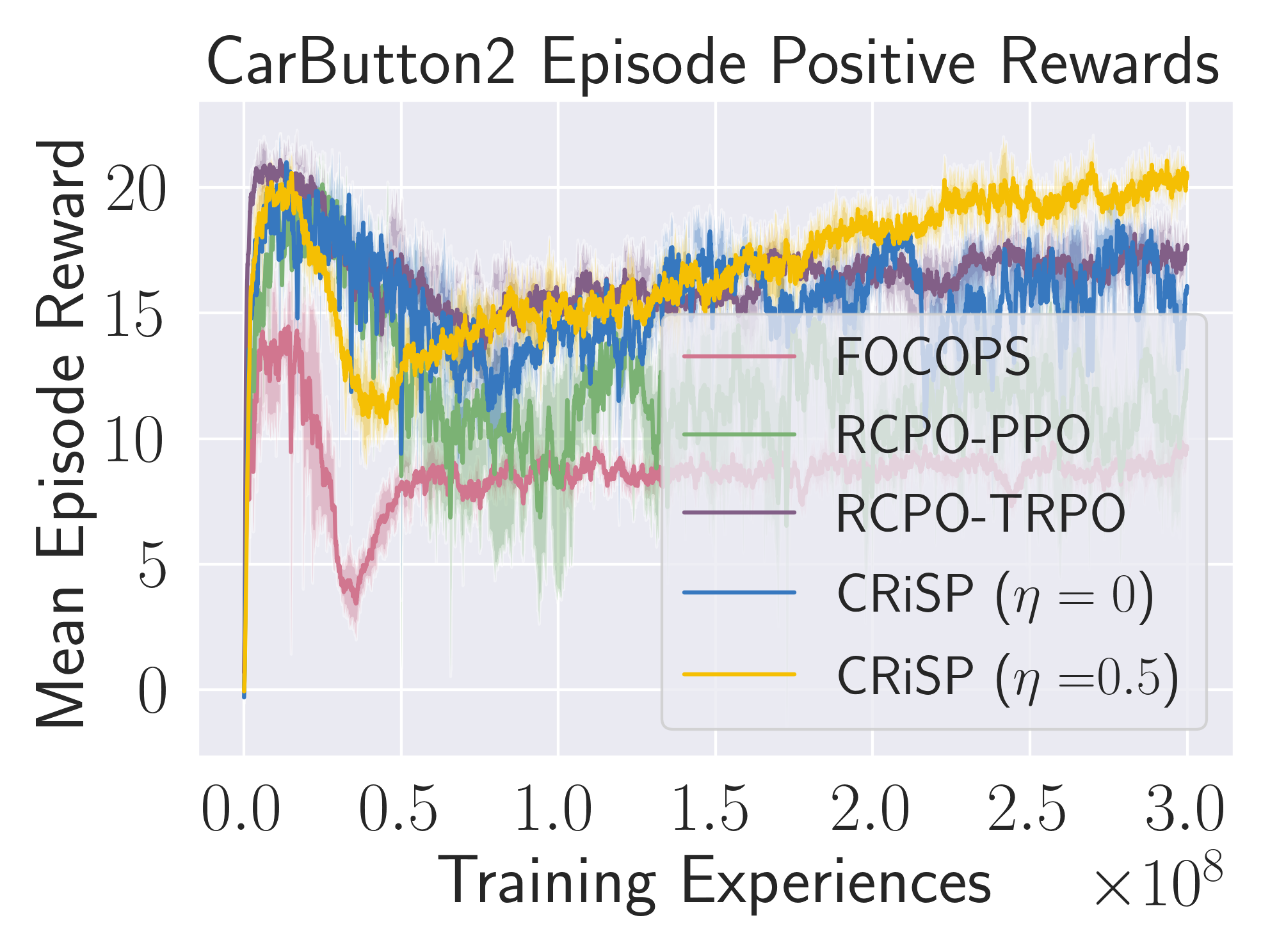}
    \includegraphics[width=0.245\textwidth]{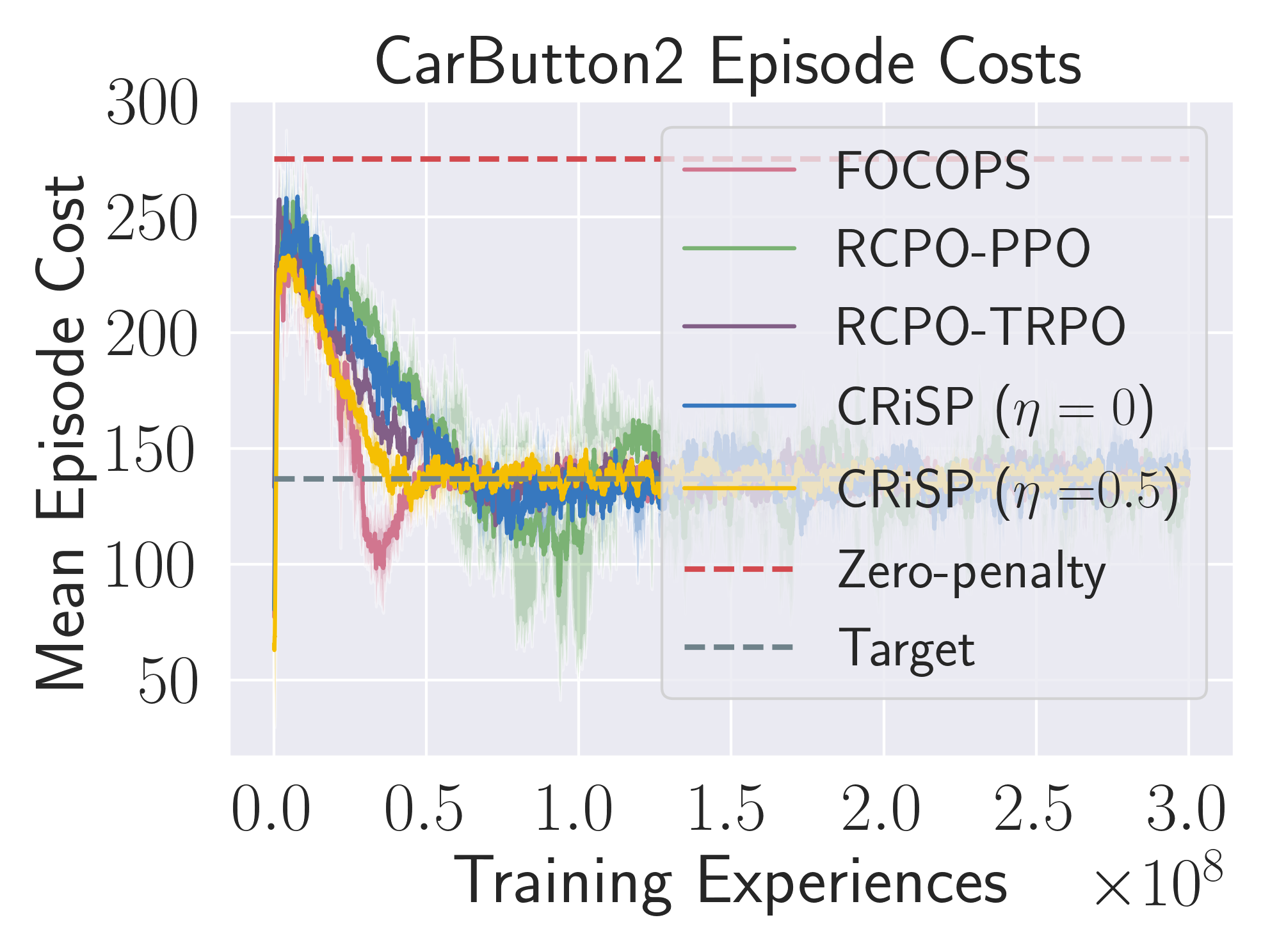}
    \includegraphics[width=0.245\textwidth]{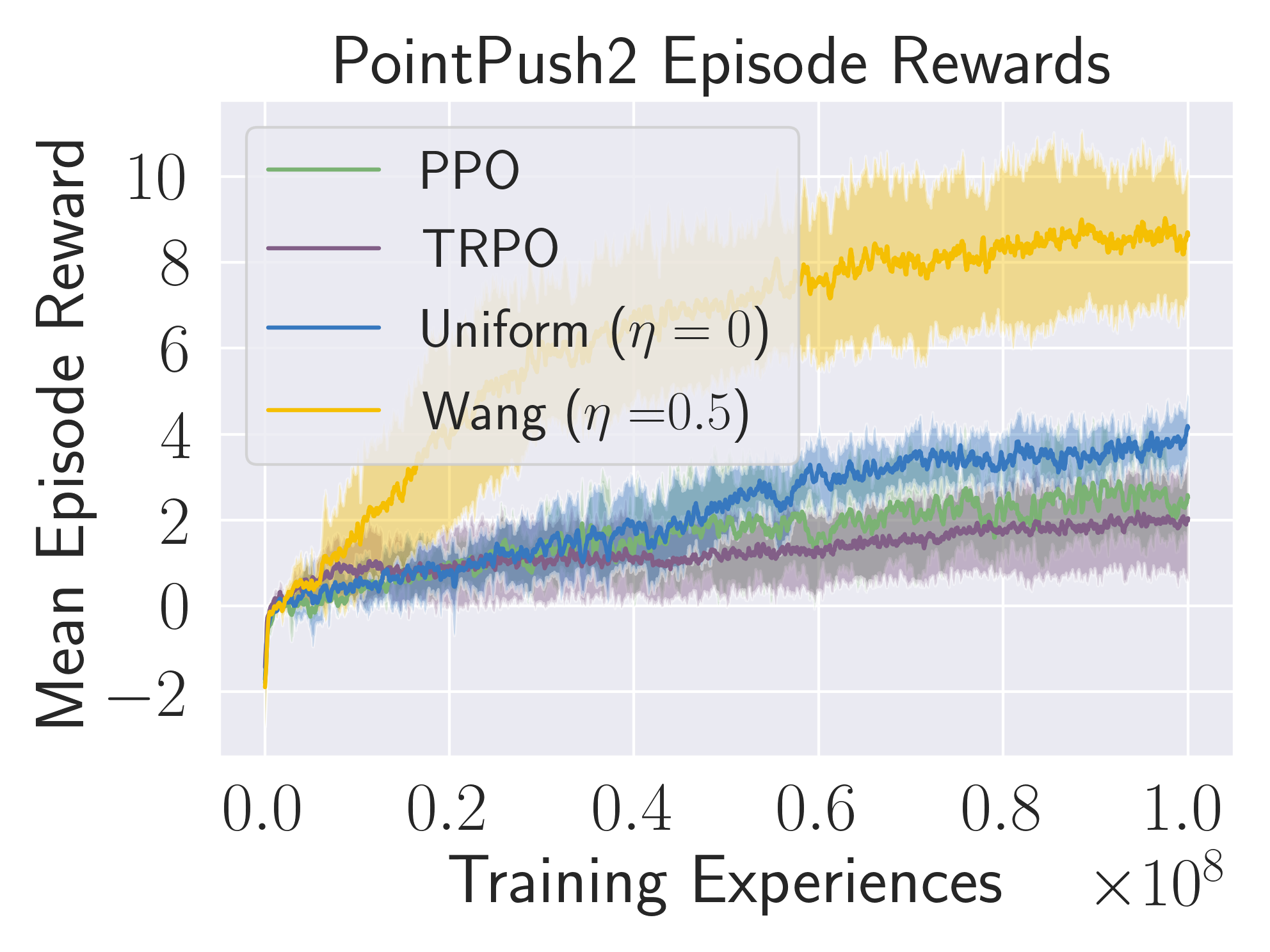}
    \includegraphics[width=0.245\textwidth]{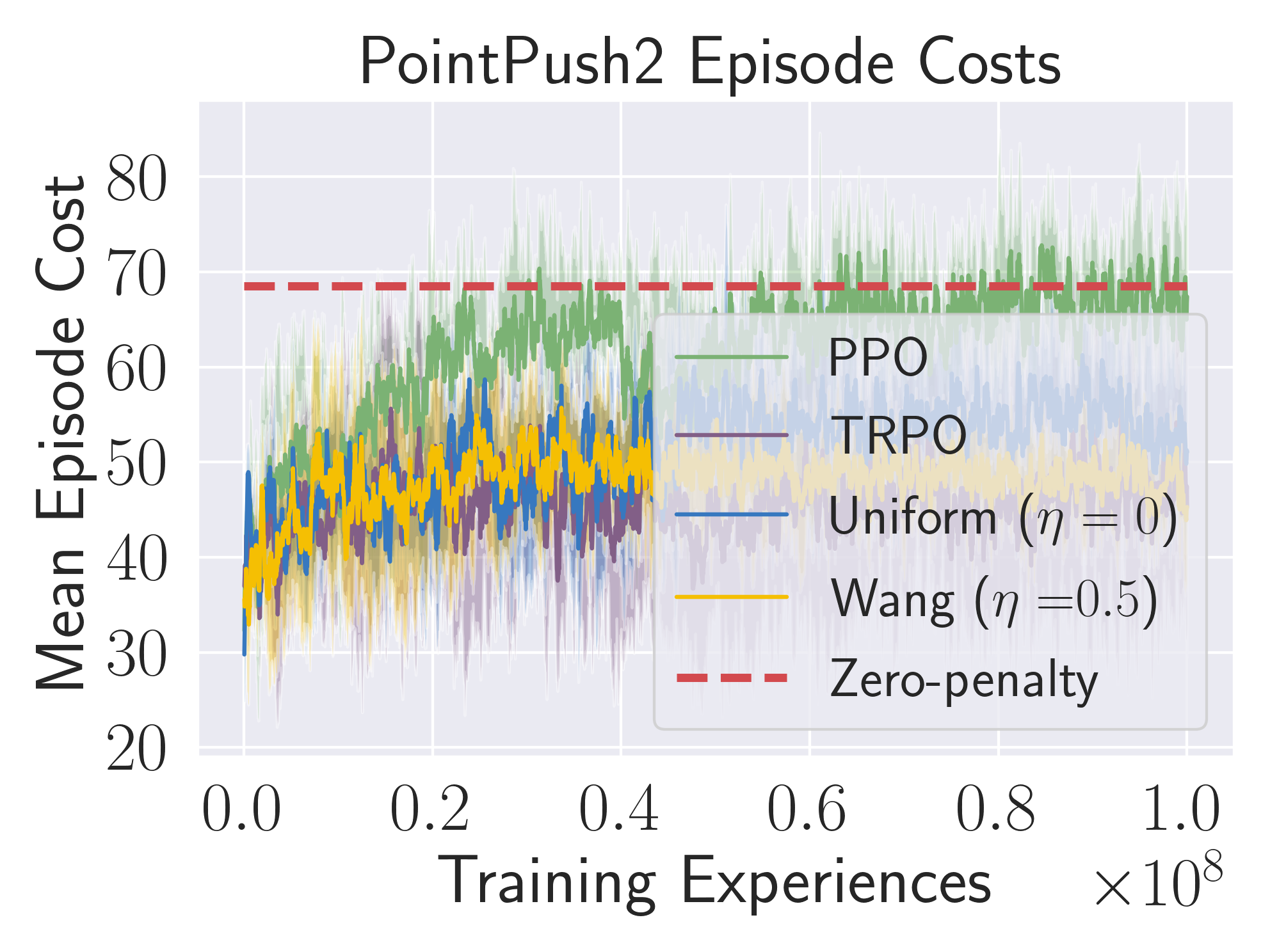}
    \includegraphics[width=0.245\textwidth]{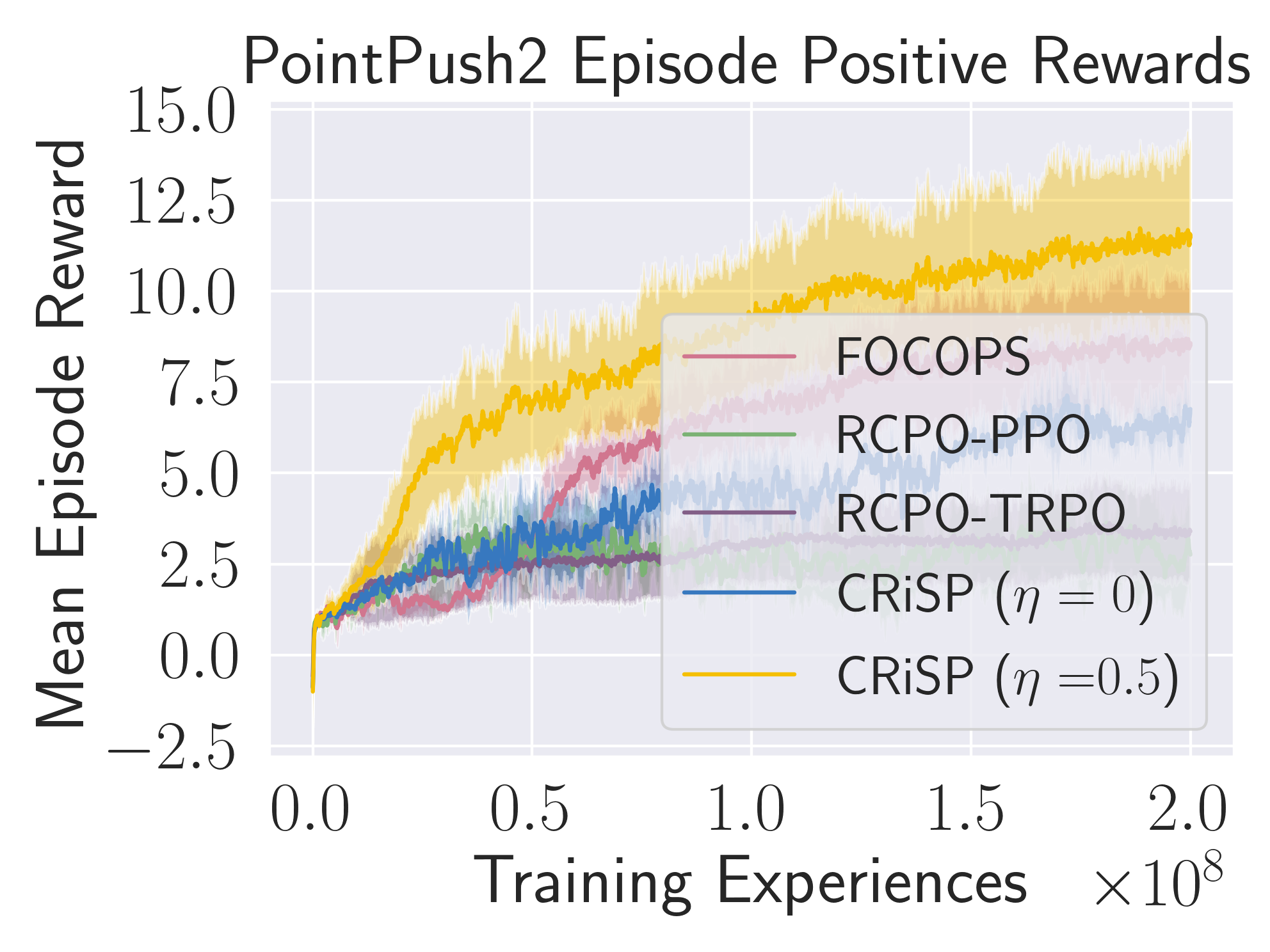}
    \includegraphics[width=0.245\textwidth]{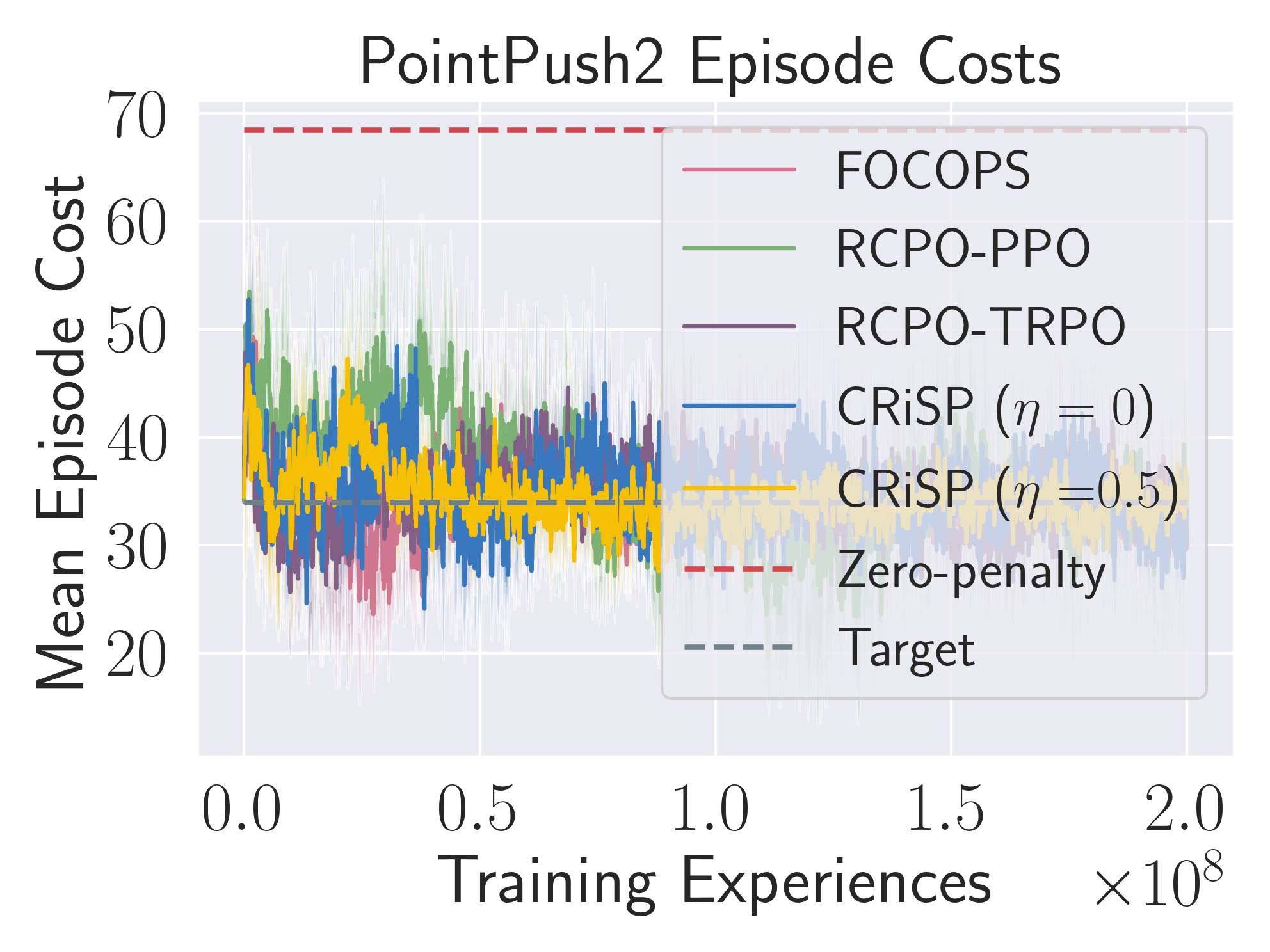}
    \centering
    \caption{Comparison of our pessimistic agents with $\eta=0.5$ (yellow) to other on-policy methods. Columns 1-2: In unconstrained learning, pessimistic agents tend toward higher total reward (including penalty) and lower cost than others. Columns 3-4: In the constrained setting, pessimistic agents accumulate more positive reward than others at the target cost level.}
    \label{long_reward}
\end{figure*}

\subsection{Comparison with Other Unconstrained Baselines}
We pursued comparisons of our risk-sensitive approach, using the pessimistic objective from \cite{Wa2000}, to state-of-the-art on-policy methods. In addition to PPO, we compared performance with Trust Region Policy Optimization (TRPO; \citet{schulman2017trust}).  Both were configured as in \cite{RaAcAm19}.  As shown in Figure \ref{long_reward} and Appendix A.7, a single pessimistic objective ($\eta=0.5$) could be used to provide both higher total reward (sum of positive and negative terms) and lower cost than PPO and TRPO in five of the six environments.  A less aggressive weighting ($\eta=0.25$) provided these gains in the last environment.

\subsection{Comparisons with Constrained Baselines}
We additionally compared the performance of the constrained version of our approach (CRiSP) to RCPO using PPO and TRPO updates as well as First-Order Constrained Optimization in Policy Space (FOCOPS; \citet{zhang20}).  These baselines were found to be significantly stronger than the constrained methods explored in \cite{RaAcAm19}; see \cite{TeMaMa19} for a likely explanation. For all tasks, we chose the cost target to be half of what a trained, unconstrained agent unaware of penalties would accumulate.  

Results are given in Figure \ref{long_reward} and Appendix A.8.  Our pessimistic agents are seen to achieve higher positive rewards than all other methods, at the same cost levels, in all environments tested.  Only FOCOPS consistently allows lower learned penalty coefficients and higher policy entropies than CRiSP; however FOCOPS does not match the positive reward accumulation of our method.  In Appendix A.9, we show that $\eta$ may be increased to hasten convergence to the target cost (though this eventually reduces reward accumulation).

\section{Discussion}\label{discussion}
As shown above, pessimistic agents consistently achieve superior performance in our formulation. One contributing factor is their enhanced exploration, which reduces the chance of premature convergence to a suboptimal policy.
However this cannot be the only factor, as evidenced by the consistently higher entropy and lower performance of the FOCOPS algorithm (see Appendix A.8).  The fact that pessimistic agents emphasize poor outcomes likely plays a significant role, as it allows behavior to continually be adjusted most where it is most necessary.  Once a problematic part of the state space is addressed, a different region takes its place.  Conversely, optimistic weightings emphasize the best outcomes in the distribution.  Already strong outcomes are given increased attention, making them likely to stay on top.  Agents trained optimistically thus become myopic, obsessing over a fraction of the state space while neglecting the rest of it.

While $\eta \ge 0$ produced gains in all environments tested, it does represent an additional hyperparameter.  Our results suggest that a reasonable strategy for choosing it is to start at $\eta=0.5$ and proceed downward toward $\eta=0$ if needed.

Future directions based on these findings include an off-policy formulation (to improve sample efficiency) and additional experiments to clarify the impact of pessimistic weightings in discrete action spaces.


\section{Conclusions}\label{conclusions}
We formulated unconstrained and constrained learning based on a risk-sensitive policy gradient estimate.  Objectives that emphasize improvement where performance is poor produced performance gains in all environments tested.

\section{Acknowledgements}
This work was funded by the Johns Hopkins Institute for Assured Autonomy.  We would like to thank the AAAI reviewers for their constructive feedback.

\bibliography{main}

\section{Technical Appendix}
\subsection{A.1 Evaluation of cross-trajectory terms in policy gradient}\label{crossterm_proof}
In this section, we first show that the cross-trajectory terms in our policy gradient estimate (Equation 6) have an expectation value of $0$.  We then argue that their removal leads to a policy gradient estimate with reduced variance.
\begin{lemma}
Cross-trajectory terms of the form $f(\tau_i, \mathbf{a}_{j,t}, \mathbf{s}_{j,t})= u(r(\tau_i))\nabla_\theta\log\pi_\theta(\mathbf{a}_{j,t}|\mathbf{s}_{j,t})$, where $i \ne j$, do not contribute to the gradient estimate (Equation 6) in expectation.
\end{lemma}
\begin{proof}
First, note that
\begin{equation}
    \begin{split}
        &E_{\tau_i \sim p_\theta(\tau), \tau_j \sim p_\theta(\tau)}f(\tau_i, \mathbf{a}_{j,t}, \mathbf{s}_{j,t}) \\
        = &E_{\tau_i \sim p_\theta(\tau), \tau_j \sim p_\theta(\tau)} u(r(\tau_i))\nabla_\theta\log\pi_\theta(\mathbf{a}_{j,t}| \mathbf{s}_{j,t})\\\nonumber
        = &E_{\tau_i \sim p_\theta(\tau)} \Bigg[ u(r(\tau_i)) E_{\tau_j \sim p_\theta(\tau)} \bigg[ \nabla_\theta\log\pi_\theta(\mathbf{a}_{j,t}|\mathbf{s}_{j,t}) \bigg| \tau_i \bigg] \Bigg]\nonumber
    \end{split}
\end{equation}

Then consider the innermost expectation:

\begin{equation}
\small
\begin{split}
&E_{\tau_j \sim p_\theta(\tau)} \bigg[ \nabla_\theta\log\pi_\theta(\mathbf{a}_{j,t}| \mathbf{s}_{j,t}) \bigg| \tau_i \bigg]  \\\nonumber
= &\int_{\mathbf{s}_{j,t} , \mathbf{a}_{j,t}} p(\mathbf{s}_{j,t} , \mathbf{a}_{j,t} | \pi_\theta, \tau_i)\nabla_\theta\log\pi_\theta(\mathbf{a}_{j,t}|\mathbf{s}_{j,t}) d\mathbf{a}_{j,t} d\mathbf{s}_{j,t} \\\nonumber
= &\int_{\mathbf{s}_{j,t}} p(\mathbf{s}_{j,t} | \pi_\theta, \tau_i) \int_{\mathbf{a}_{j,t}} \pi_\theta(\mathbf{a}_{j,t}|\mathbf{s}_{j,t})\nabla_\theta\log\pi_\theta(\mathbf{a}_{j,t}|\mathbf{s}_{j,t})d\mathbf{a}_{j,t} d\mathbf{s}_{j,t} \\\nonumber
= &\int_{\mathbf{s}_{j,t}} p(\mathbf{s}_{j,t} | \pi_\theta, \tau_i) \int_{\mathbf{a}_{j,t}} \nabla_\theta\pi_\theta(\mathbf{a}_{j,t}|\mathbf{s}_{j,t})d\mathbf{a}_{j,t} d\mathbf{s}_{j,t}\\\nonumber
= &\int_{\mathbf{s}_{j,t}} p(\mathbf{s}_{j,t} | \pi_\theta, \tau_i) \nabla_\theta\int_{\mathbf{a}_{j,t}}  \pi_\theta(\mathbf{a}_{j,t}|\mathbf{s}_{j,t})d\mathbf{a}_{j,t} d\mathbf{s}_{j,t}\\\nonumber
= &\int_{\mathbf{s}_{j,t}} p(\mathbf{s}_{j,t} | \pi_\theta, \tau_i) (\nabla_\theta1) d\mathbf{s}_{j,t}  = 0.\nonumber
\label{crossterms}
 \end{split}
\end{equation}

\end{proof}

\begin{lemma}\label{lemma15}  The removal of cross-trajectory terms in the expression (Equation 6) leads to reduced variance in the policy gradient estimate. \end{lemma}

\begin{proof}
Consider that the full expression (Equation 6) may be written as the sum of terms of the form $f(\tau_i, \mathbf{a}_{j,t}, \mathbf{s}_{j,t})= u(r(\tau_i))\nabla_\theta\log\pi_\theta(\mathbf{a}_{j,t}, \mathbf{s}_{j,t})$.  Its variance is the sum of the total variance from terms where $i=j$, the total variance from terms where $i \ne j$, and a term proportional to the covariance of these two totals.  However, because each term in the covariance contains at least one trajectory that differs from the rest, the above reasoning may be applied to argue that the covariance is 0. Hence, the removal of the cross-trajectory terms lowers the variance of the policy gradient estimate by the variance of the cross-trajectory terms.
\end{proof}

\subsection{A.2 Introduction of Static and State-Dependent Baselines} \label{baseline_app}
\begin{lemma}\label{lemma2} A static baseline of the utility may be added to the policy gradient estimate (Equation 9) without introduction of bias.
\end{lemma}
\begin{proof}
First define
\begin{equation}
\small
W\bigg(\frac{i}{n}\bigg) \equiv \bigg( w'\bigg(\frac{i}{n} \bigg) + w'\bigg(\frac{i-1}{n} \bigg) \bigg).  \nonumber
\end{equation}
The additional term is $0$ in expectation as
\begin{equation}
\small
\begin{split}
    &E_{\tau_i \sim p_\theta(\tau)} \bigg[ b\bigg( w'\bigg(\frac{i}{n} \bigg) + w'\bigg(\frac{i-1}{n} \bigg) \bigg) \nabla_\theta \log \pi_\theta(\mathbf{a}_{i,t}| \mathbf{s}_{i,t}) \bigg] \\ \nonumber
    = &b W\bigg(\frac{i}{n}\bigg) \int_{\mathbf{s}_{i,t}, \mathbf{a}_{i,t}} p(\mathbf{s}_{i,t} , \mathbf{a}_{i,t} | \pi_\theta)\bigg[ \nabla_\theta \log \pi_\theta(\mathbf{a}_{i,t}| \mathbf{s}_{i,t}) \bigg]d\mathbf{a}_{i,t} d\mathbf{s}_{i,t}\\ \nonumber
    = &b W\bigg(\frac{i}{n}\bigg)\ \int_{\mathbf{s}_{i, t}} \!\!\!\!\! p(\mathbf{s}_{i,t}|\pi_\theta) \int_{\mathbf{a}_{i, t}} \!\!\!\!\! \pi_\theta(\mathbf{a}_{i,t}| \mathbf{s}_{i,t}) \nabla_\theta \log\pi_\theta(\mathbf{a}_{i,t}| \mathbf{s}_{i,t})d\mathbf{a}_{i,t} d\mathbf{s}_{i,t}\\ \nonumber
    = &b W\bigg(\frac{i}{n}\bigg) \int_{\mathbf{s}_{i, t}} p(\mathbf{s}_{i,t}|\pi_\theta) \nabla_\theta \int_{\mathbf{a}_{i, t}} \pi_\theta(\mathbf{a}_{i,t}| \mathbf{s}_{i,t})d\mathbf{a}_{i,t} d\mathbf{s}_{i,t}\\ \nonumber
    = &b W\bigg(\frac{i}{n}\bigg) \int_{\mathbf{s}_{i, t}} p(\mathbf{s}_{i,t}|\pi_\theta) (\nabla_\theta 1) d\mathbf{s}_{i,t}  = 0. \nonumber
\end{split}
\end{equation}
The contribution of the weight terms $(w'(\frac{i}{n}) + w'(\frac{i-1}{n}))$ may be pulled out of the integral between the first and second line because of its independence on both state and action.  This term is fixed for a given trajectory by the rank of its reward amongst the rewards accumulated on all trajectories in the current batch.

\end{proof}
In our variance reduction experiments (Appendix A.5), the ``Base" agent uses $b$ equal to the mean of full-episode utility in the current batch. 

As described in the main text, we may further adjust the policy gradient estimate through introduction of per-step utilities.  In this case, we may justify the use of a state-dependent baseline through the following.
\begin{lemma}\label{lemma3} A state-dependent baseline $V_\phi(\mathbf{s}_{i,t})$ may be added to the policy gradient estimate (Equation 9) without introduction of bias, if per-step utilities are assumed.
\end{lemma}
\begin{proof}
As above, define
\begin{equation}
\small
W\bigg(\frac{i}{n}\bigg) \equiv \bigg( w'\bigg(\frac{i}{n} \bigg) + w'\bigg(\frac{i-1}{n} \bigg)\bigg). \nonumber
\end{equation}

The additional term is $0$ in expectation as
\begin{equation}
\small
\begin{split}
    &E_{\tau_i \sim p_\theta(\tau)} \bigg[\bigg( w'\bigg(\frac{i}{n} \bigg) + w'\bigg(\frac{i-1}{n} \bigg) \bigg) \nabla_\theta \log \pi_\theta(\mathbf{a}_{i,t}| \mathbf{s}_{i,t})V_\phi(\mathbf{s}_{i,t}) \bigg] \\\nonumber
    = &W\bigg(\frac{i}{n}\bigg) \int_{\mathbf{s}_{i,t}, \mathbf{a}_{i,t}} p(\mathbf{s}_{i,t} , \mathbf{a}_{i,t} | \pi_\theta) V_\phi(\mathbf{s}_{i,t}) \nabla_\theta \log \pi_\theta(\mathbf{a}_{i,t}| \mathbf{s}_{i,t}) d\mathbf{a}_{i,t} d\mathbf{s}_{i,t}\\ \nonumber
    = &W\bigg(\frac{i}{n}\bigg) \int_{\mathbf{s}_{i, t}} \!\!\!\!\! p(\mathbf{s}_{i,t}|\pi_\theta) V_\phi(\mathbf{s}_{i,t})  \int_{\mathbf{a}_{i, t}} \!\!\!\!\! \pi_\theta(\mathbf{a}_{i,t}| \mathbf{s}_{i,t}) \nabla_\theta \log\pi_\theta(\mathbf{a}_{i,t}| \mathbf{s}_{i,t})d\mathbf{a}_{i,t} d\mathbf{s}_{i,t}\\\nonumber
    = &W\bigg(\frac{i}{n}\bigg) \int_{\mathbf{s}_{i, t}} p(\mathbf{s}_{i,t}|\pi_\theta) V_\phi(\mathbf{s}_{i,t}) \nabla_\theta \int_{\mathbf{a}_{i, t}} \pi_\theta(\mathbf{a}_{i,t}| \mathbf{s}_{i,t})d\mathbf{a}_{i,t} d\mathbf{s}_{i,t} \\ \nonumber
     = &W\bigg(\frac{i}{n}\bigg) \int_{\mathbf{s}_{i, t}} p(\mathbf{s}_{i,t}|\pi_\theta) V_\phi(\mathbf{s}_{i,t}) (\nabla_\theta 1) d\mathbf{s}_{i,t} = 0. \nonumber
\end{split}
\end{equation}
The rationale for pulling the $w'$ terms out of the integral is the same as in Lemma \ref{lemma2}.
\end{proof}
Finally, we note that the ability to pull the contribution of the weight terms $(w'(\frac{i}{n}) + w'(\frac{i-1}{n}))$ to the front of Equation 11 allows us to formulate advantage estimates based on per-step utility.  Bootstrap estimates of the value function $V_\phi(\mathbf{s}_{i,t})$ and Generalized Advantage Estimation  as in \cite{ScMoLeJoAb16} can be conducted exactly as they are in standard on-policy learning, if rewards are replaced by per-step utilities.

\subsection{A.3 Cumulative Prospect Proximal Policy Optimization}\label{c3po_app}
Algorithm~\ref{alg:c3po} provides pseudocode for applying our risk-sensitive policy gradient in an unconstrained setting.  It is a special case of the more general, constrained method.

\begin{algorithm}
 \caption{Cumulative Prospect PPO (C3PO)}
 \begin{algorithmic}[1]
\REQUIRE Policy: initial parameters $\theta_0$, learning rate $\alpha_\theta$, updates per batch $M_\theta$
\REQUIRE Value: initial parameters $\phi_0$, learning rate $\alpha_\phi$, updates per batch $M_\phi$
\REQUIRE Stopping threshold $D_{\text{KL, stop}}$, discount factor $\gamma$
\FOR{$k = 0, 1, 2, \ldots$}
\STATE Collect set of episodes $\mathcal{D}_k = \{ \tau_i \}$ by running policy $\pi(\theta_k)$ in the environment
\STATE Compute discounted utilities-to-go: \\
\begin{center}
$\hat{u}(\mathbf{s}_{i, t}, \mathbf{a}_{i, t}) = \sum_{t'=t}^{T_i}\gamma^{t'-t}u(\mathbf{s}_{i, t'}, \mathbf{a}_{i, t'})$
\end{center}
\STATE Fit value function ($M_\phi$ steps): \\
\begin{center}
\small
$\phi_{k+1} = \text{arg}\min_{\phi} \frac{1}{\sum_iT_i} \sum_{i,t} \bigg( V_\phi(\mathbf{s}_{i, t}) - \hat{u}(\mathbf{s}_{i, t}, \mathbf{a}_{i, t}))  \bigg) ^2$
\end{center}
\STATE Update utility advantage estimates $A_u^\pi(\mathbf{s}, \mathbf{a})$ using new $V_\phi(\mathbf{s})$
\STATE Compute weight coefficients based on ordered full-episode rewards
\STATE Update policy using clipped-action policy gradient correction over $M_\theta$ steps with KL-based early stopping (threshold $D_{\text{KL, stop}}$): \\
\begin{center}
$\theta_{k+1} = \text{arg}\max_{\theta} \bigg(\frac{1}{N}\sum_{i=1}^N  \left(w'(\frac{i}{N})+w'(\frac{i-1}{N})\right) * $ \\
$\sum_{t=1}^{T_i}\nabla_\theta L_{\text{clip}}(\log \pi_\theta(\mathbf{a}_{i,t}| \mathbf{s}_{i,t}), A_u^\pi(\mathbf{s}_{i,t}, \mathbf{a}_{i,t})) \bigg)$
\end{center}
\ENDFOR
\end{algorithmic}
\label{alg:c3po}
\end{algorithm}

\subsection{A.4 Additional Information on Safety Gym}\label{safety_gym_app}
As mentioned in the Experiments section of the main text, we chose to evaluate our approach using the OpenAI Safety Gym \citep{RaAcAm19}.  The choice was governed by our desire to test in conditions with clear cost-benefit trade-offs, significant stochasticity, adequate complexity, and available benchmarks.  While our methods are not limited to particular task types or observation/action spaces, we found Safety Gym to be suitable for exploring their potential.  

The six environments chosen were the most obstacle-rich of the publicly available environments that used the ``Point'' and ``Car'' robots.  The Point robot is constrained to the 2D plane and has two control dimensions: one for moving forward/backward and one for turning. The Car robot also has two control dimensions, corresponding to independently actuated parallel wheels. It has a freely rotating wheel and, while it is not constrained to the 2D plane, typically remains in it.  While we expect our results to extend to the remaining default robot, ``Doggo'', we did not experiment with it because of the order of magnitude longer training times it exhibited in \cite{RaAcAm19}.
Several types of obstacles and tasks were present in the environments we evaluated.  In all cases, the robot is given a fixed amount of time (1000 steps) to complete the prescribed task as many times as possible and is motivated by both sparse and dense reward contributions.  In the ``Goal'' environments, the robot must navigate to a series of randomly-assigned goal positions, with a new target being assigned as soon as a goal is reached.  In the ``Button'' environments, the robot must reach and press a sequence of goal buttons while avoiding other buttons.  In the ``Push'' task, the robot must push a box to a series of goal positions.  The set of obstacles are different for each task; among the three environments there are a total of five different constraint elements (hazards, vases, incorrect buttons, pillars, and gremlins), each with different dynamics.  See \cite{RaAcAm19} for further details.

All of our experiments used a single indicator for overall cost at each time step (the OpenAI default).  In the unconstrained experiments, each cost event was assigned a fixed (negative) weight in the reward function.  For Push environments, this was 0.025.  For Button environments, it was 0.05.  For Goal environments, it was 0.075.  These choices were made to illustrate reasonable learning for the three different task types.

\subsection{A.5 Empirical Performance of Variance Reduction and Regularization Measures}\label{var_red_text}
To gauge the impact of the variance reduction and regularization techniques discussed in the main text, we evaluated their performance in maximizing the value function of Cumulative Prospect Theory \citep{TvKa92}.  This function has two integrals of the form used in Equation 2:

\begin{equation}
\begin{split}
J&_{CPT}(\theta) = -\int_{-\infty}^{\infty} u^-(r(\tau)) \frac{d}{dr(\tau)}\bigg( w^-(P_\theta(r(\tau)))\bigg)dr(\tau)\\
 + &\int_{-\infty}^{\infty} u^+(r(\tau))\frac{d}{dr(\tau)}\bigg( -w^+(1-P_\theta(r(\tau)))\bigg) dr(\tau)\nonumber
\label{cpt}
\end{split}
\end{equation}

In \cite{TvKa92}, the utility functions are computed relative to a reference point and reflect the tendency of humans to be more risk-averse in the presence of gains than in the presence of losses.  The weight functions $\{w^+, w^-\}$ model our inclination to emphasize the best and worst possible outcomes in our decision-making.

\begin{figure*}[h]
    \centering
    \includegraphics[width=0.234\textwidth]{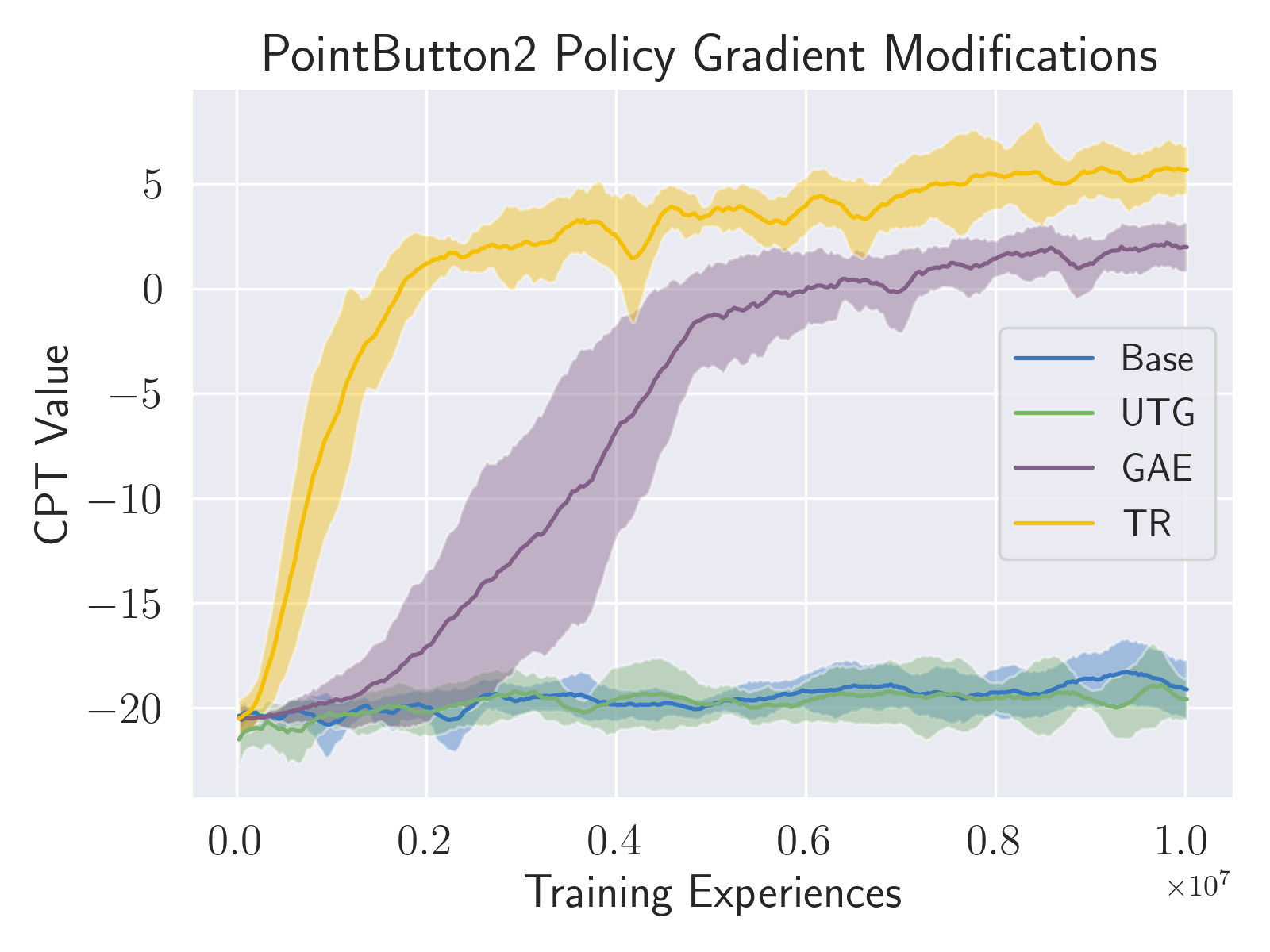}
    \includegraphics[width=0.234\textwidth]{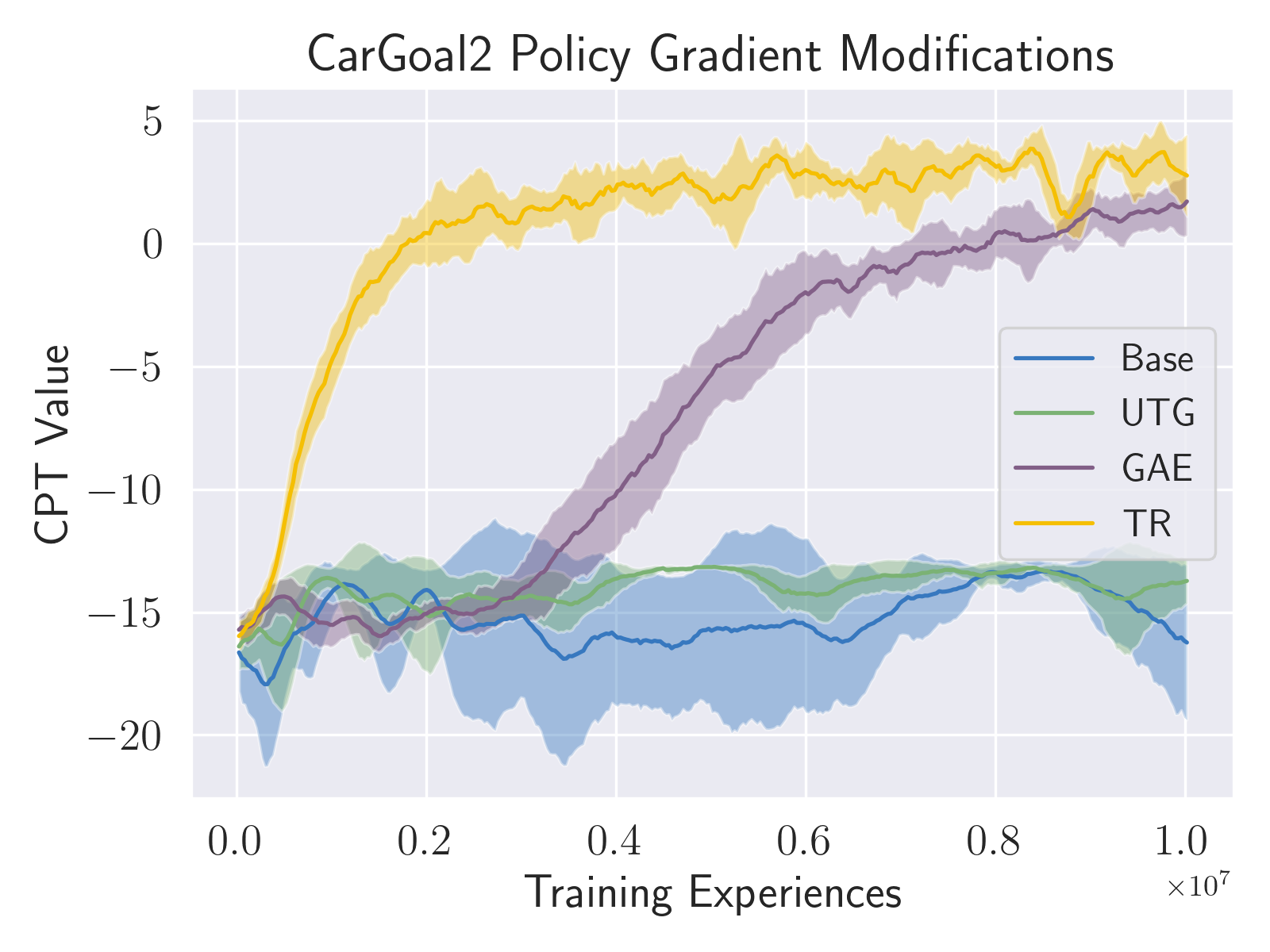}
    \includegraphics[width=0.234\textwidth]{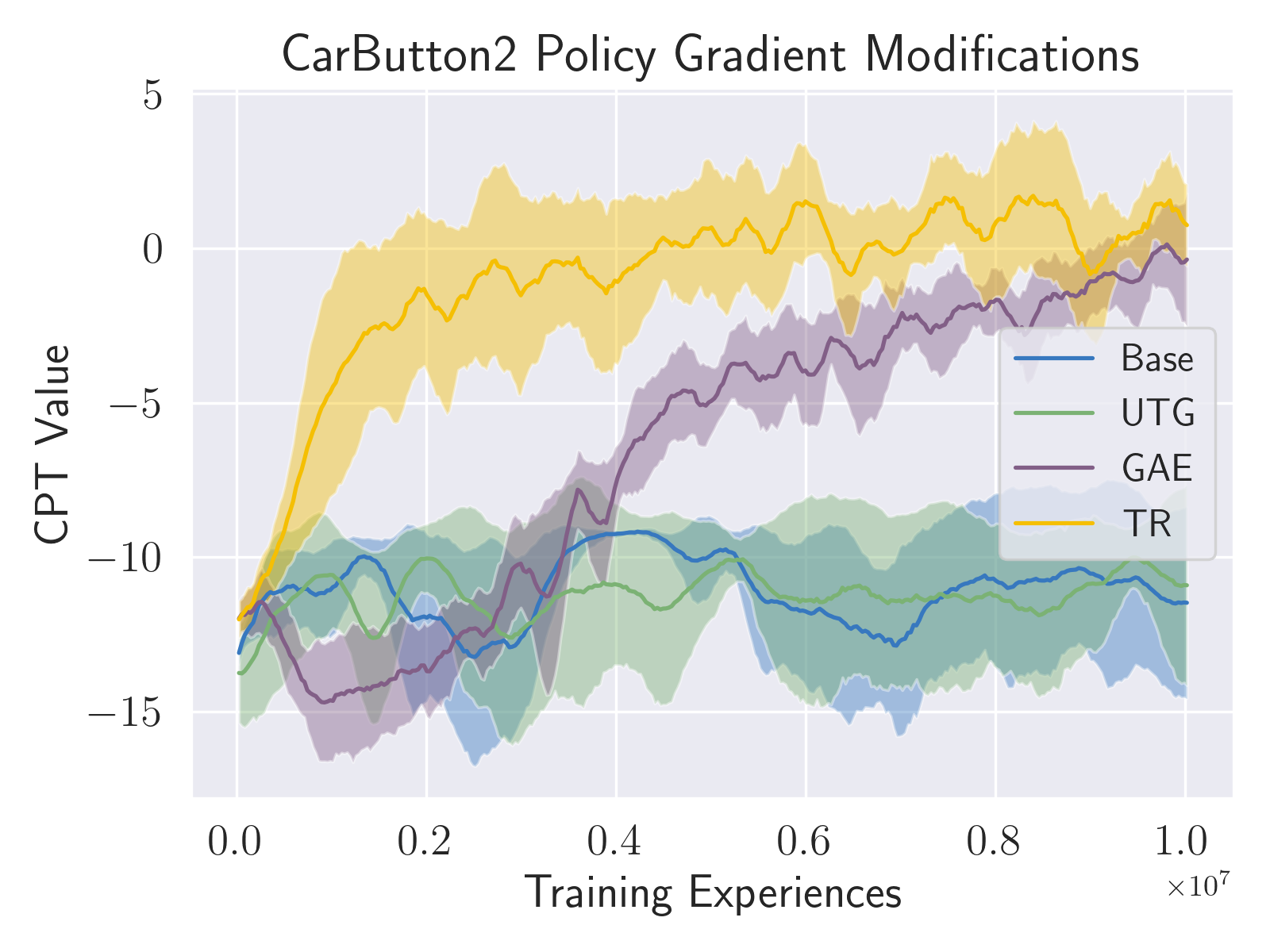}
    \caption{Impact of variance reduction and regularization measures on optimization of the CPT value function.  Here ``Base'' refers to the risk-sensitive policy gradient estimate (10), ``UTG'' adds utility-to-go and a neural network baseline (11), ``GAE'' incorporates generalized advantage estimation, and ``TR'' implements trust regions via clipping.  Shading represents the variation over five random seeds.}
    \label{varRed}
\end{figure*}

\begin{table*}[bp]
\centering
\begin{tabular}{|l|clll|}
\hline
\multicolumn{1}{|c|}{}     & \multicolumn{4}{c|}{\textbf{PointGoal2}}                                                                                                                                                             \\
                           & \textbf{Expected Reward}                        & \multicolumn{1}{c}{\textbf{CPT Value}}         & \multicolumn{1}{c}{\textbf{Wang($\eta=-0.5$)}}  & \multicolumn{1}{c|}{\textbf{Wang($\eta=0.5$)}}  \\ \hline
\textbf{Expected Reward}   & 15.6 $\pm$ 0.3                                  & 3.7 $\pm$ 0.3                                  & 18.5 $\pm$ 0.3                                  & 12.4 $\pm$ 0.4                                  \\
\textbf{CPT Value}         & 15.2 $\pm$ 1.4                                  & 3.3 $\pm$ 1.4                                  & 17.9 $\pm$ 1.3                                  & 12.1 $\pm$ 1.6                                  \\
\textbf{Wang($\eta=-0.5$)} & 11.0 $\pm$ 3.2                                  & -0.3 $\pm$ 3.1                                 & 14.0 $\pm$ 3.0                                 & 7.8 $\pm$ 3.3                                   \\
\textbf{Wang ($\eta=0.5$)} & \textcolor{blue}{\textbf{17.9 $\pm$ 0.4}} & \textcolor{blue}{\textbf{5.7 $\pm$ 0.5}} & \textcolor{blue}{\textbf{20.2 $\pm$ 0.5}} & \textcolor{blue}{\textbf{15.3 $\pm$ 0.5}} \\ \hline
                           & \multicolumn{4}{c|}{\textbf{CarButton2}}                                                                                                                                                             \\
                           & \textbf{Expected Reward}                        & \multicolumn{1}{c}{\textbf{CPT Value}}         & \multicolumn{1}{c}{\textbf{Wang($\eta=-0.5$)}}  & \multicolumn{1}{c|}{\textbf{Wang($\eta=0.5$)}}  \\ \hline
\textbf{Expected Reward}   & 11.5 $\pm$ 1.2                                  & 0.5 $\pm$ 1.5                                  & 16.0 $\pm$ 0.7                                  & 6.4 $\pm$ 1.7                                   \\
\textbf{CPT Value}         & 11.8 $\pm$ 1.2                                  & 1.2 $\pm$ 1.0                                  & 15.9 $\pm$ 1.0                                  & 7.1 $\pm$ 1.4                                   \\
\textbf{Wang($\eta=-0.5$)} & 8.7 $\pm$ 1.0                                   & -2.6 $\pm$ 1.7                                 & 13.8 $\pm$ 0.9                                  & 2.8 $\pm$ 1.8                                   \\
\textbf{Wang ($\eta=0.5$)} & \textcolor{blue}{\textbf{16.7 $\pm$ 0.5}} & \textcolor{blue}{\textbf{5.2 $\pm$ 0.8}} & \textcolor{blue}{\textbf{20.7 $\pm$ 0.4}} & \textcolor{blue}{\textbf{12.2 $\pm$ 0.7}} \\ \hline
                           & \multicolumn{4}{c|}{\textbf{PointPush2}}                                                                                                                                                             \\
                           & \textbf{Expected Reward}                        & \multicolumn{1}{c}{\textbf{CPT Value}}         & \multicolumn{1}{c}{\textbf{Wang($\eta=-0.5$)}}  & \multicolumn{1}{c|}{\textbf{Wang($\eta=0.5$)}}  \\ \hline
\textbf{Expected Reward}   & 4.1 $\pm$ 0.6                                   & 1.0 $\pm$ 0.5                                  & 6.1 $\pm$ 0.8                                   & 2.0 $\pm$ 0.6                                   \\
\textbf{CPT Value}         & 3.9 $\pm$ 1.5                                   & 0.9 $\pm$ 1.2                                  & 5.5 $\pm$ 1.9                                   & 2.2 $\pm$ 1.2                                   \\
\textbf{Wang($\eta=-0.5$)} & 3.2 $\pm$ 2.1                                   & -0.2 $\pm$ 1.5                                 & 5.2 $\pm$ 2.7                                   & 1.0 $\pm$ 1.6                                   \\
\textbf{Wang ($\eta=0.5$)} & \textcolor{blue}{\textbf{9.3 $\pm$ 2.1}}  & \textcolor{blue}{\textbf{4.5 $\pm$ 1.3}} & \textcolor{blue}{\textbf{11.8 $\pm$ 2.4}} & \textcolor{blue}{\textbf{6.6 $\pm$ 1.7}}  \\ \hline
\end{tabular}
\caption{Testing performance in terms of different methods for different training objectives.  Rows reflect training objectives and columns are statistics for metrics computed on test distributions for each of the 5 networks trained as described in the ``Differing Objectives'' section of the main text.  For each metric in each environment tested, the ``pessimistic'' weighting (Wang $\eta=0.5$) outperformed others.}
\centering
\end{table*}

More specifically, in these experiments we used the piecewise utility functions $u^+(r)=H(r-r_0)(r-r_0)^\sigma $ and $u^-=\lambda H(r_0-r)(r_0-r)^\sigma $ with static reference $r_0=10$, $\sigma=0.88$, and $\lambda=2.25$.  The weight function $w(p) = \frac{p^\eta}{(p^{\eta} + (1-p)^{\eta})^{\frac{1}{\eta}}}$ was used, where $\eta = 0.61$ for $r < r_0$ and $\eta = 0.69$ for $r \ge r_0$. The cost weight in each environment's reward function was fixed at 0.05, and the clipped action policy gradient correction \cite{pmlr-v80-fujita18a} was not used.

Four methods were evaluated:

\begin{itemize}
    \item \textbf{Base}: Risk-sensitive policy gradient with a static baseline (Equation 10)
    \item \textbf{UTG}: Base with utility-to-go and a neural network baseline (Equation 11)
    \item \textbf{GAE}: UTG with generalized advantage estimation (Equation 13 without clipping)
    \item \textbf{TR}: GAE with trust regions (Equation 13 with clipping (Equation 14))
\end{itemize}

As shown in Figure \ref{varRed}, the incorporation of these techniques increased the sample efficiency of the CPT value optimization significantly.  Consequently, we used the full complement (TR) in all other experiments.

\subsection{A.6 Weighting Effect on Metric Optimization}\label{diff_obj}

In Table 1, we display computed metrics of the reward distribution in testing for each of the objectives evaluated in our experiments.  We ran 1000 episodes for each of the 5 random seeds trained, with the numbers below reflecting statistics of the different objectives computed over those 5 evaluations.  In all cases, sampling was turned off- the agent simply chose the action corresponding to the max of its policy distribution.  In all cases, the agents with ``pessimistic'' risk profiles outperformed all others.

\subsection{A.7 Additional Comparisons with Unconstrained Methods}\label{unconstrained_app}
Figure 2 shows plots of average episode reward and average number of episode cost events throughout training for the remainder of the environments on which we conducted unconstrained runs.

\begin{figure}
    \centering
    \includegraphics[width=0.234\textwidth]{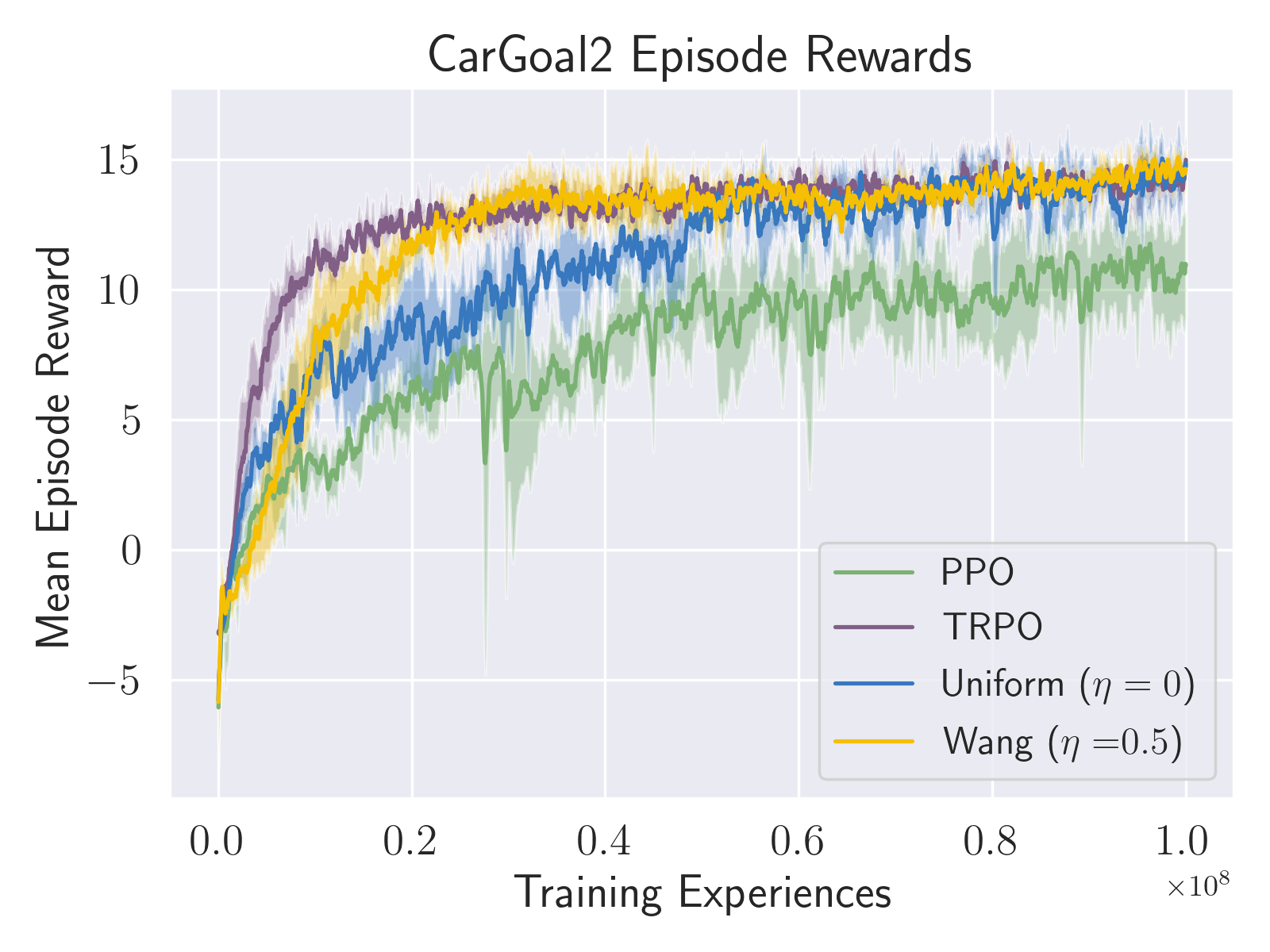}
    \includegraphics[width=0.234\textwidth]{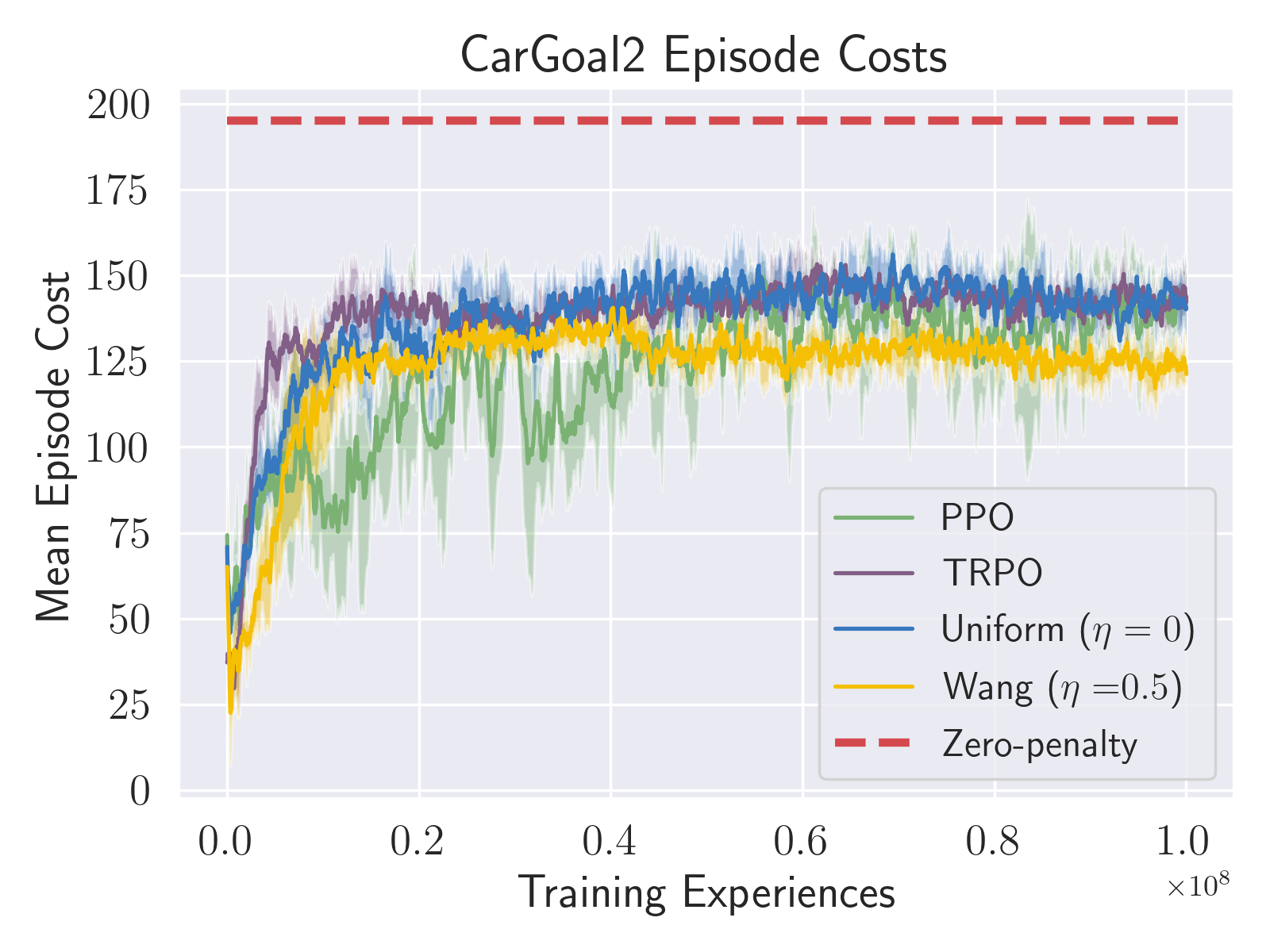}
    \includegraphics[width=0.234\textwidth]{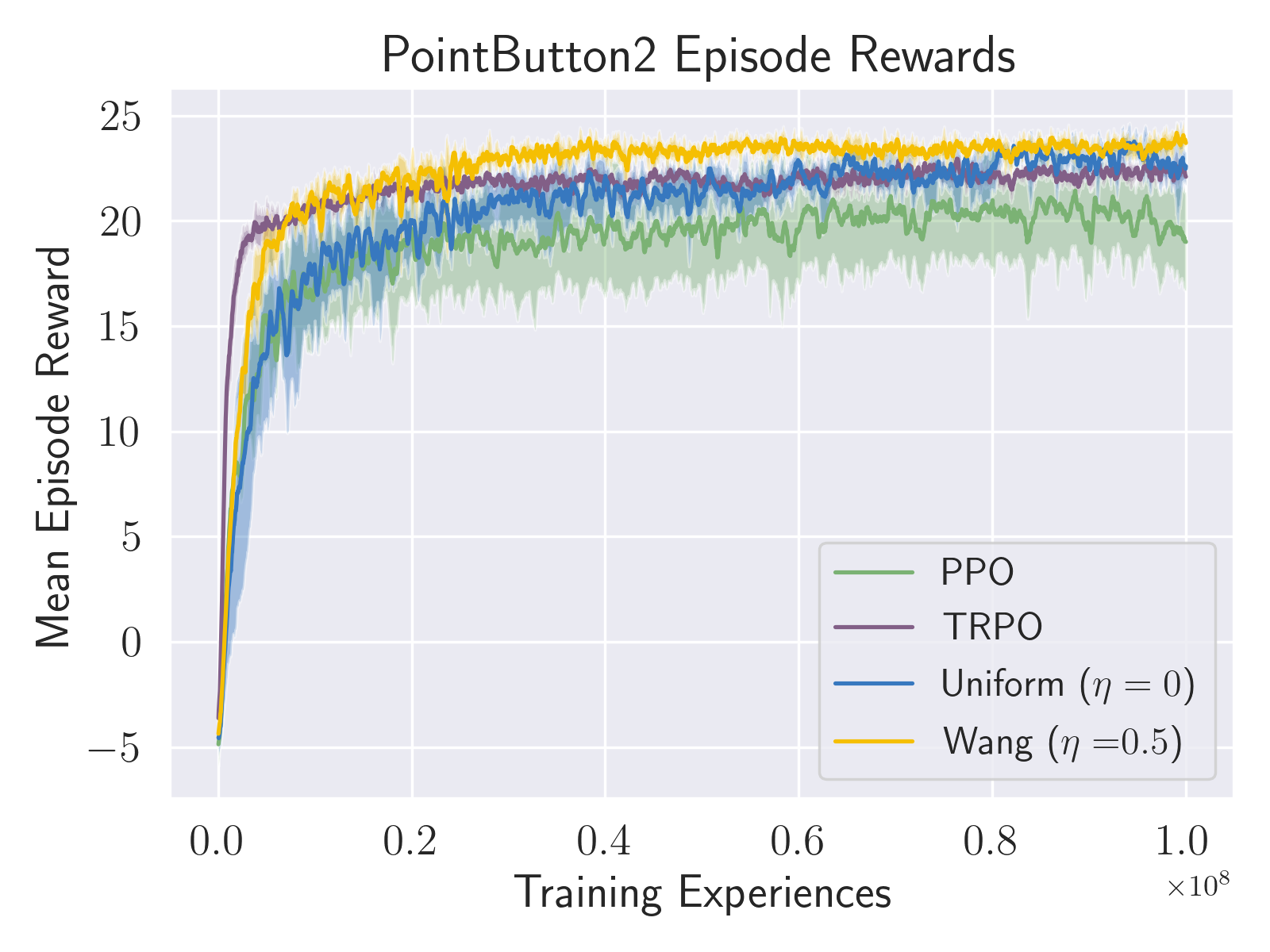}
    \includegraphics[width=0.234\textwidth]{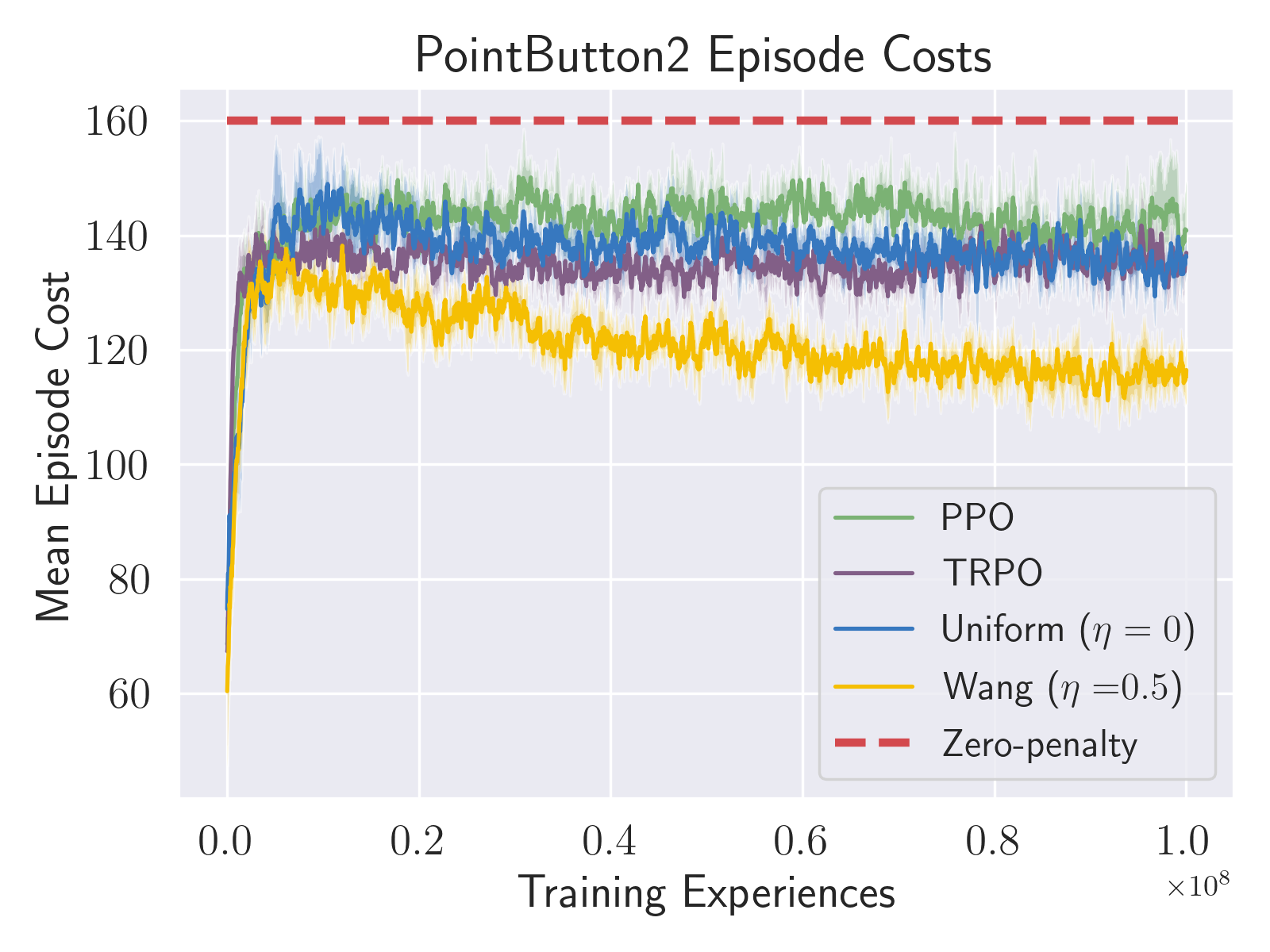}
    \includegraphics[width=0.234\textwidth]{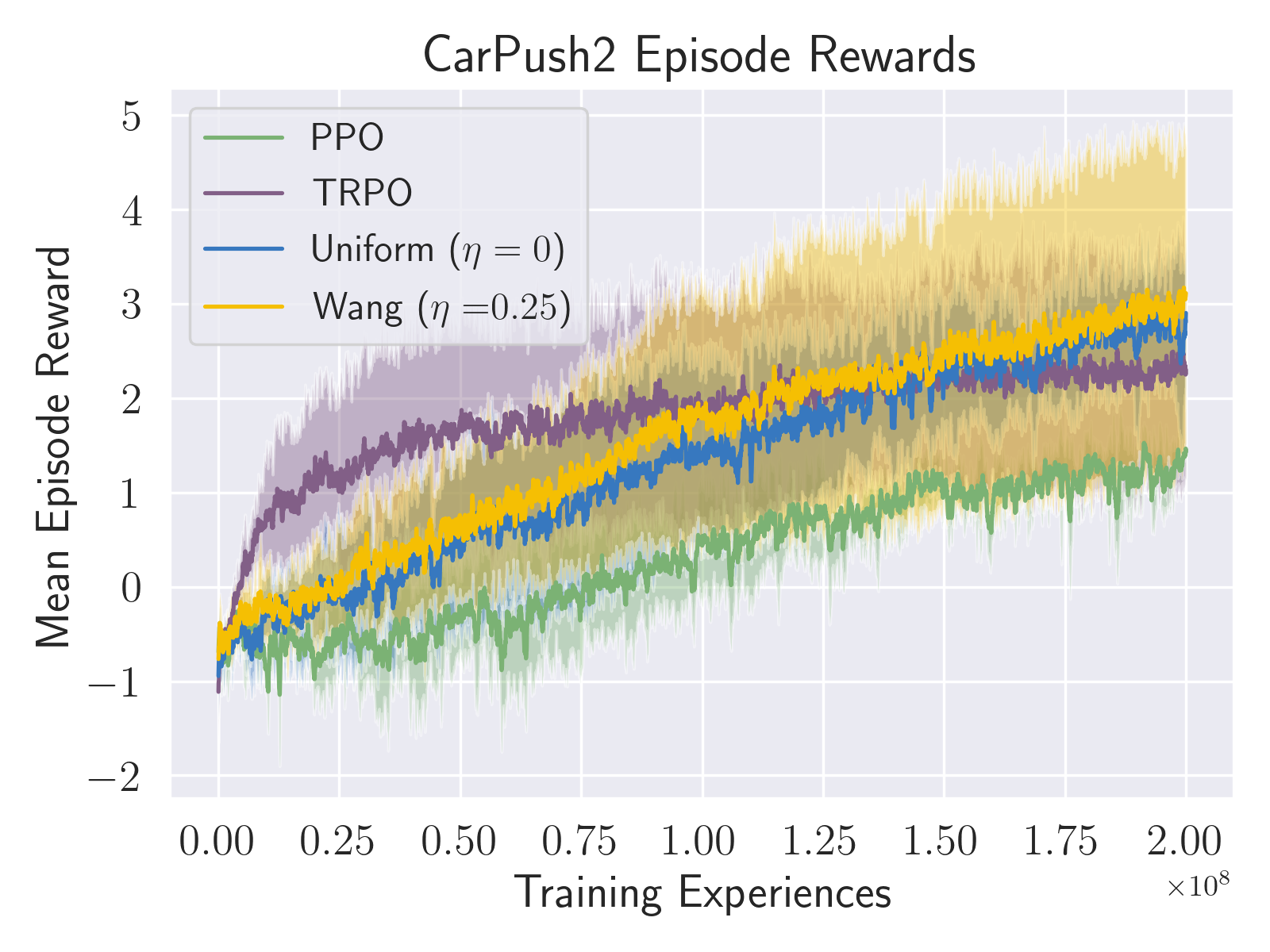}
    \includegraphics[width=0.234\textwidth]{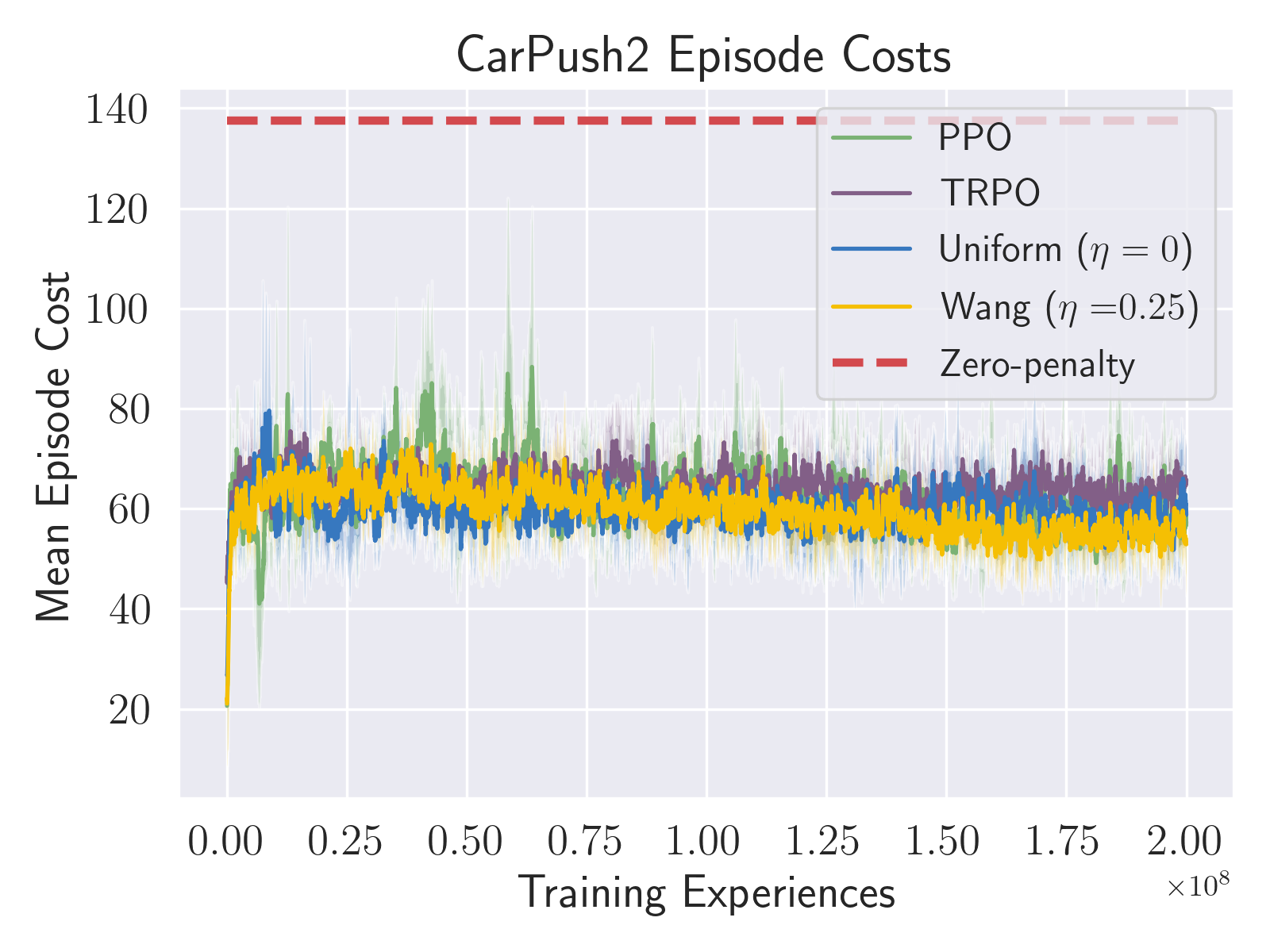}
    \centering
    \caption{Left column: Average episode reward (including penalty) over training for different unconstrained learning approaches in remaining three environments. Right column: Average number of cost events per episode (lower is better) over training for different unconstrained learning approaches in remaining three environments.}
    \label{long_cost_rest}
\end{figure}

\subsection{A.8 Additional Comparisons with Constrained Methods}\label{constrained_app}

In Figure 3 we include additional plots comparing our constrained, risk-sensitive approach to baselines.  We also include plots of the penalty weight (Figure 4), policy entropy (Figure 5), and truncated (taking into account control bounds) policy entropy (Figure 6) for each method and task.

As stated in the main text, the chosen baseline methods were found to significantly outperform the baselines used in \cite{RaAcAm19}.  PPO-Lagrangian and TRPO-Lagrangian, as formulated in the OpenAI safety-starter-agents, were found to be oscillatory for all but very slow learning rates of the Lagrange multiplier and/or very low cost limits. In the former case, agents were not able to reach the prescribed cost levels in a competitive amount of time.  We also ran
experiments with Constrained Policy Optimization (CPO; \cite{AcHeTaAb17}) using the OpenAI safety-starter-agents implementation.  However we decided to not include them because, as observed in \cite{RaAcAm19}, CPO did not adhere to the provided cost limits.

Our configurations of RCPO-PPO, RCPO-TRPO, and FOCOPS were chosen to provide stable learning that converged to the target cost level in a reasonable amount of time.  Our implementation of FOCOPS differed from the original \cite{zhang20} in that we used the Adam optimizer to update the Lagrange multiplier.  We found this change to significantly stabilize the learning, improving the strength of the baseline.

\begin{figure}
    \centering
    \includegraphics[width=0.234\textwidth]{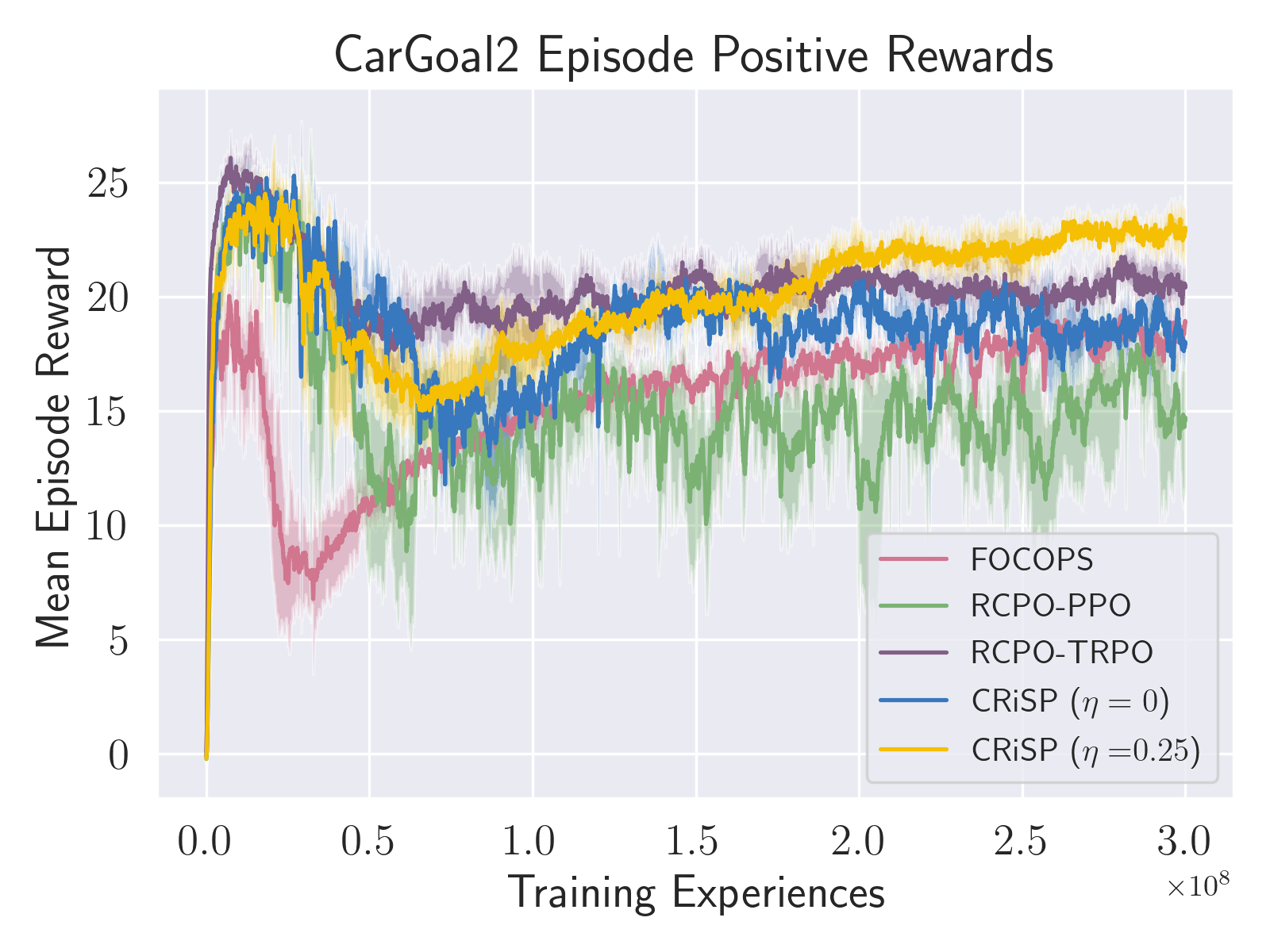}
    \includegraphics[width=0.234\textwidth]{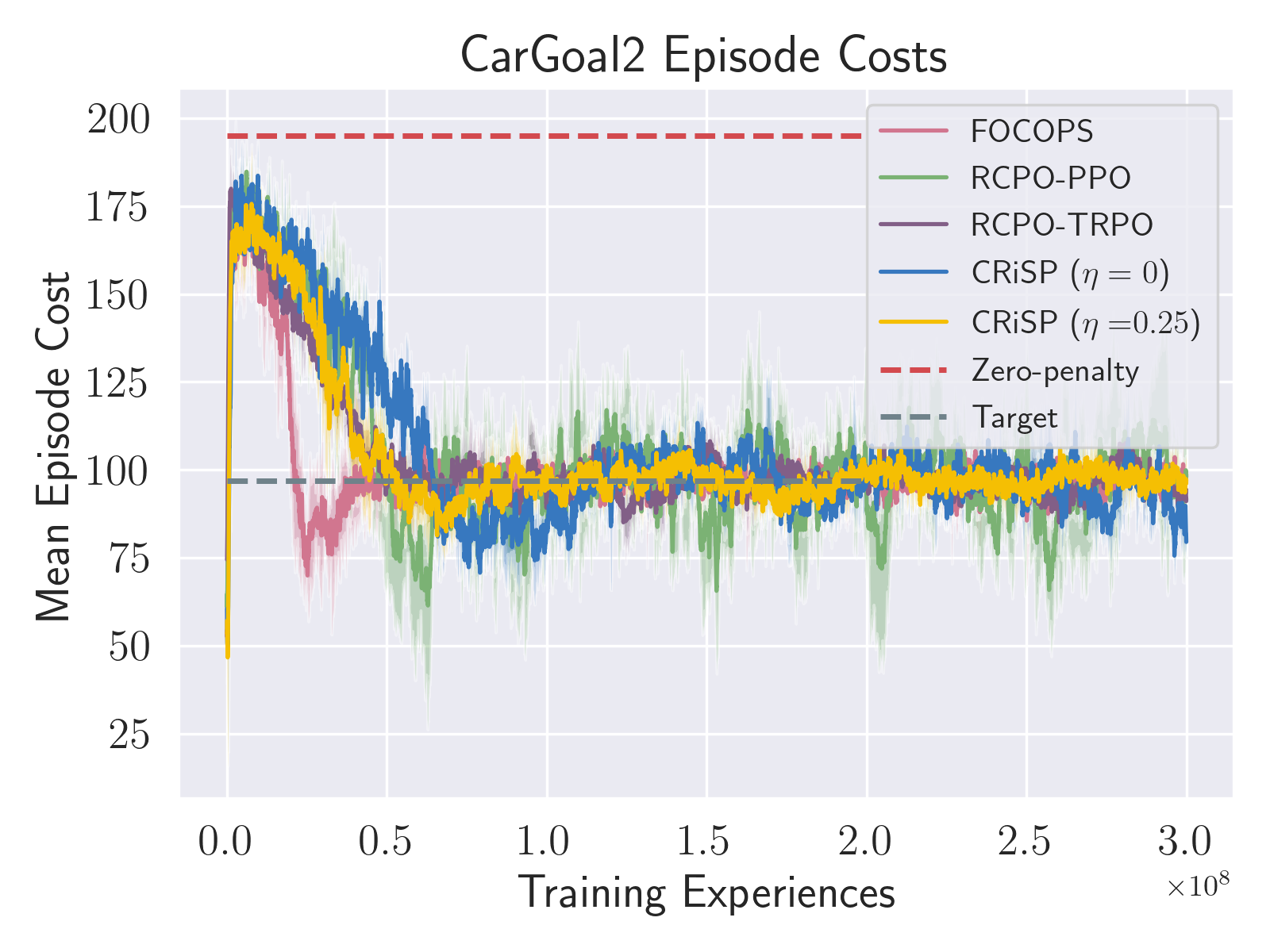}
    \includegraphics[width=0.234\textwidth]{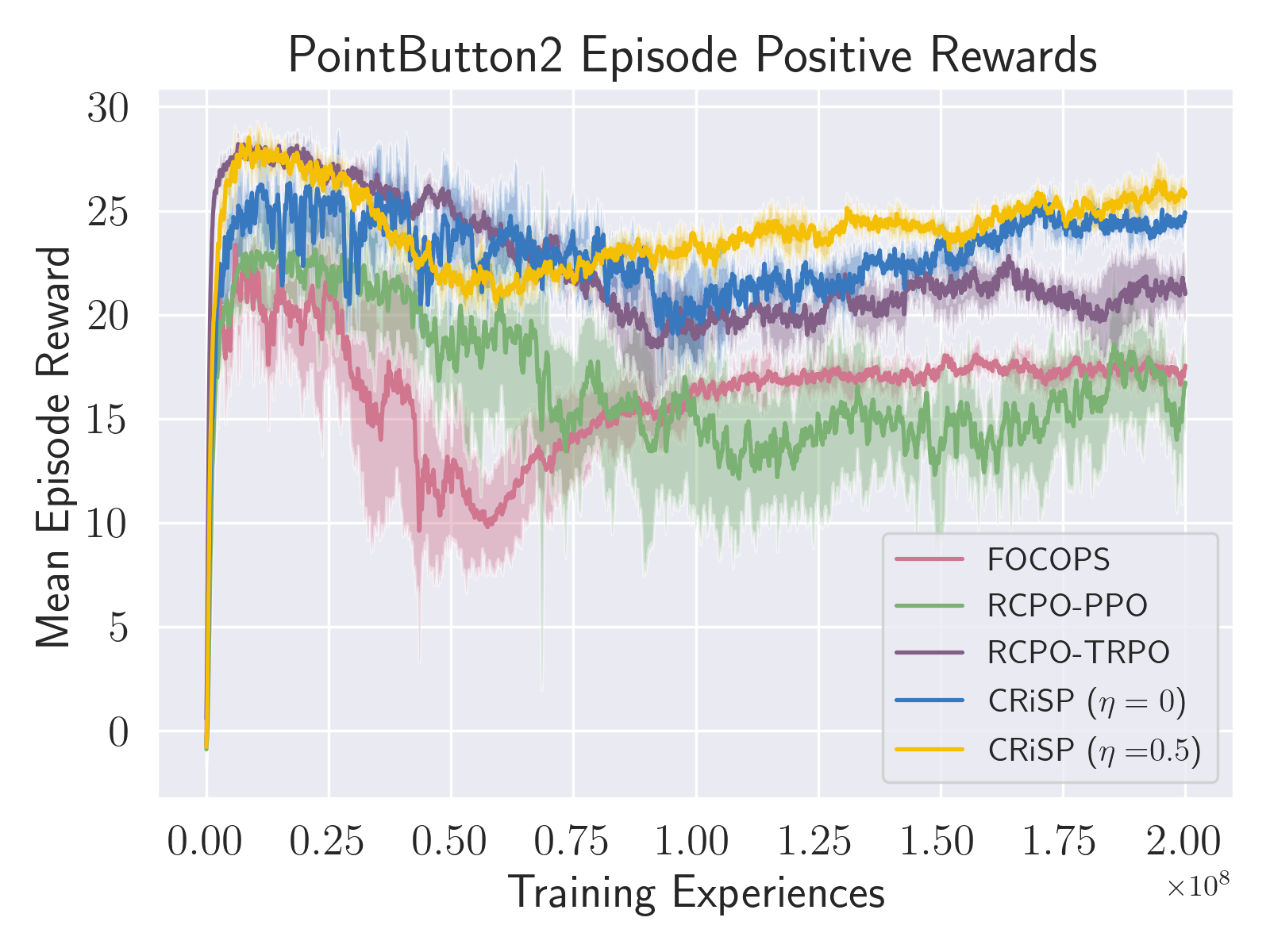}
    \includegraphics[width=0.234\textwidth]{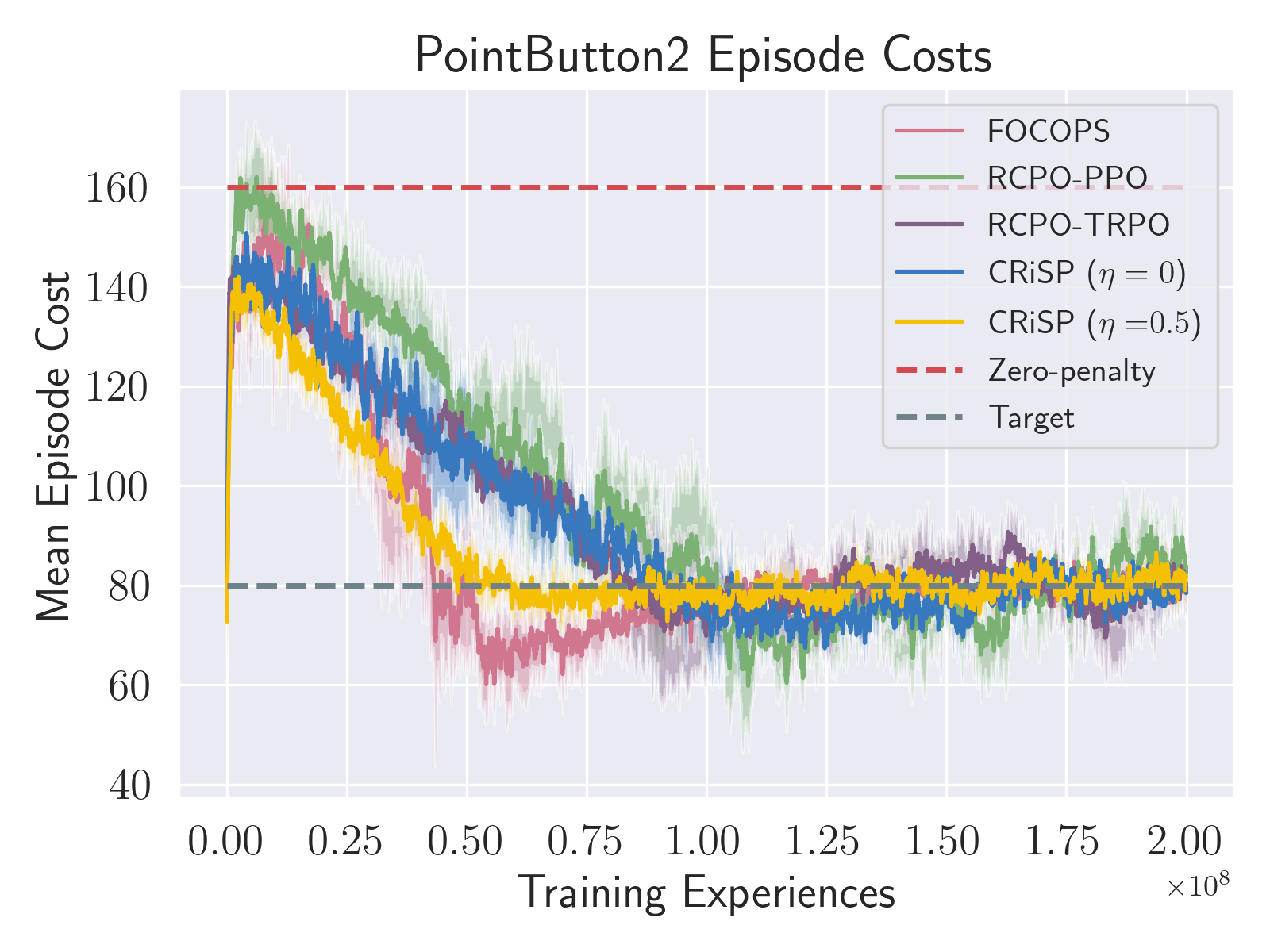}
    \includegraphics[width=0.234\textwidth]{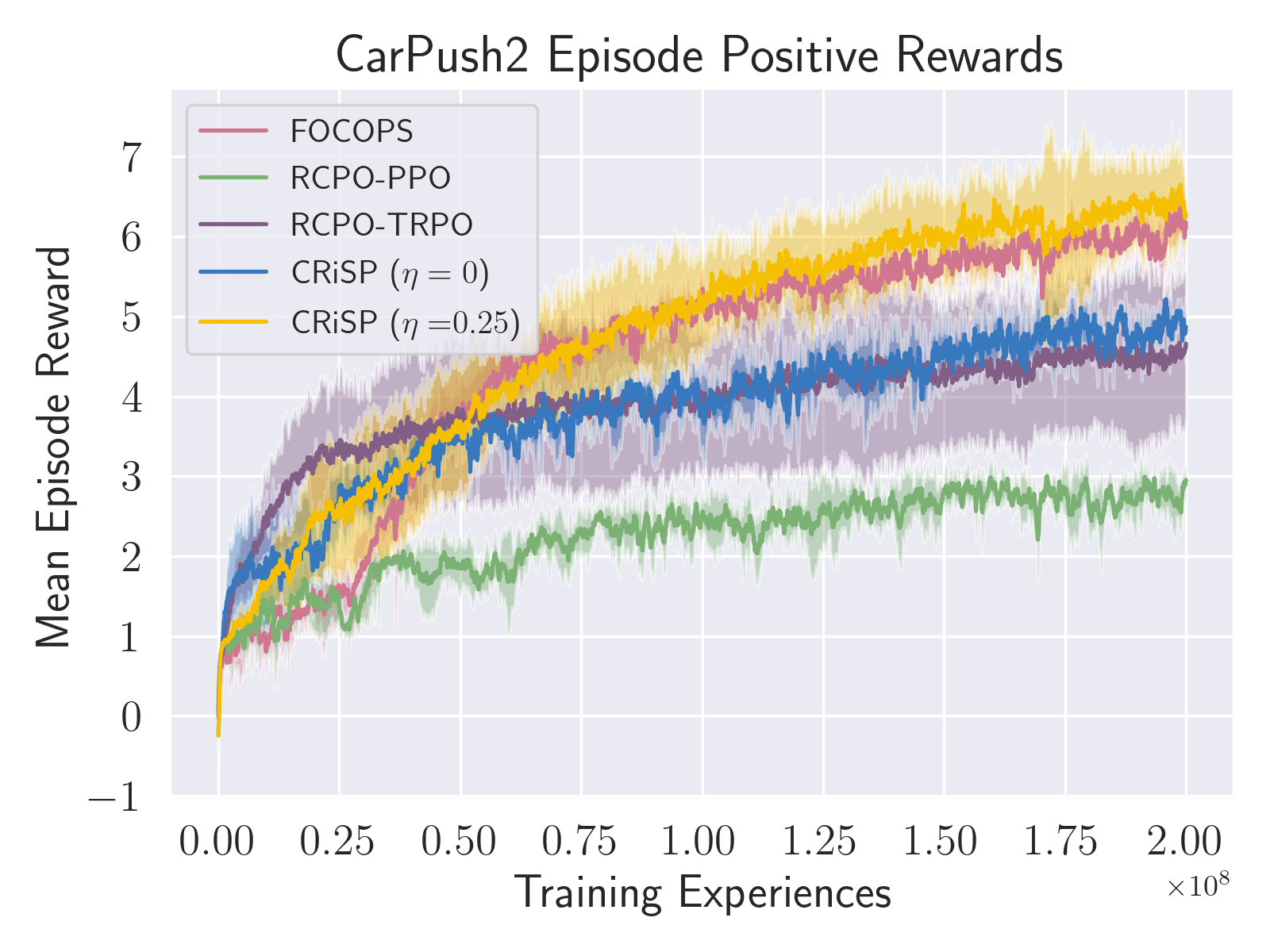}
    \includegraphics[width=0.234\textwidth]{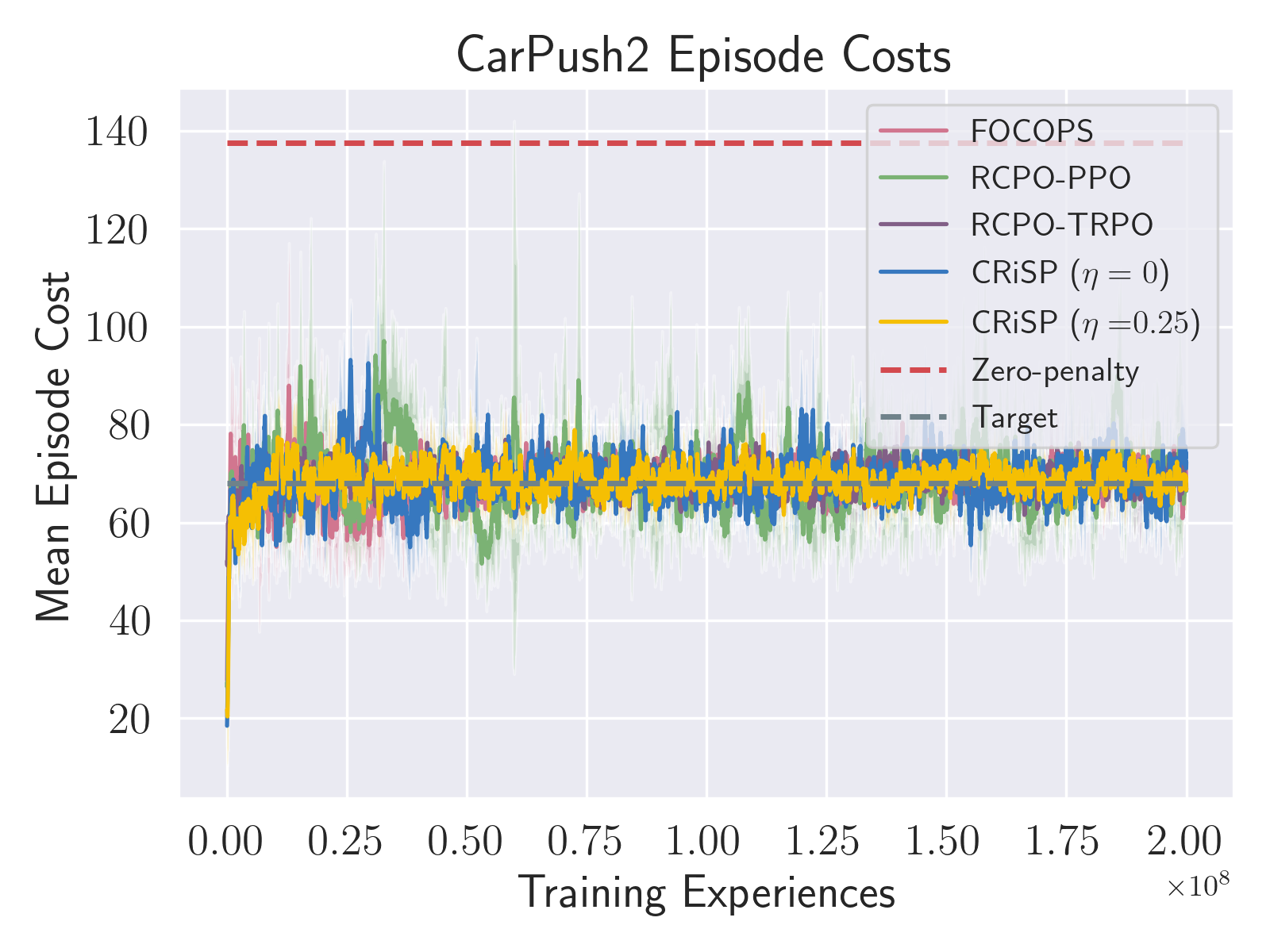}
     \caption{Our constrained, pessimistic, risk-sensitive method (CRiSP; yellow) accumulates larger positive rewards (left column) at the prescribed cost levels (right column) than other methods.}
    \label{constr2}
\end{figure}

\begin{figure}
    \centering
    \includegraphics[width=0.234\textwidth]{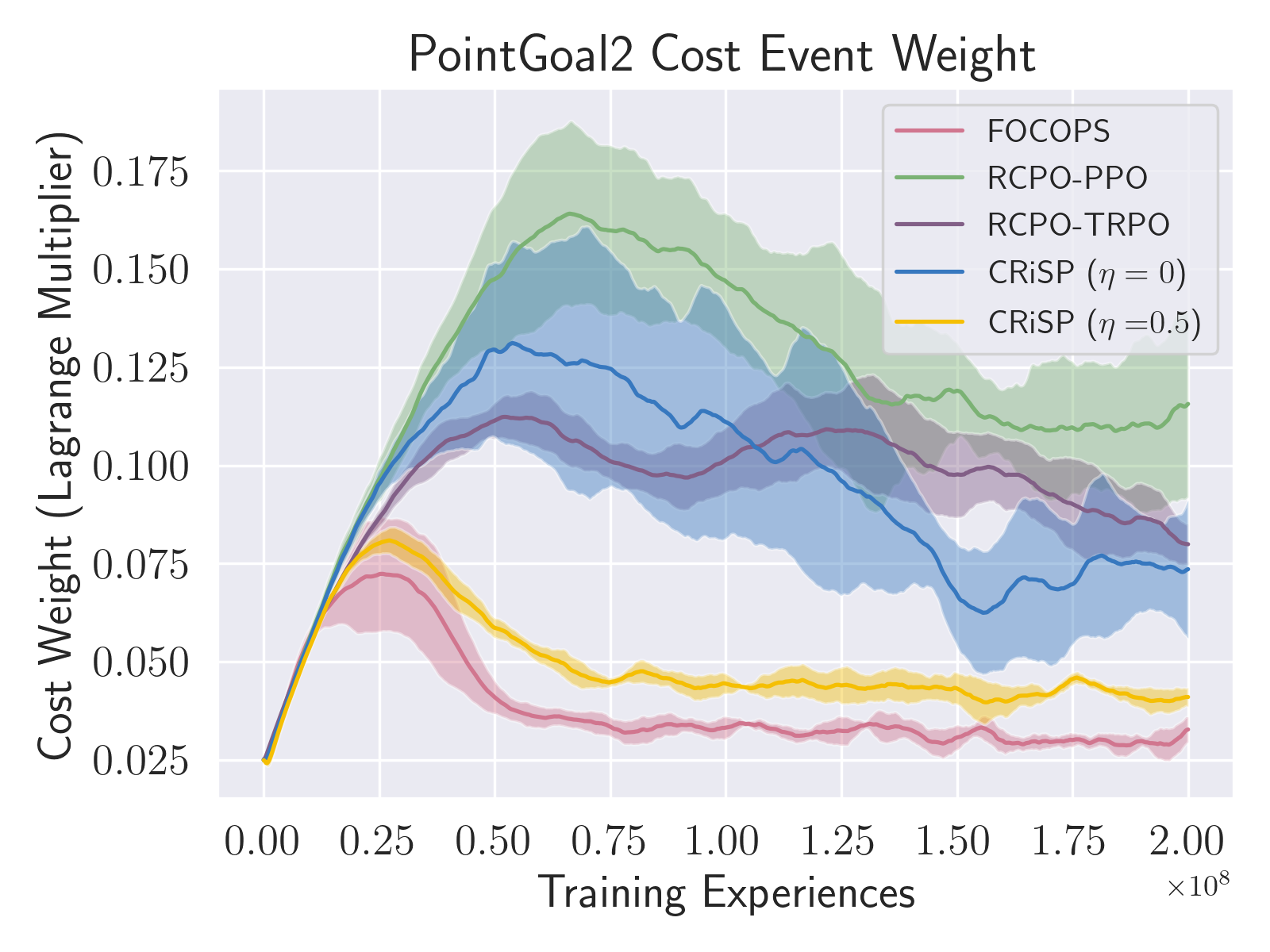}
    \includegraphics[width=0.234\textwidth]{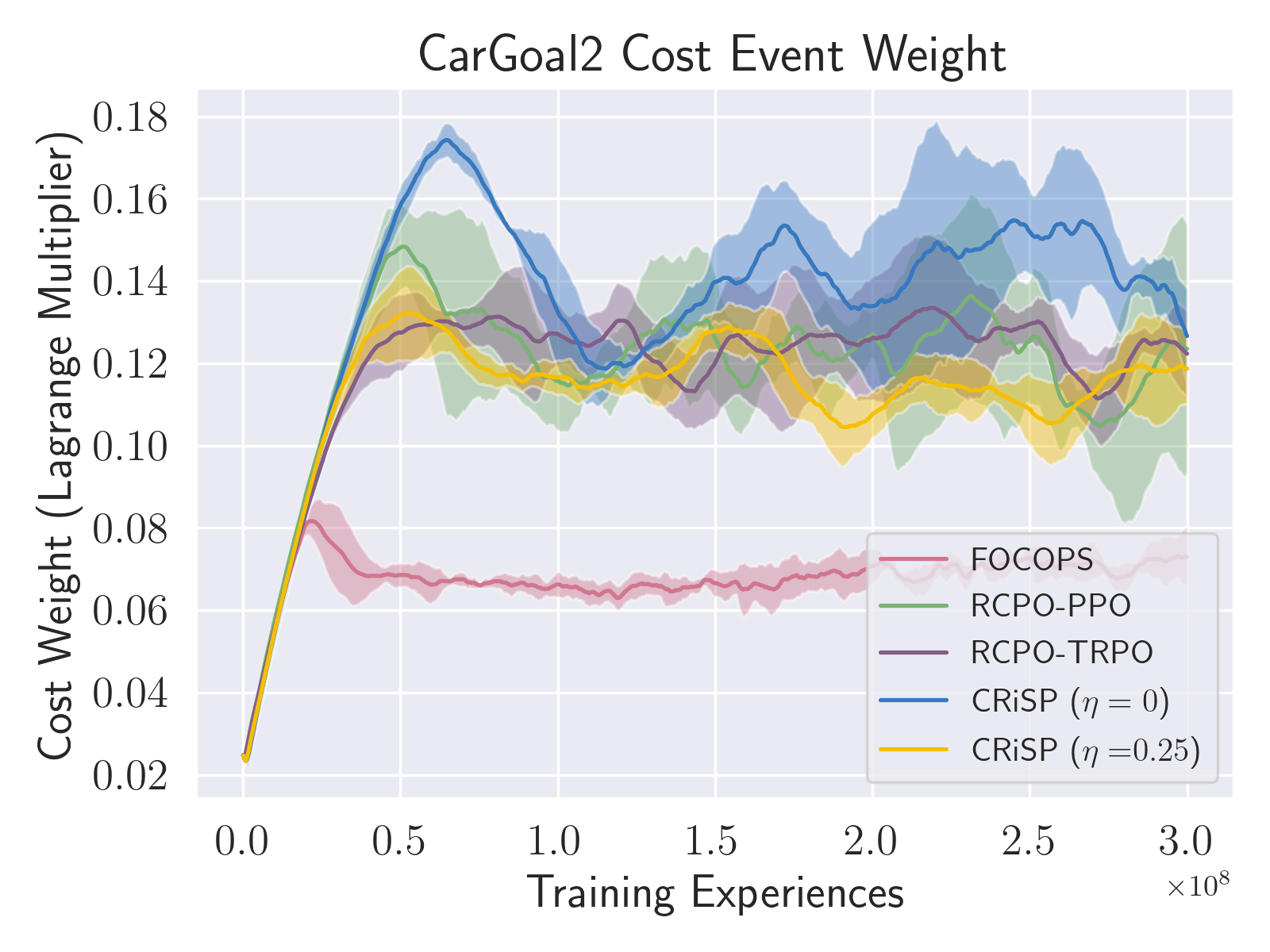}
    \includegraphics[width=0.234\textwidth]{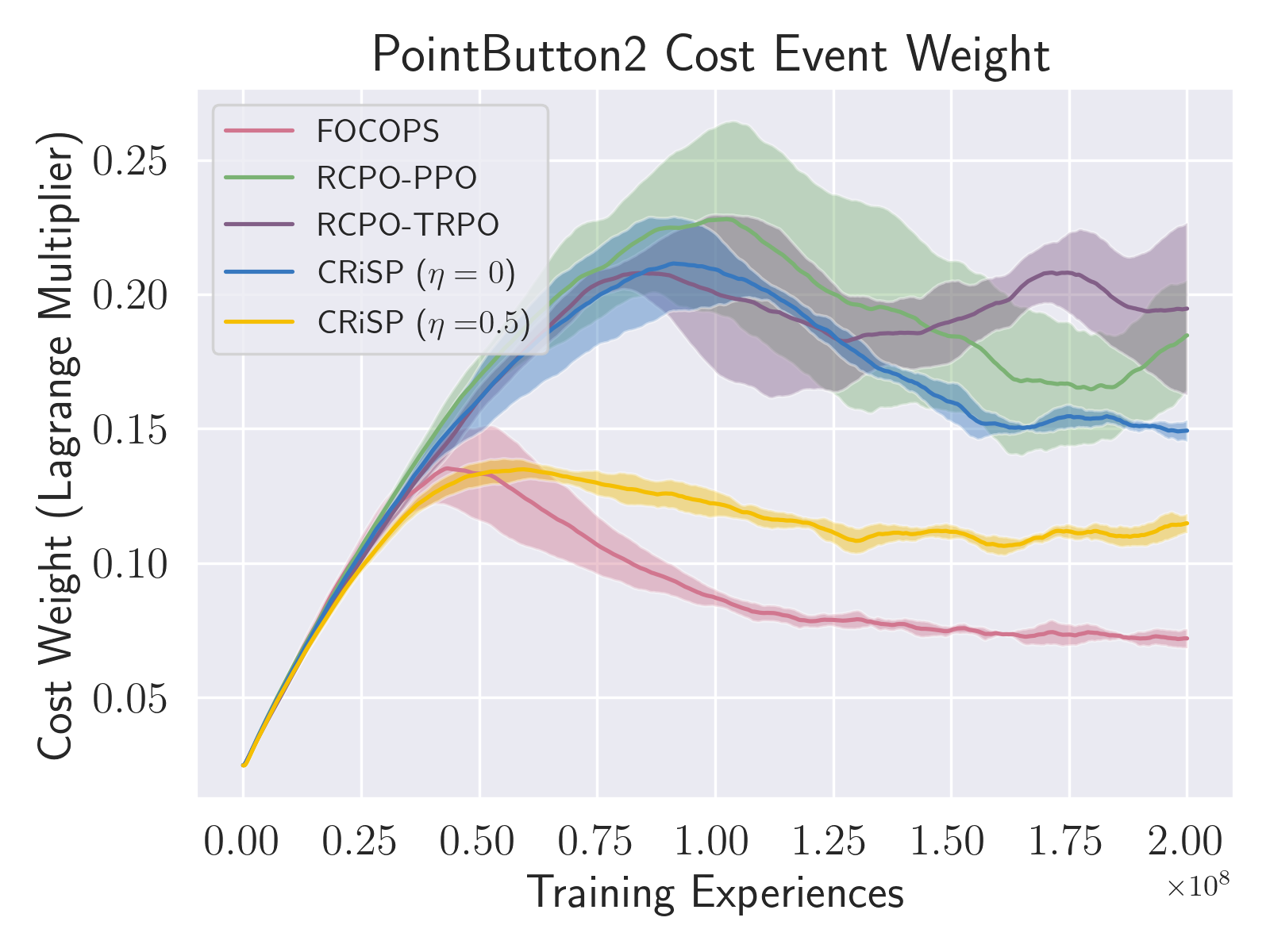}
    \includegraphics[width=0.234\textwidth]{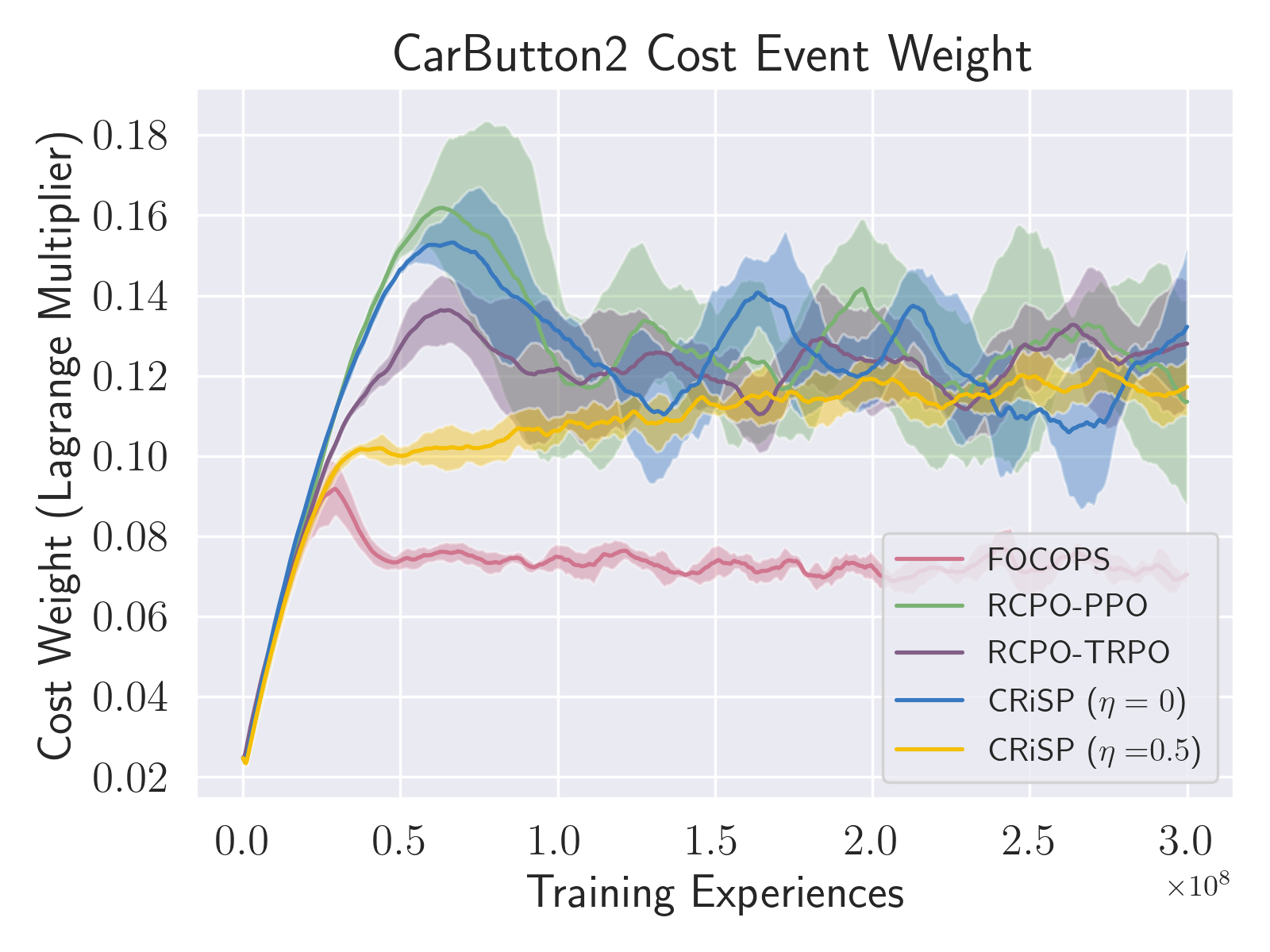}
    \includegraphics[width=0.234\textwidth]{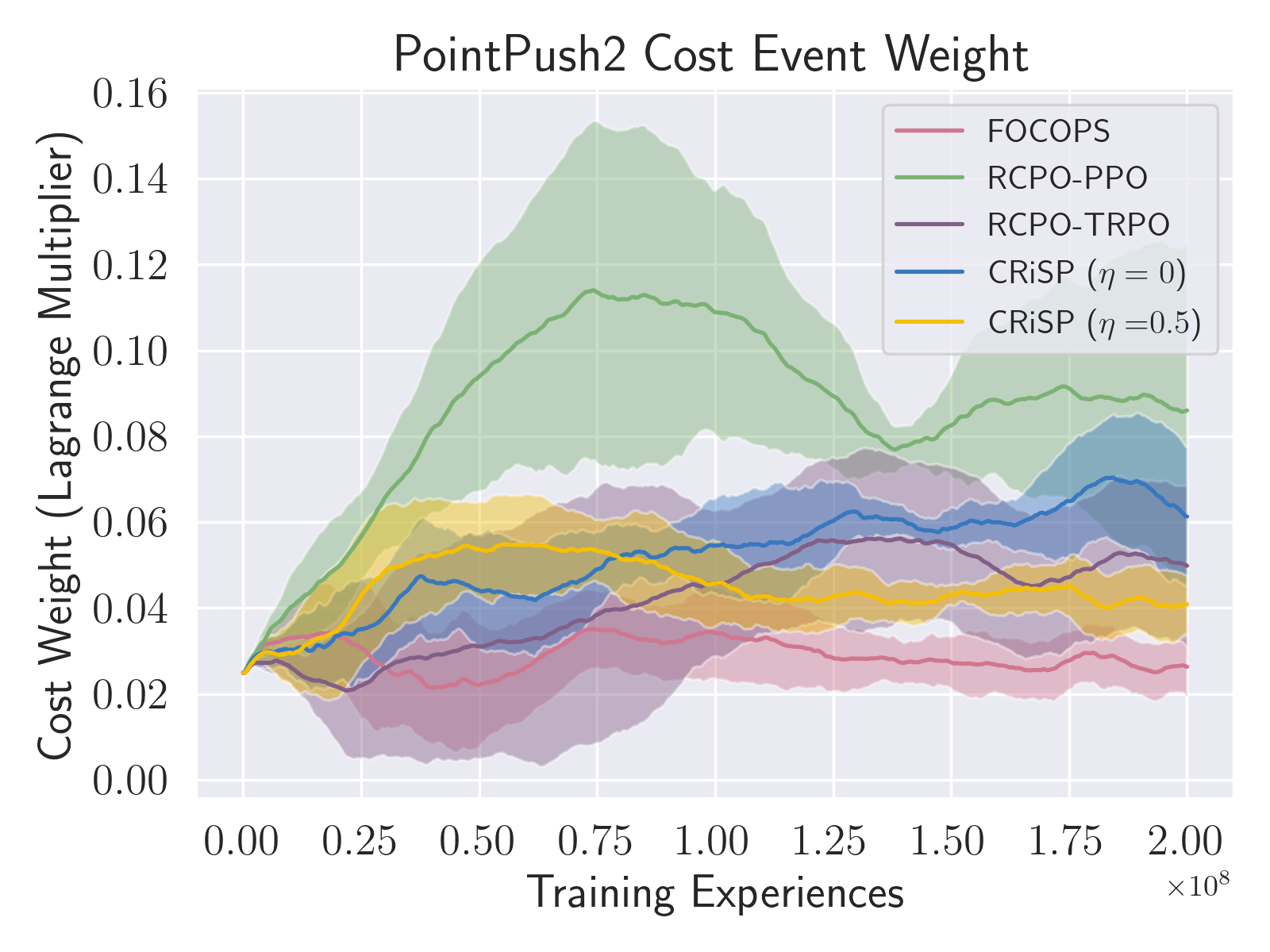}
    \includegraphics[width=0.234\textwidth]{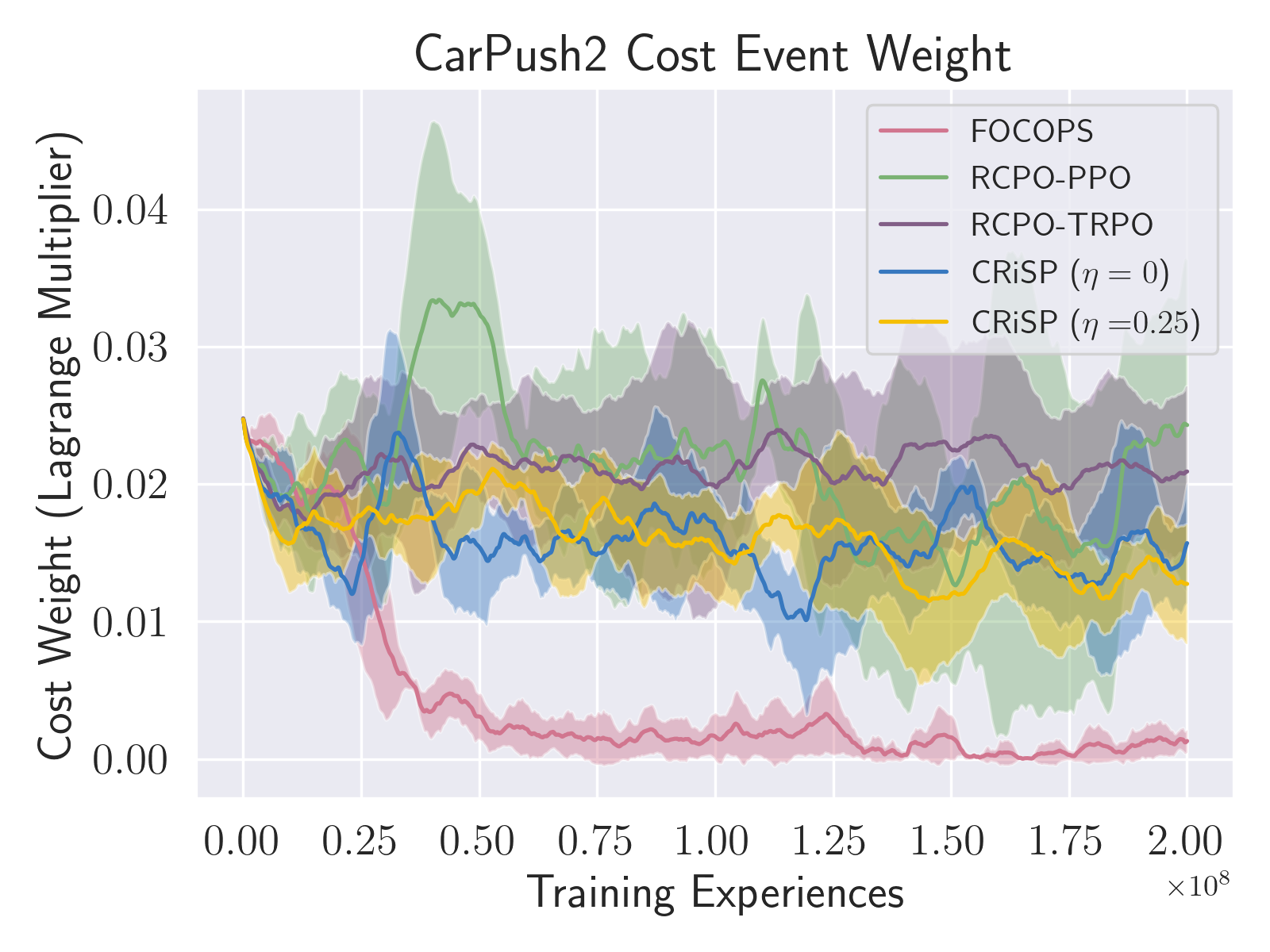}
     \caption{Learned cost weights for each method and environment.  The pessimistic risk-sensitive agents (CRiSP; $\eta > 0$) tend to converge to a lower learned penalty than every method except FOCOPS.  Note that the lower penalties used by FOCOPS do not correspond to higher positive reward accumulation.}
    \label{constr2}
\end{figure}

\begin{figure}
    \centering
    \includegraphics[width=0.234\textwidth]{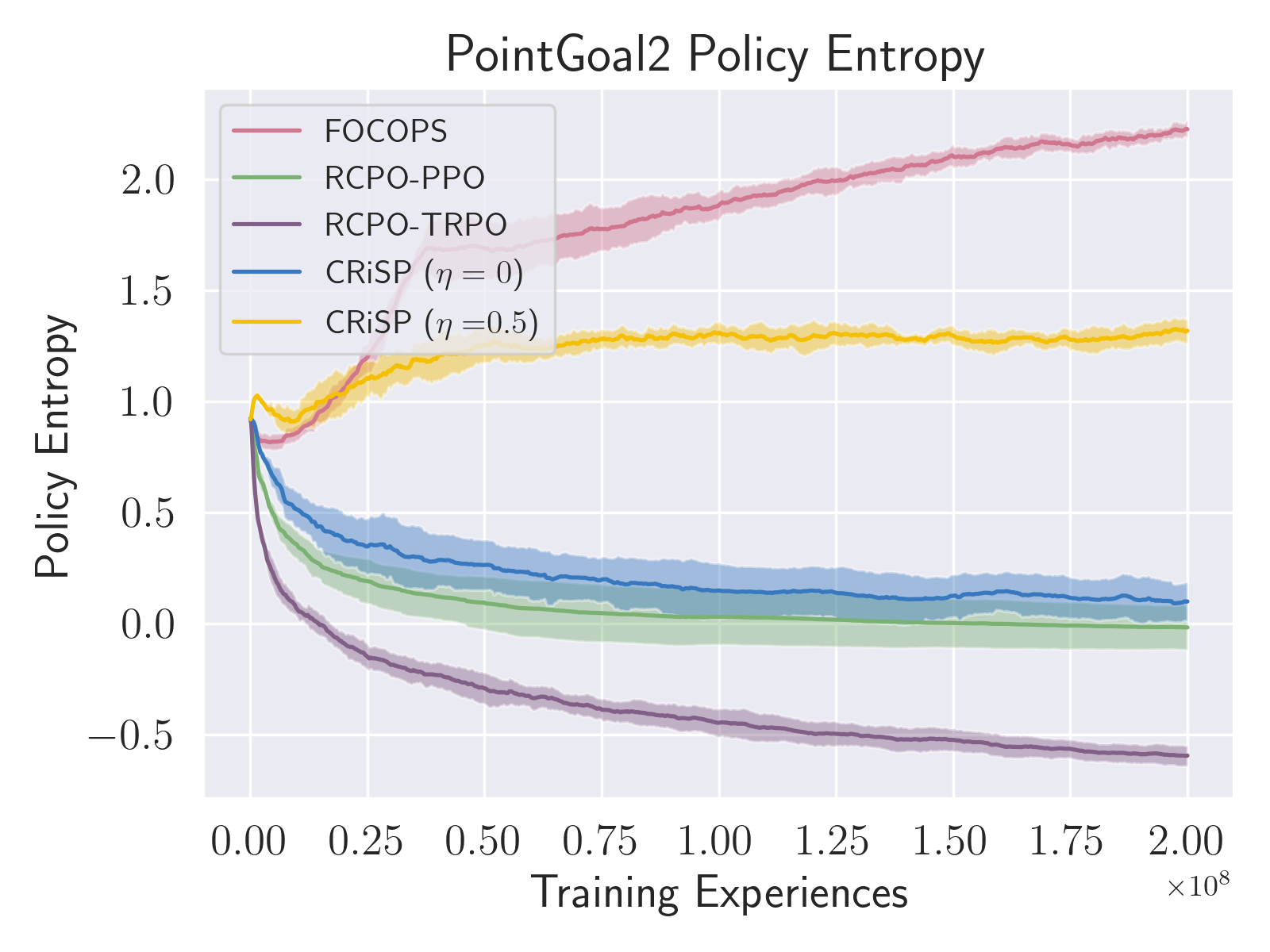} 
    \includegraphics[width=0.234\textwidth]{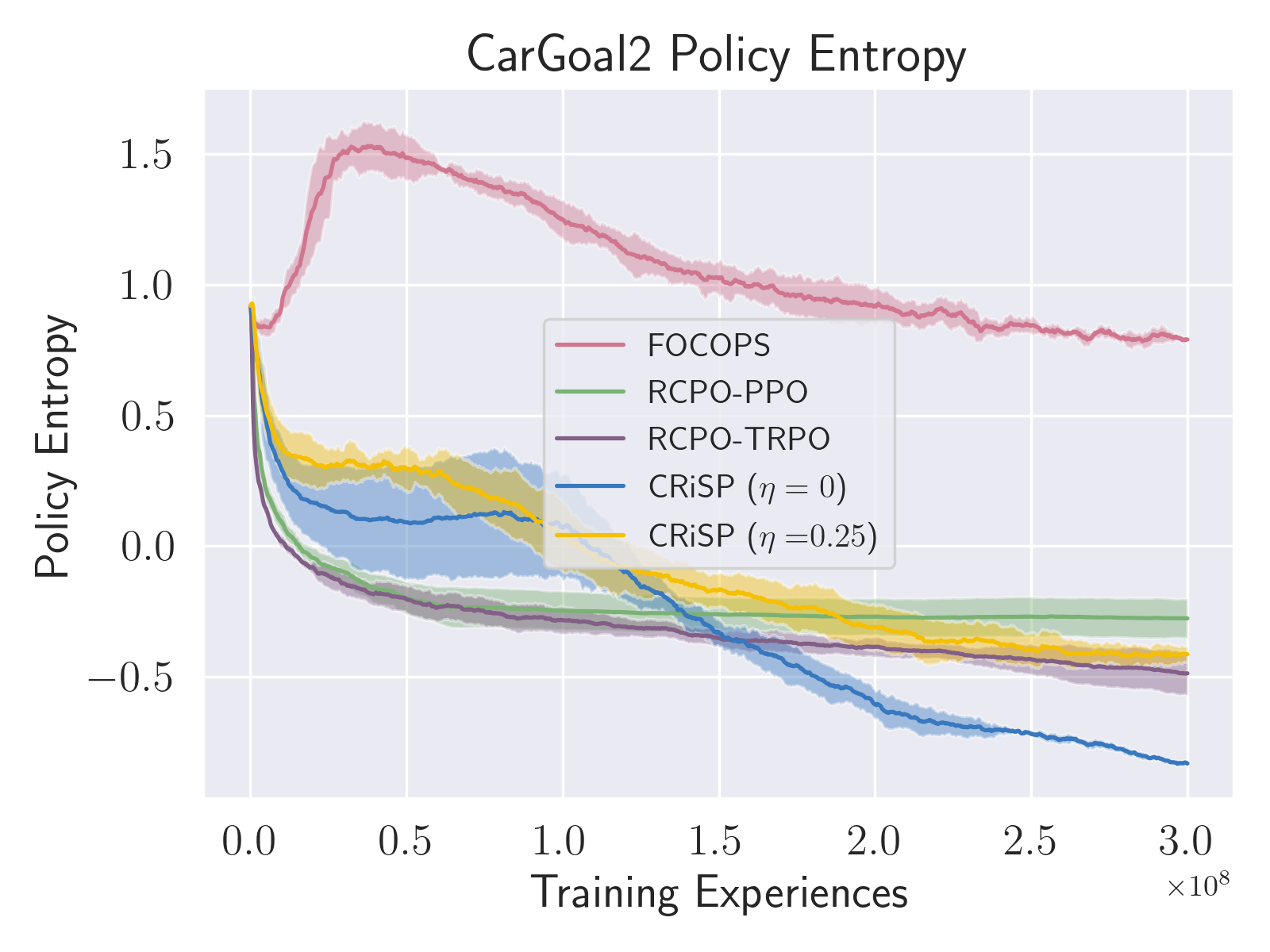}
    \includegraphics[width=0.234\textwidth]{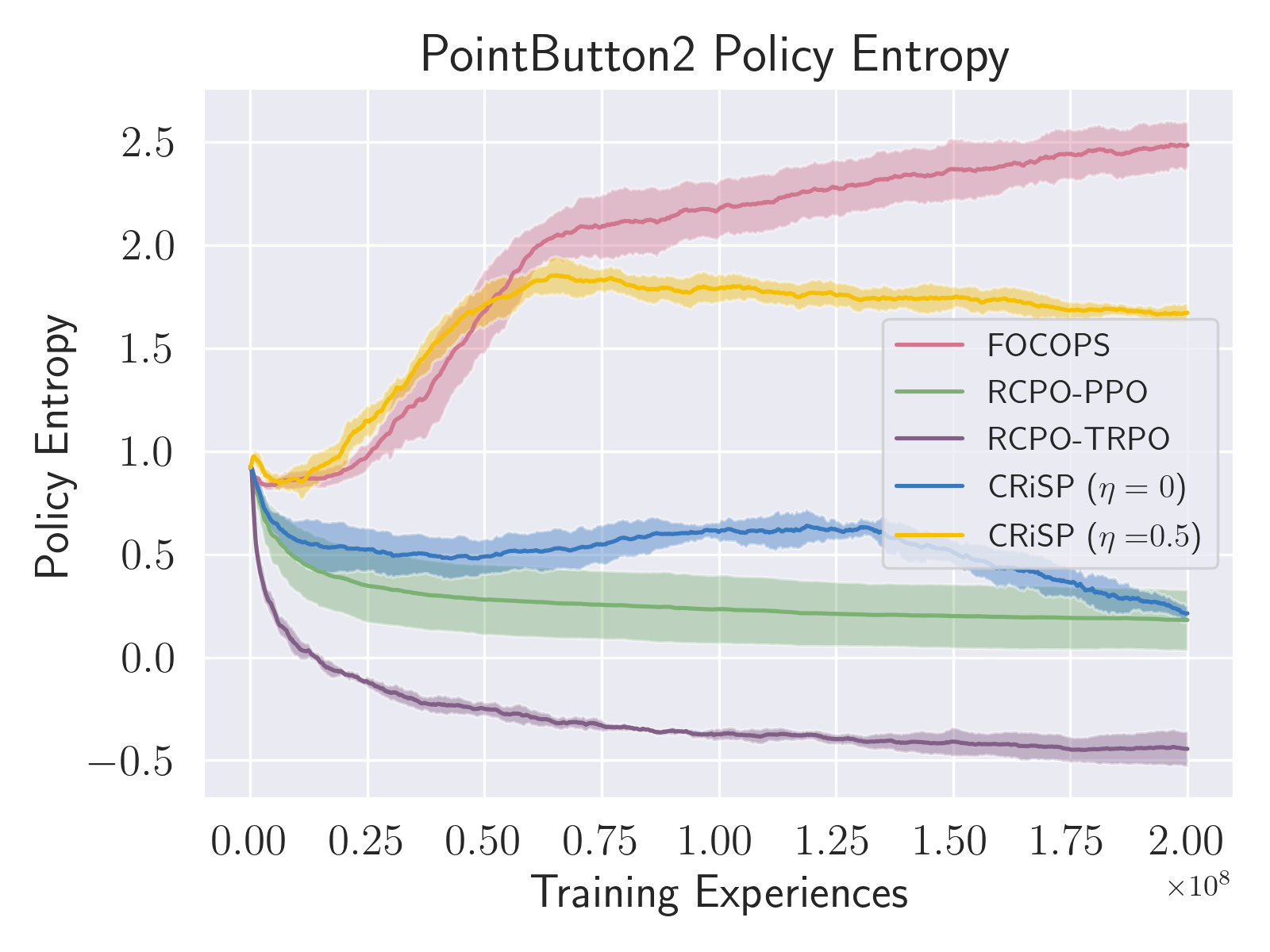}
    \includegraphics[width=0.234\textwidth]{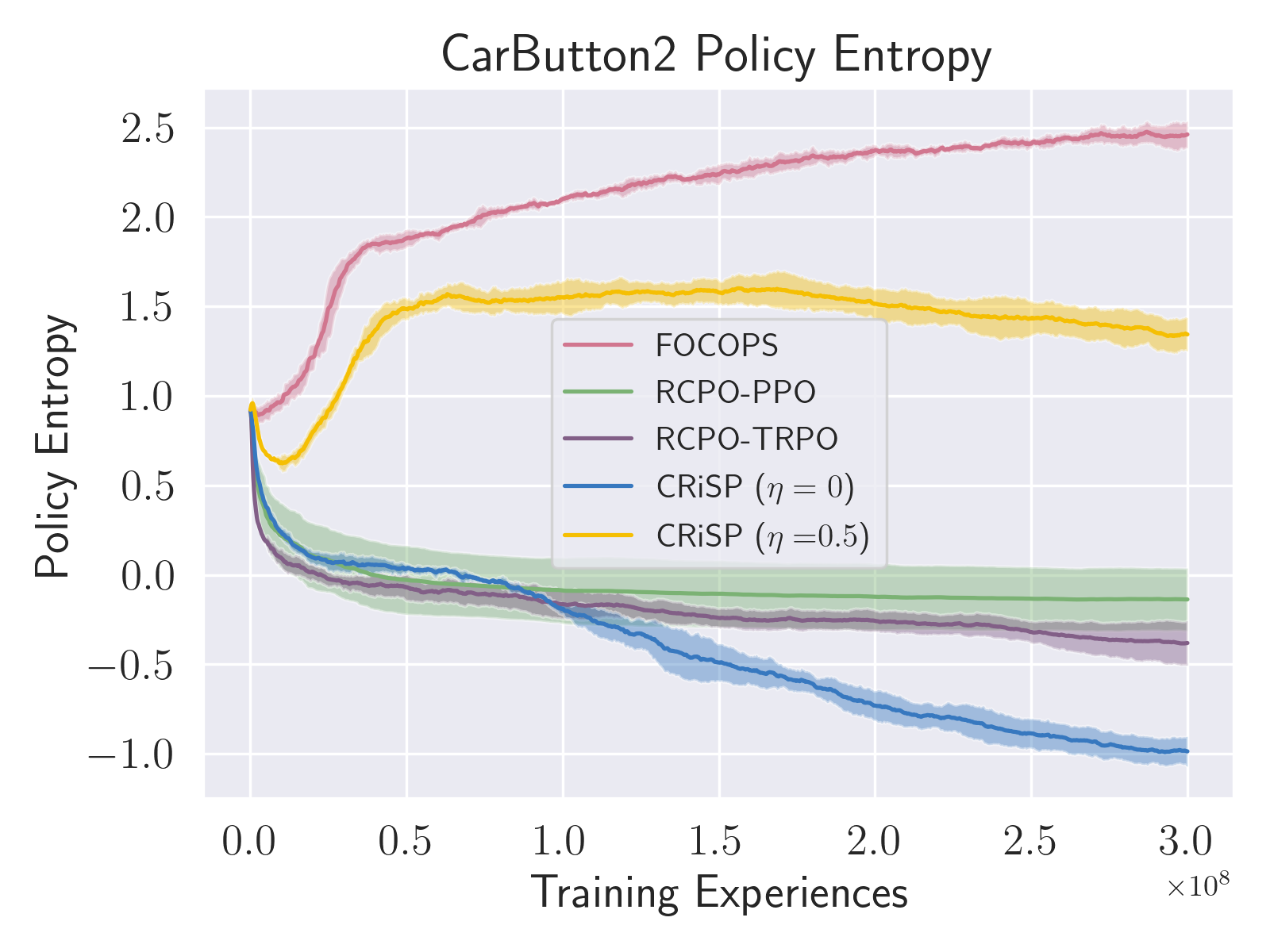}
    \includegraphics[width=0.234\textwidth]{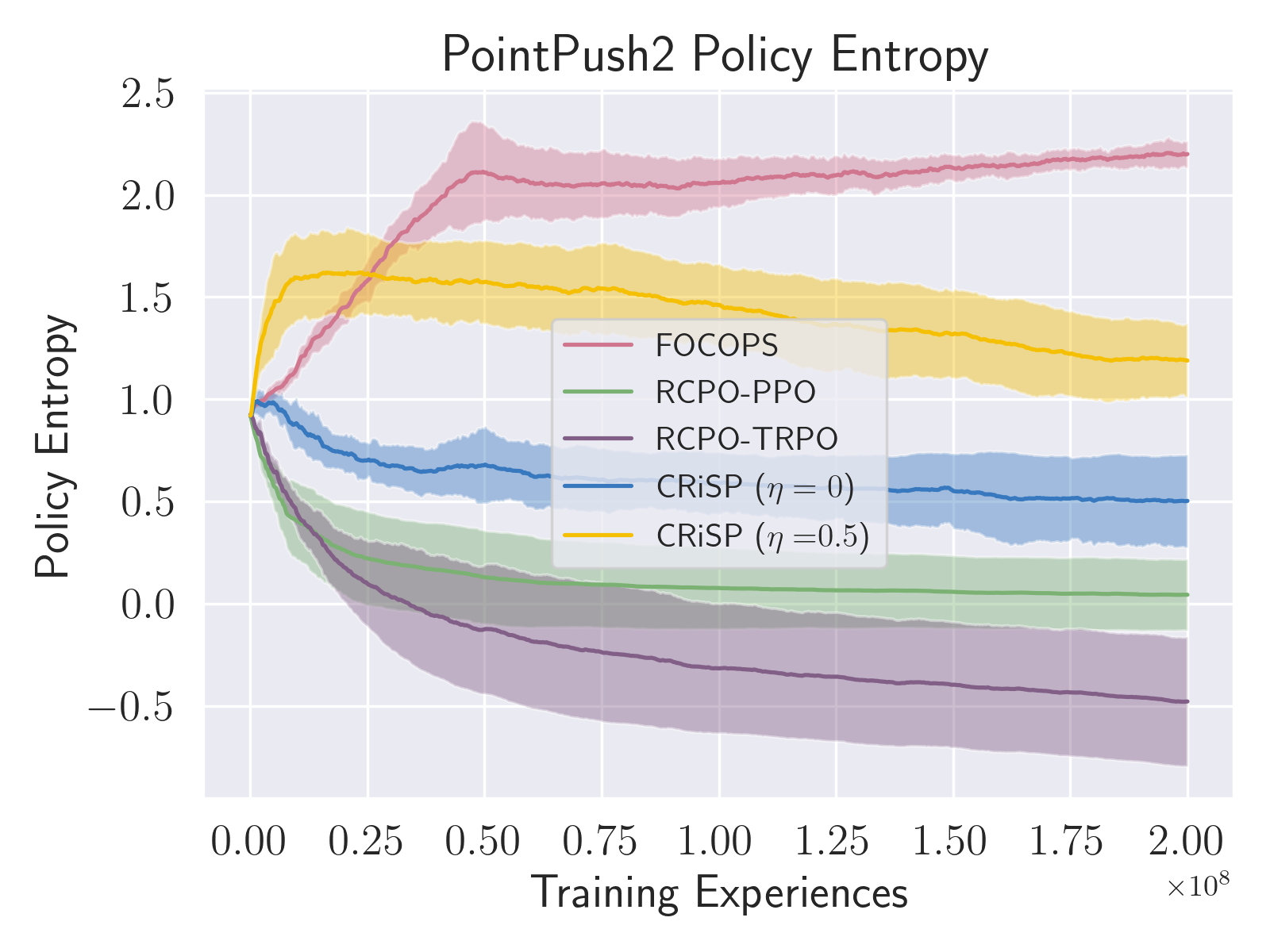}
    \includegraphics[width=0.234\textwidth]{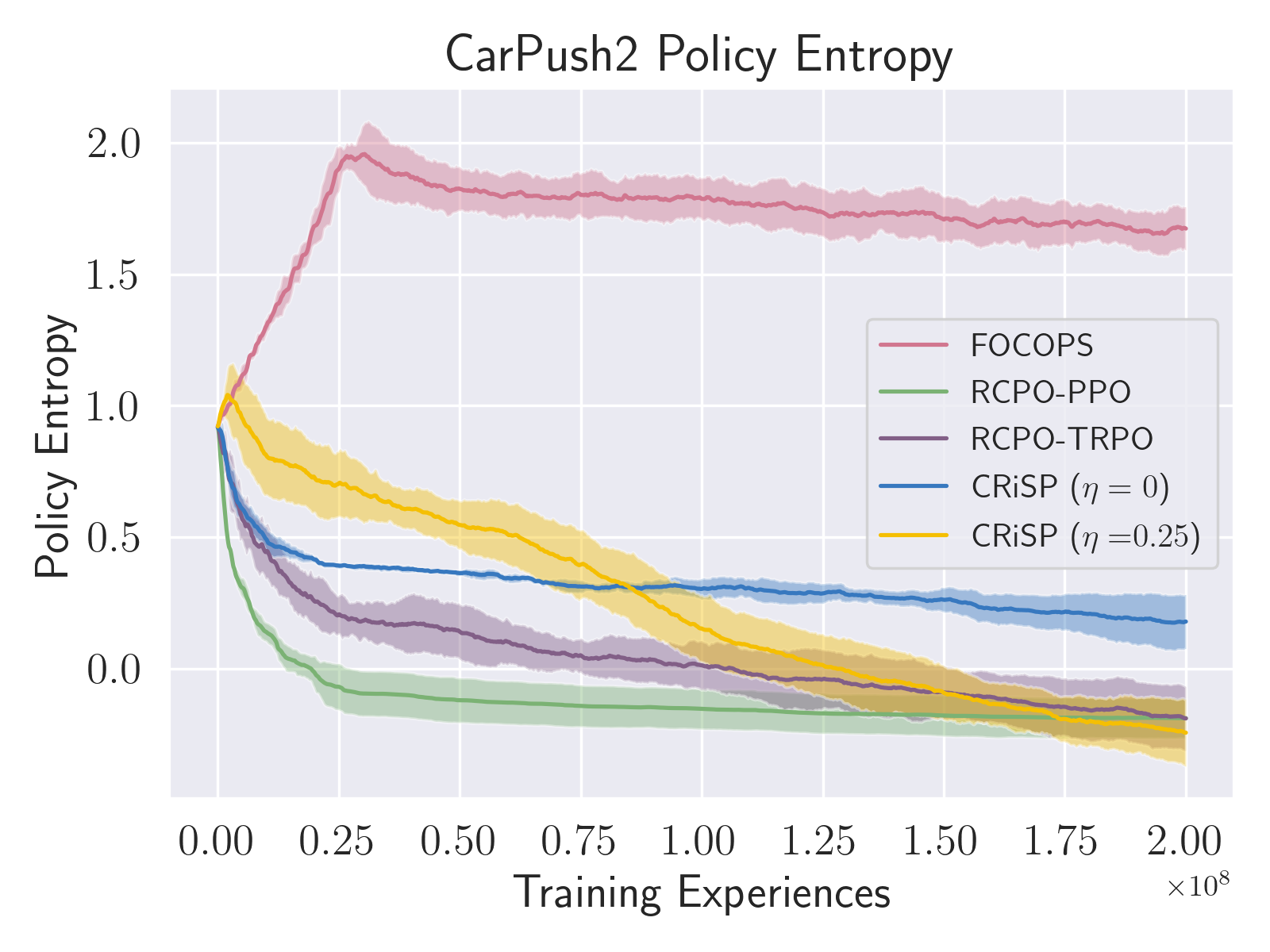}
     \caption{Policy entropies for constrained learning methods.  These entropies do not take into account control bounds.}
    \label{constr3}
\end{figure}

\begin{figure}
    \centering
    \includegraphics[width=0.234\textwidth]{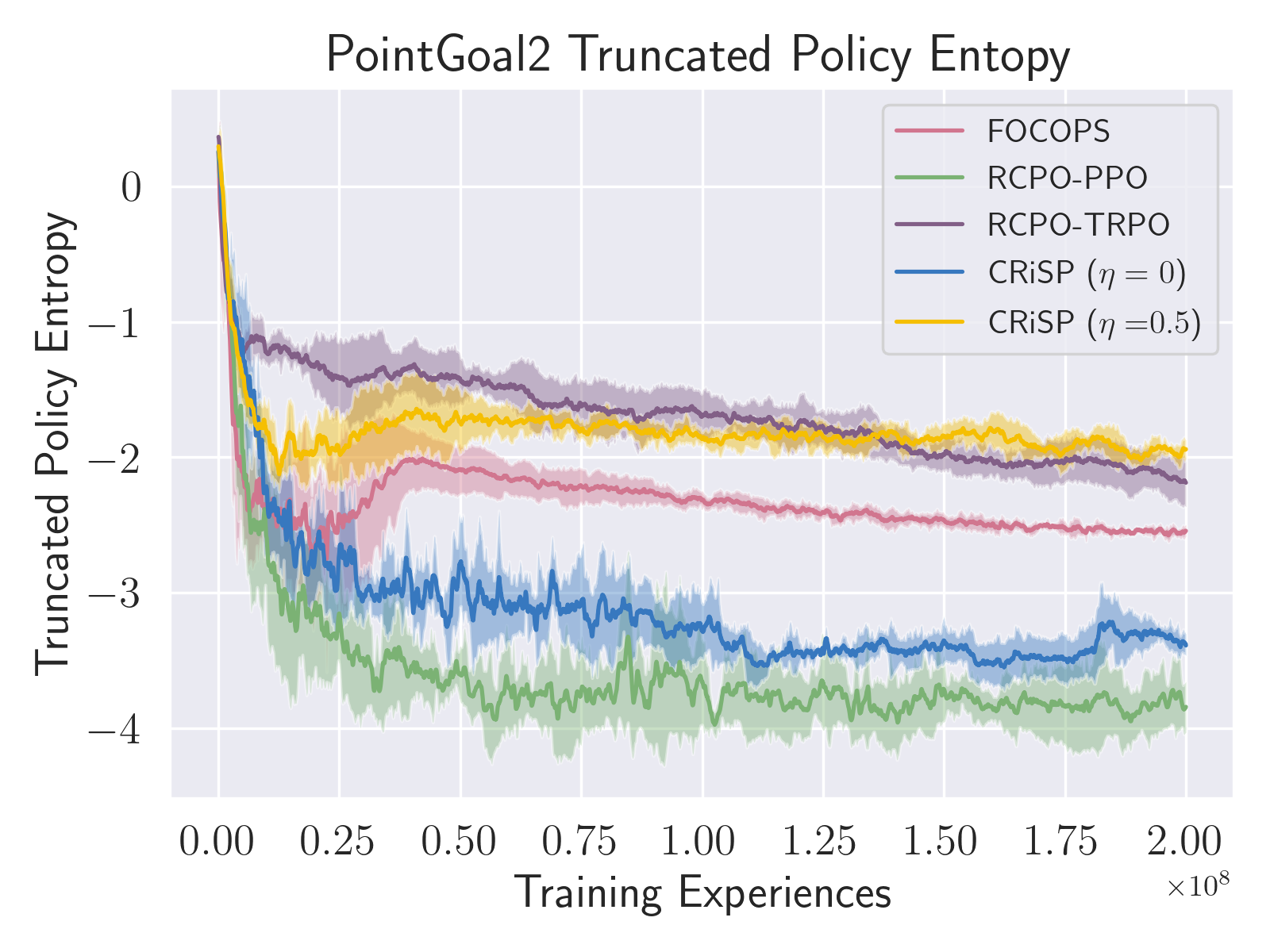}
    \includegraphics[width=0.234\textwidth]{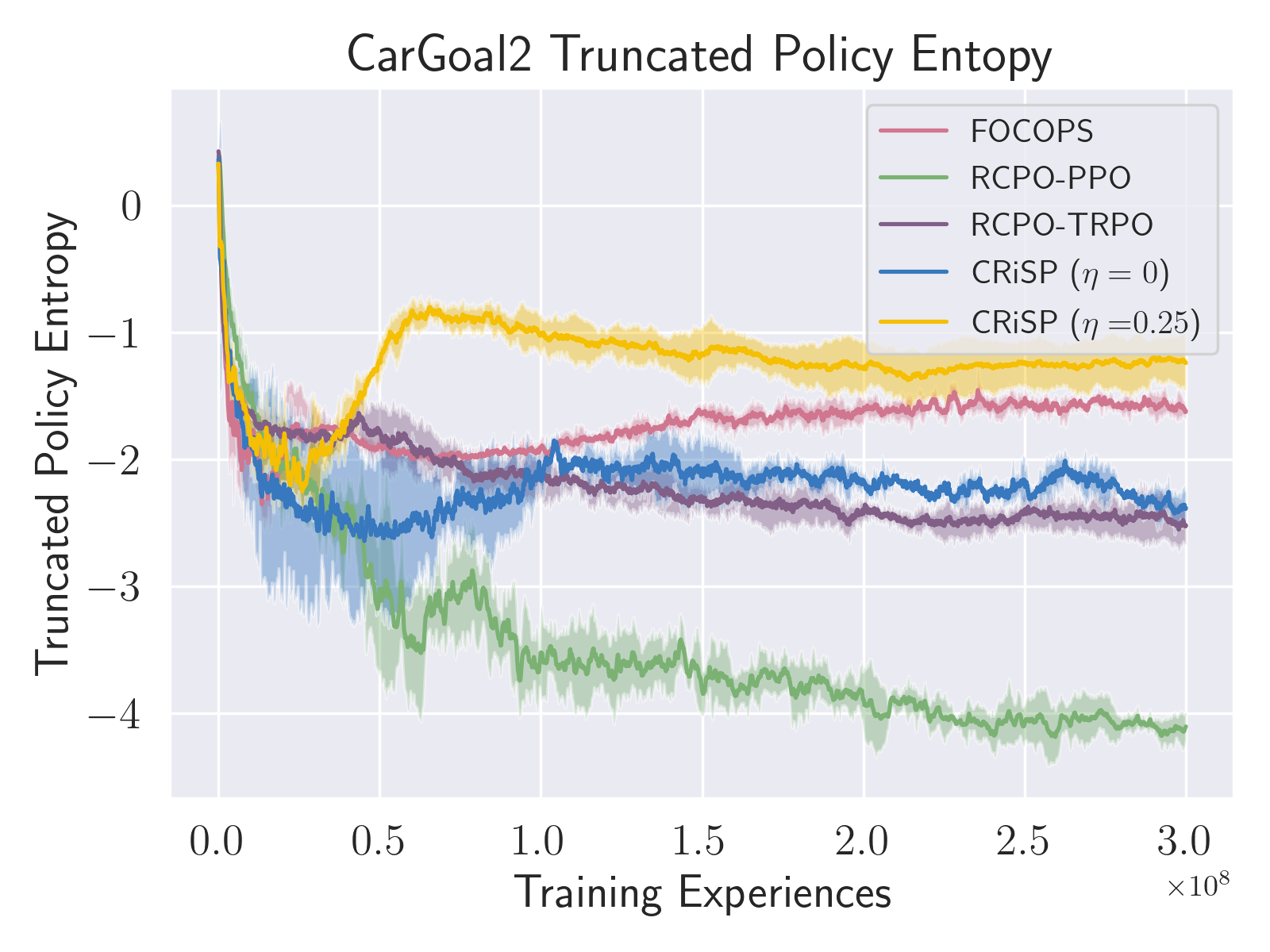}
    \includegraphics[width=0.234\textwidth]{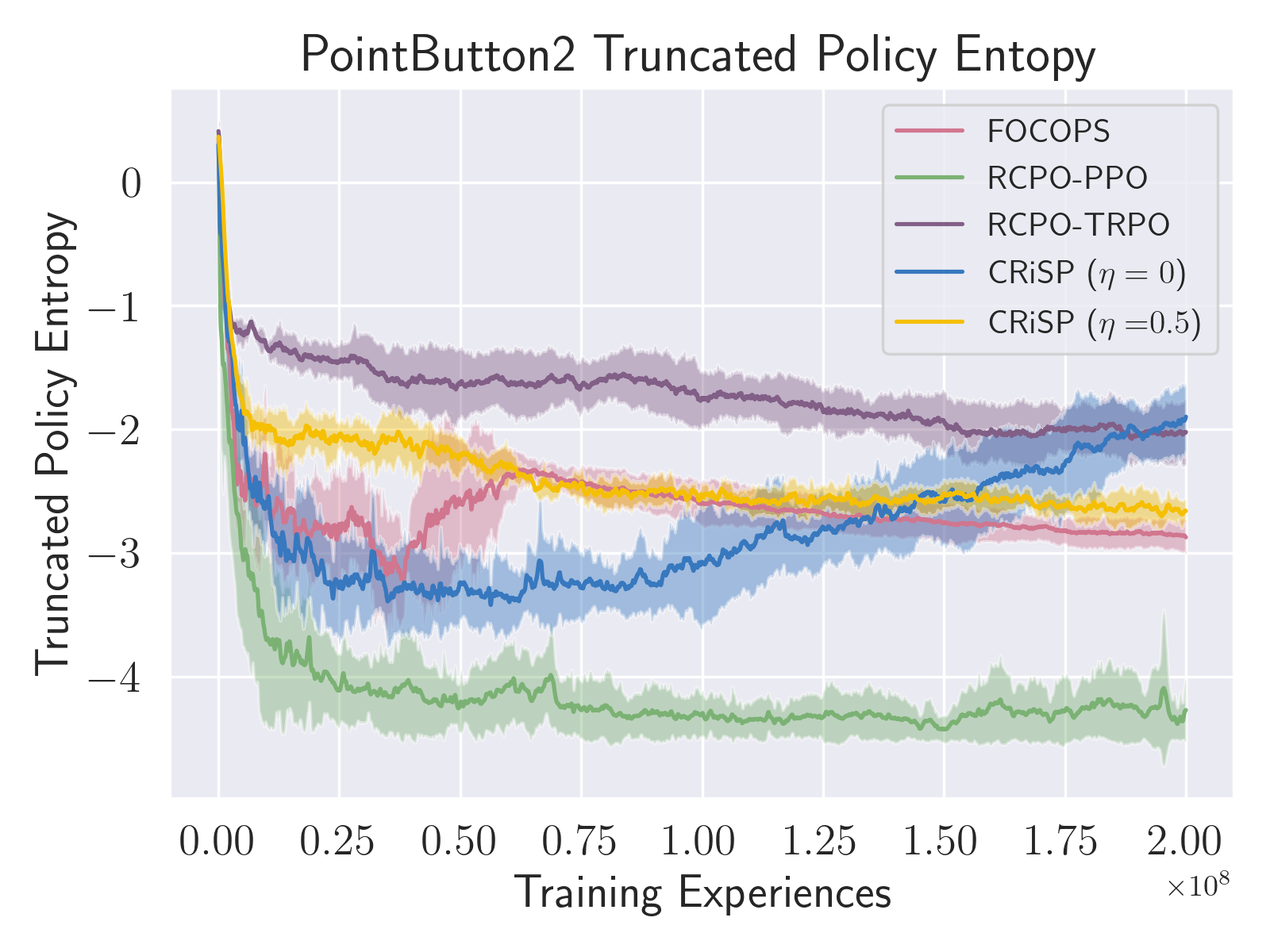}
    \includegraphics[width=0.234\textwidth]{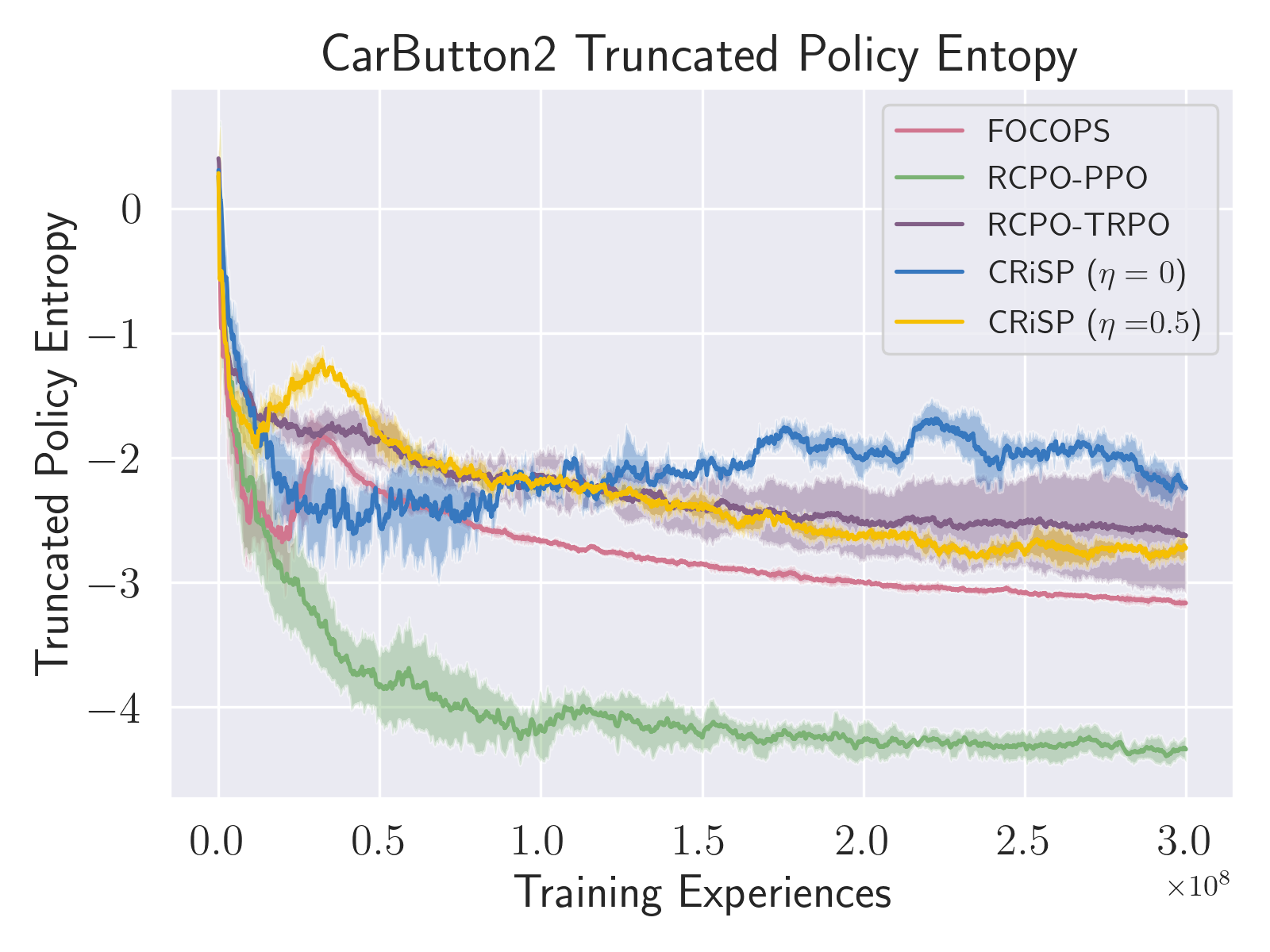}
    \includegraphics[width=0.234\textwidth]{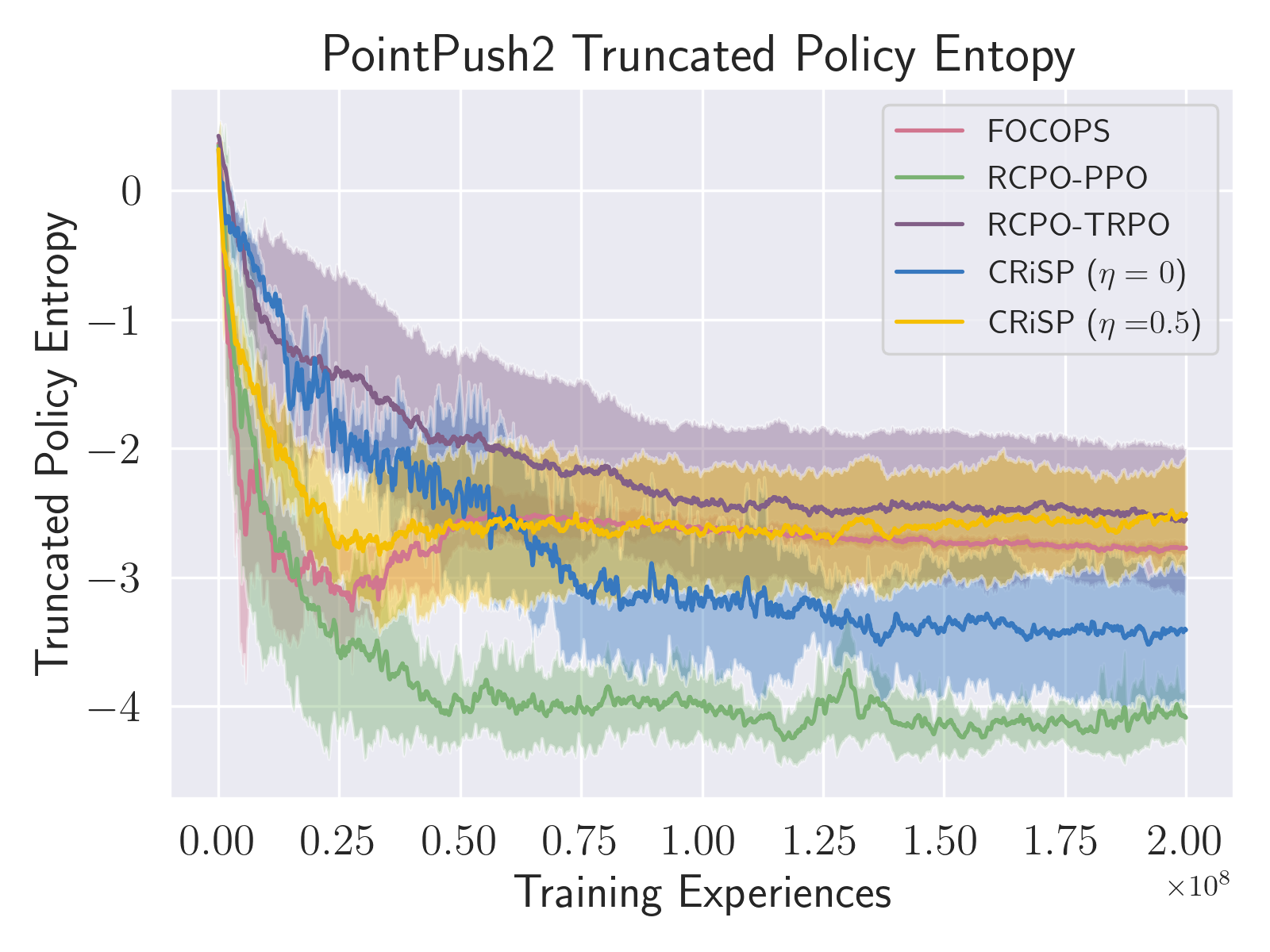}
    \includegraphics[width=0.234\textwidth]{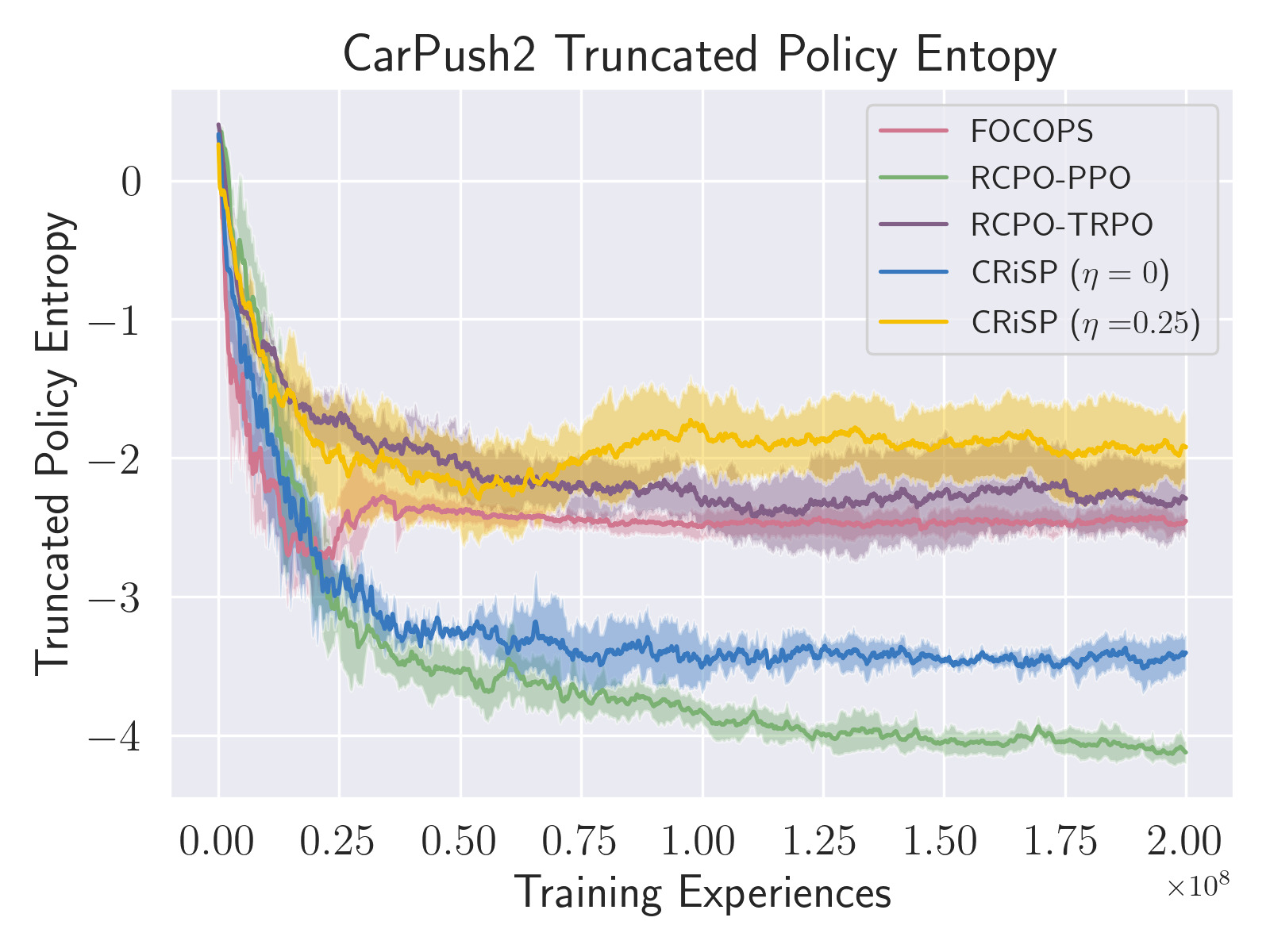}
     \caption{Truncated (i.e., taking into account control bounds) policy entropies for constrained learning methods.}
    \label{constr4}
\end{figure}

\subsection{A.9 Effect of Tuning $\eta$}\label{constr_eta_app}
We explored the effect of varying $\eta$ in the objective for one environment.  The weight functions and coefficients corresponding to different $\eta$ are shown in Figure 7.  In Figures 8-9, we show how performance changes with $\eta$ for unconstrained and constrained learning in that environment. Recall that increasing $\eta$ increases the emphasis placed on worst-case outcomes (``pessimistic'' weighting), while decreasing $\eta$ increases the emphasis on best-case outcomes (``optimistic'' weighting). $\eta = 0$ corresponds to risk-neutral behavior, or the standard RL objective.

The following trends are seen to be fairly general across environments:

\begin{itemize}
\item{Policy entropy remains higher as $\eta$ is increased}

\item{Truncated entropy, or entropy taking into account the control bounds, may be higher or lower with increased $\eta$, depending on the efficacy of control near the bounds of the environment.}

\item{Accumulated rewards tend to be similar over a broad range of positive $\eta$ centered on $\eta=0.5$.  The range is bounded by uniform weighting below and the point where the agent tends to ignore good outcomes too much above.}

\item{Costs at the minimum found by the agent tend to decrease with increasing $\eta$, up to the point where the agent ignores good outcomes too much.}

\item{Prescribed cost targets tend to be reached more quickly with increased $\eta$ in the constrained setting, again to the point where the good outcomes are ignored too much.}

\end{itemize}

\begin{figure}
    \centering
    \includegraphics[width=0.234\textwidth]{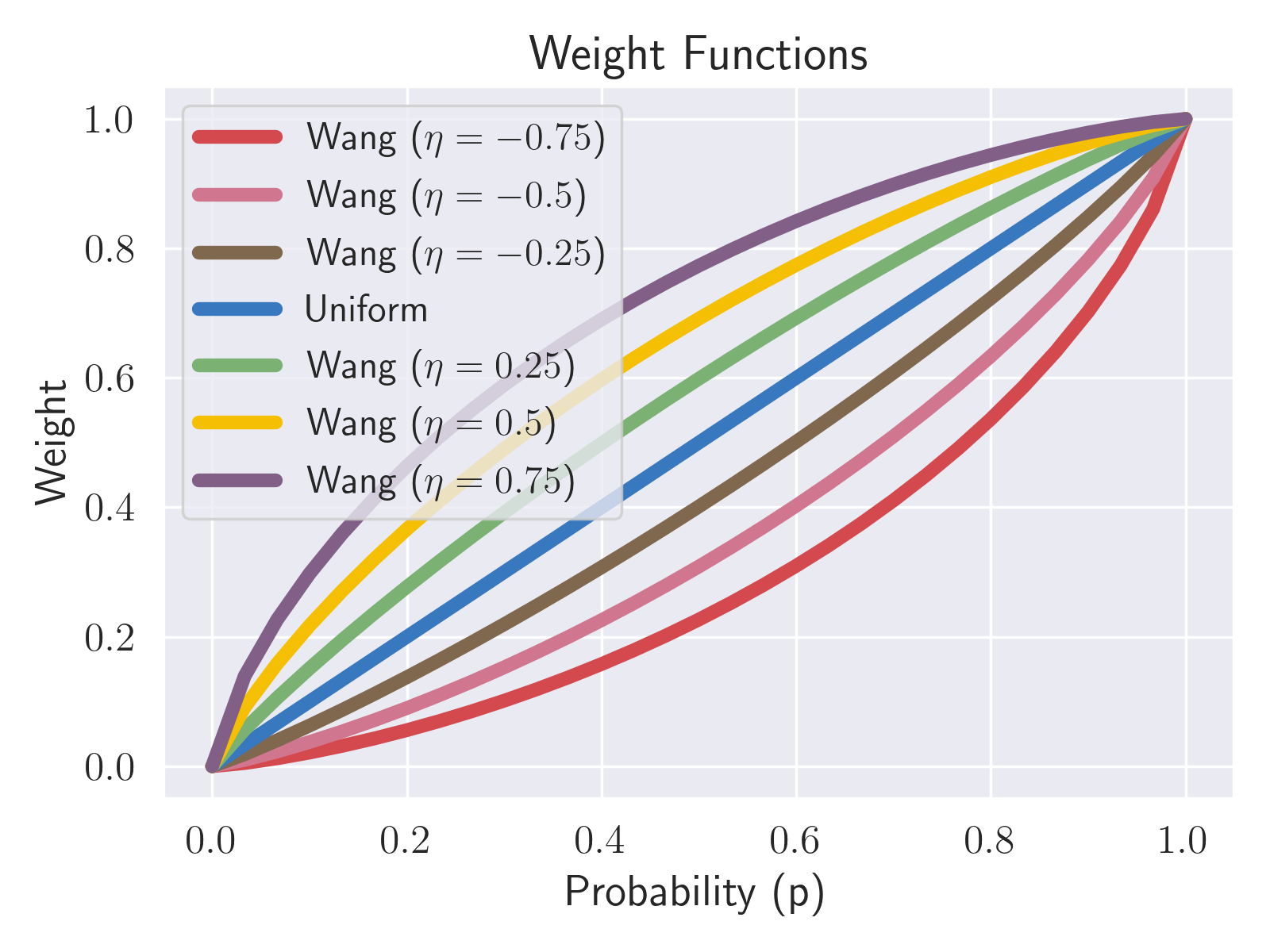}
    \includegraphics[width=0.234\textwidth]{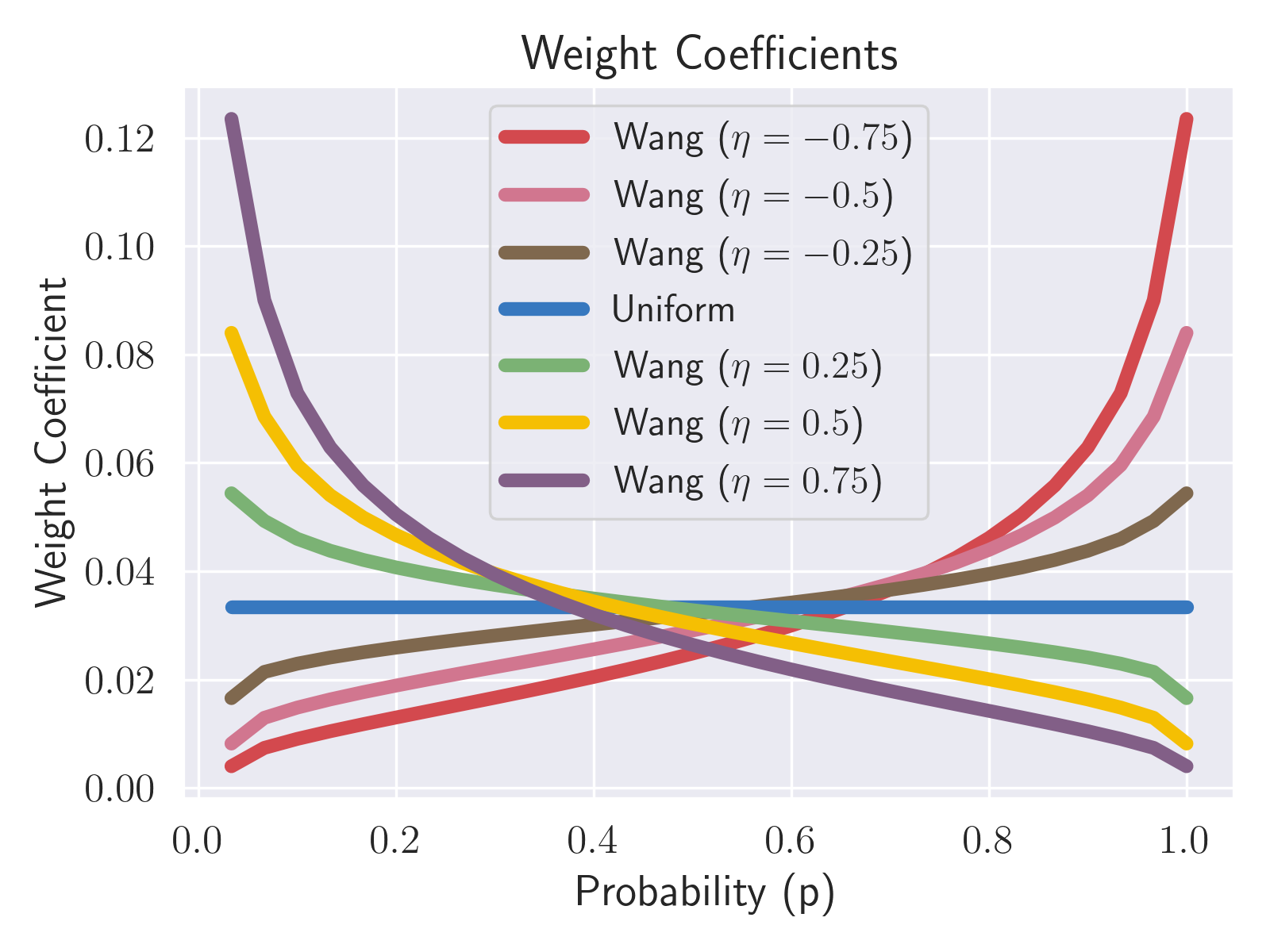}
    \caption{Effect of tuning $\eta$ on the weight functions (left) and coefficients in the risk-sensitive policy gradient estimate (Equation 9; right).}
    \label{all_weights}
\end{figure}

\begin{figure}
    \centering
    \includegraphics[width=0.234\textwidth]{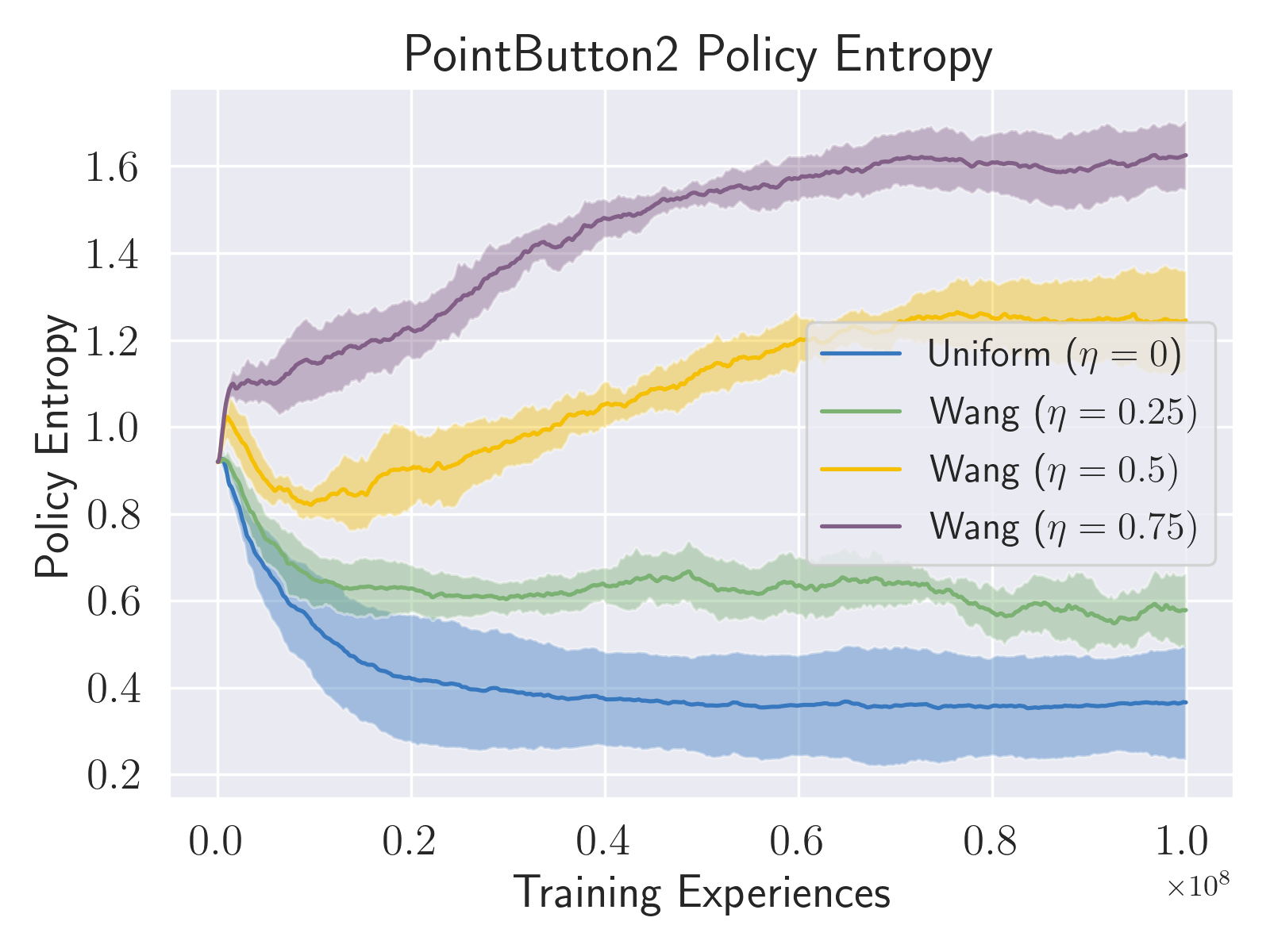}
    \includegraphics[width=0.234\textwidth]{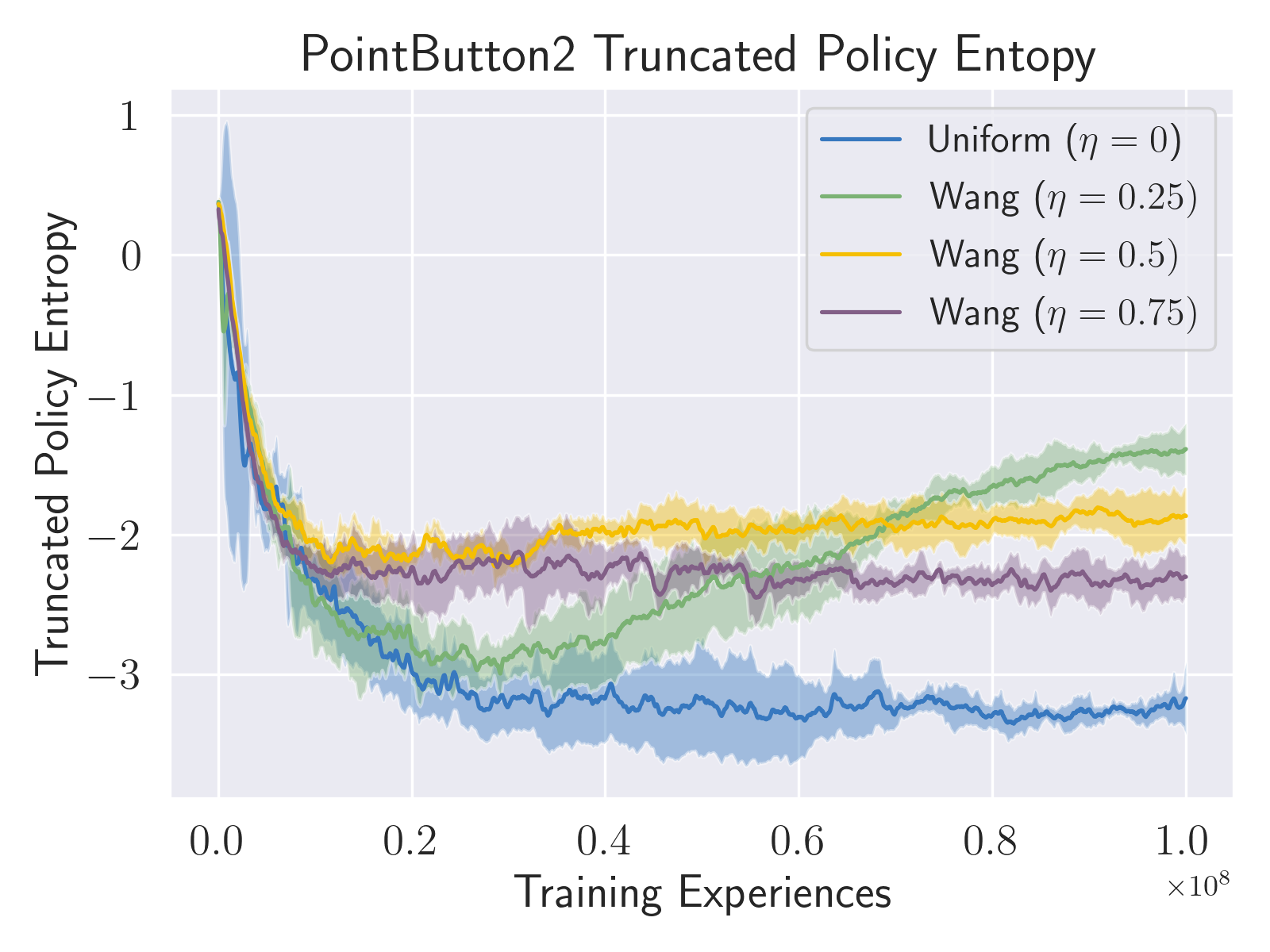}
    \includegraphics[width=0.234\textwidth]{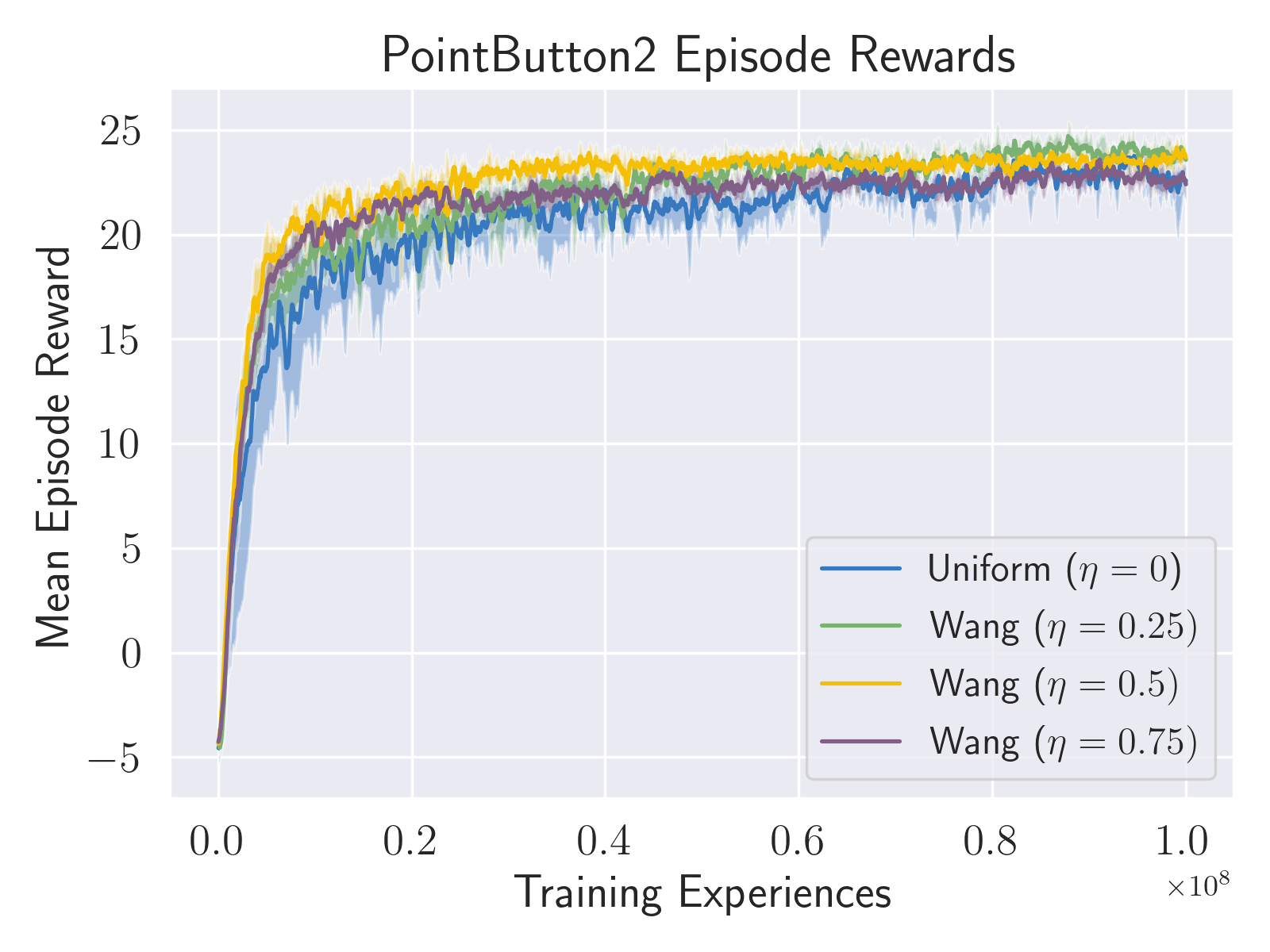}
    \includegraphics[width=0.234\textwidth]{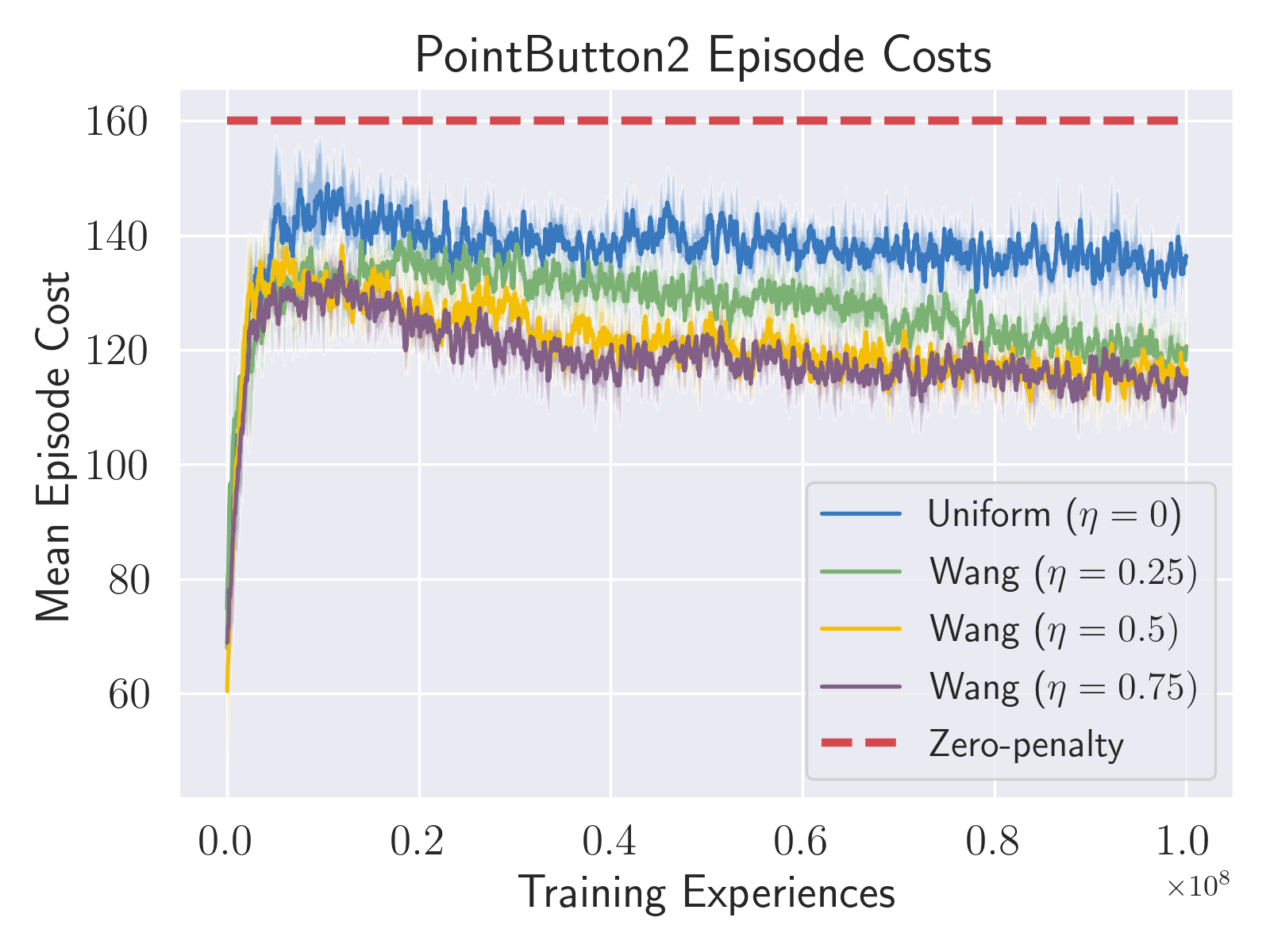}
     \caption{Effect of tuning $\eta$ on unconstrained learning in one environment (PointButton2).  Top row: policy distributions become wider with increasing $\eta$, while entropy considering control bounds depends on the locations of the control actions of the optimized policies.  Bottom row: episode rewards are similar over a fairly broad range of $\eta > 0$, with cost levels decreasing with increasing $\eta$ in that range.}
    \label{constr_eta1}
\end{figure}

\begin{figure}
    \centering
    \includegraphics[width=0.234\textwidth]{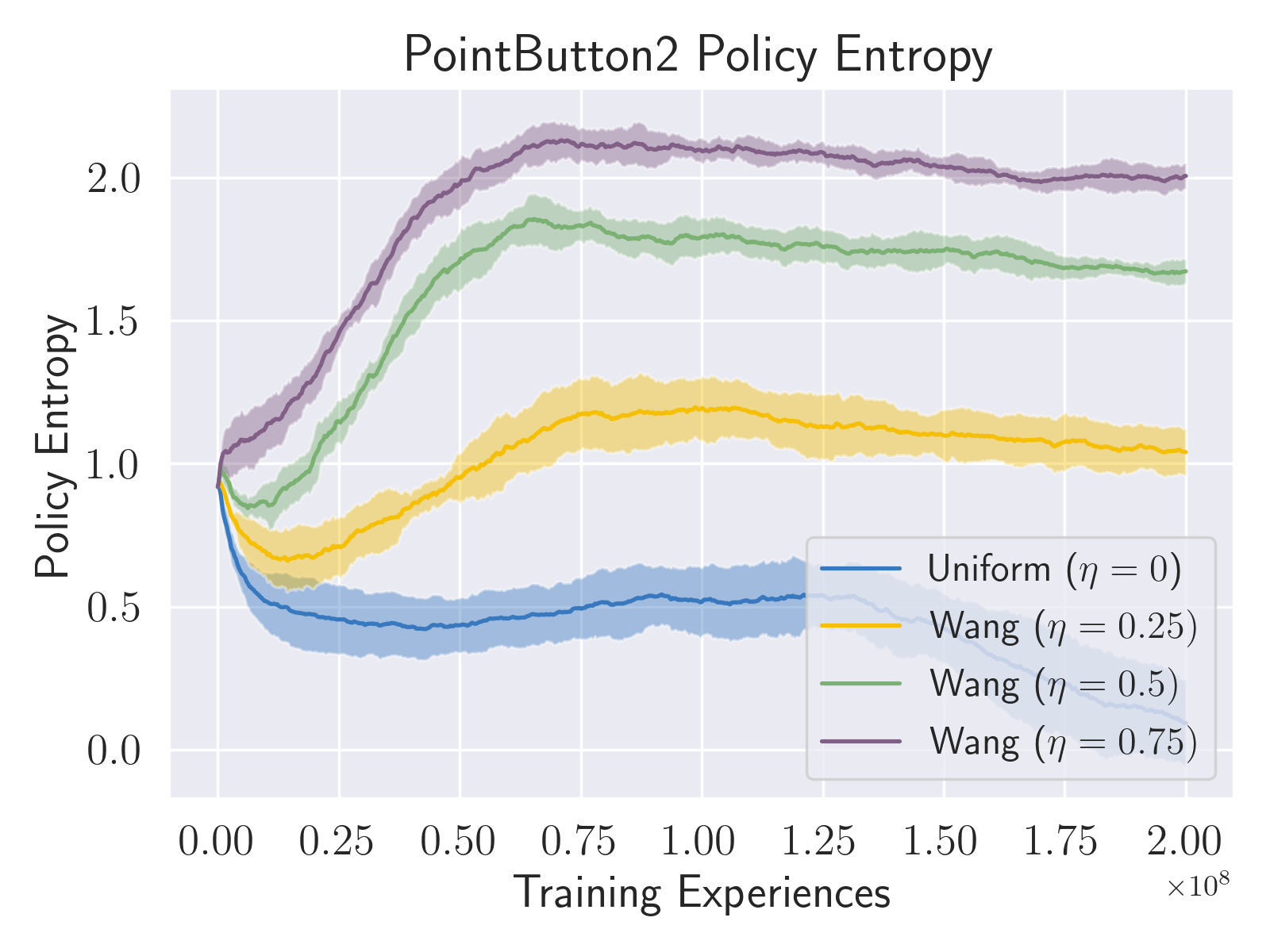}
    \includegraphics[width=0.234\textwidth]{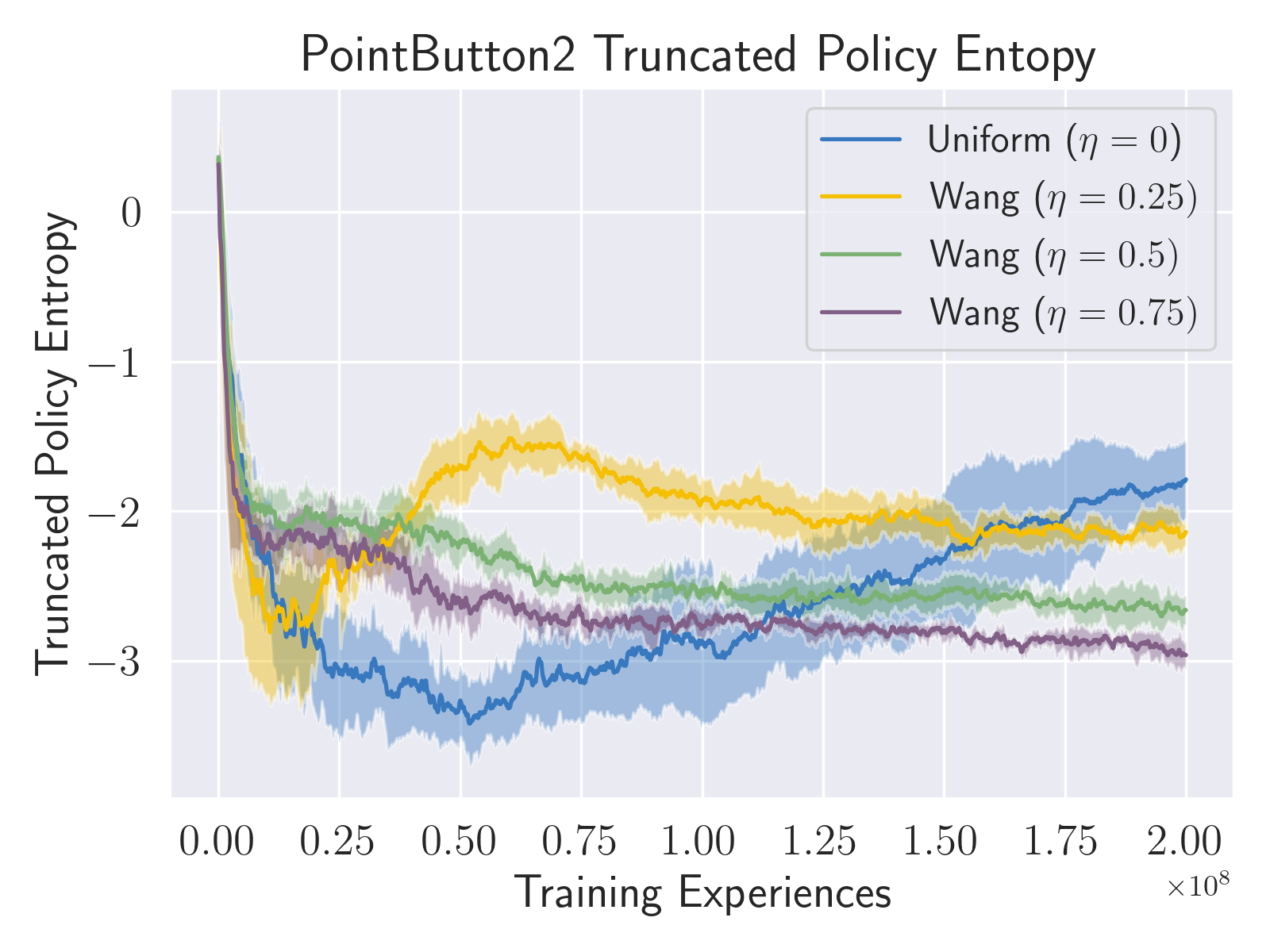}
    \includegraphics[width=0.234\textwidth]{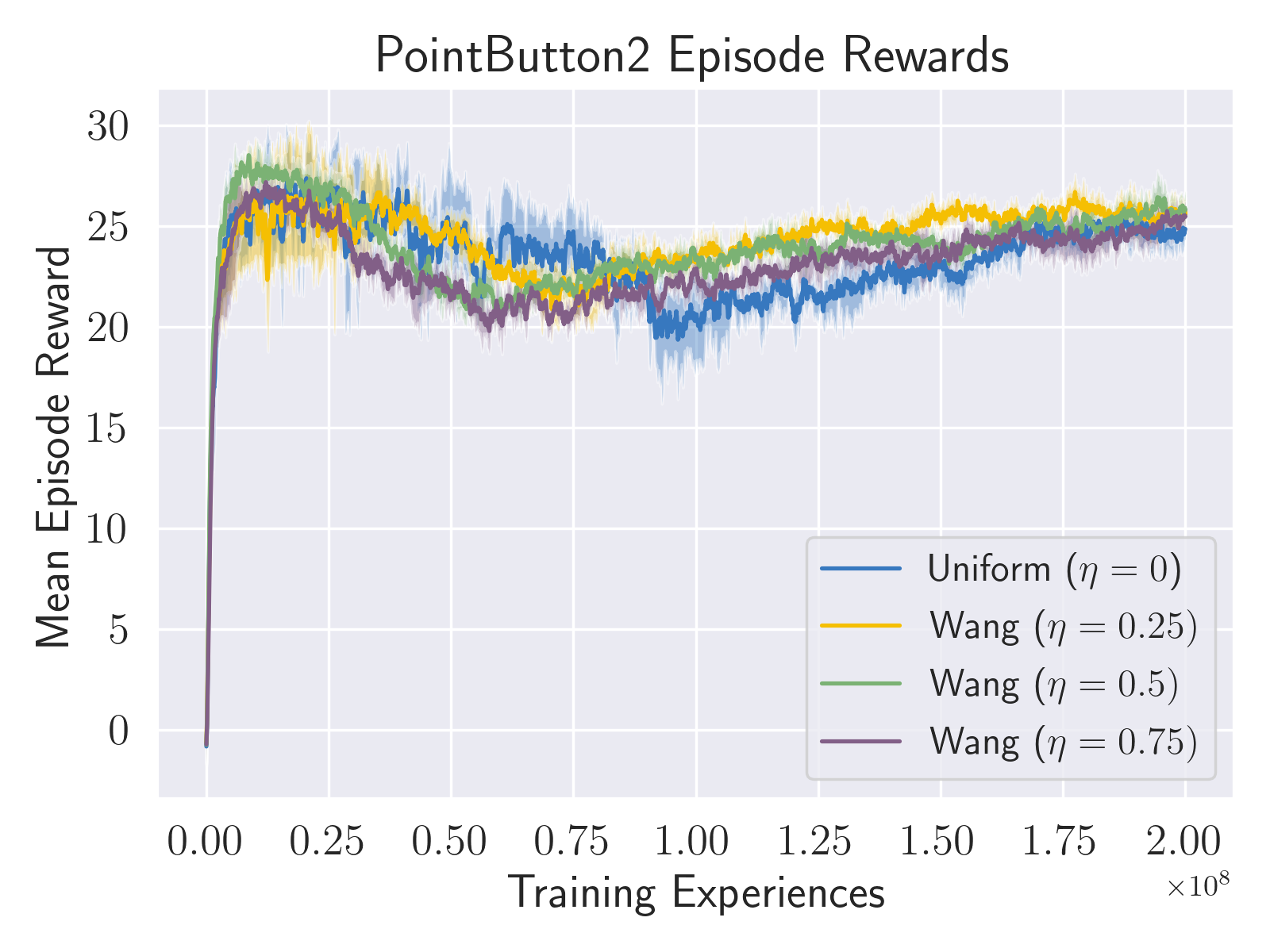}
    \includegraphics[width=0.234\textwidth]{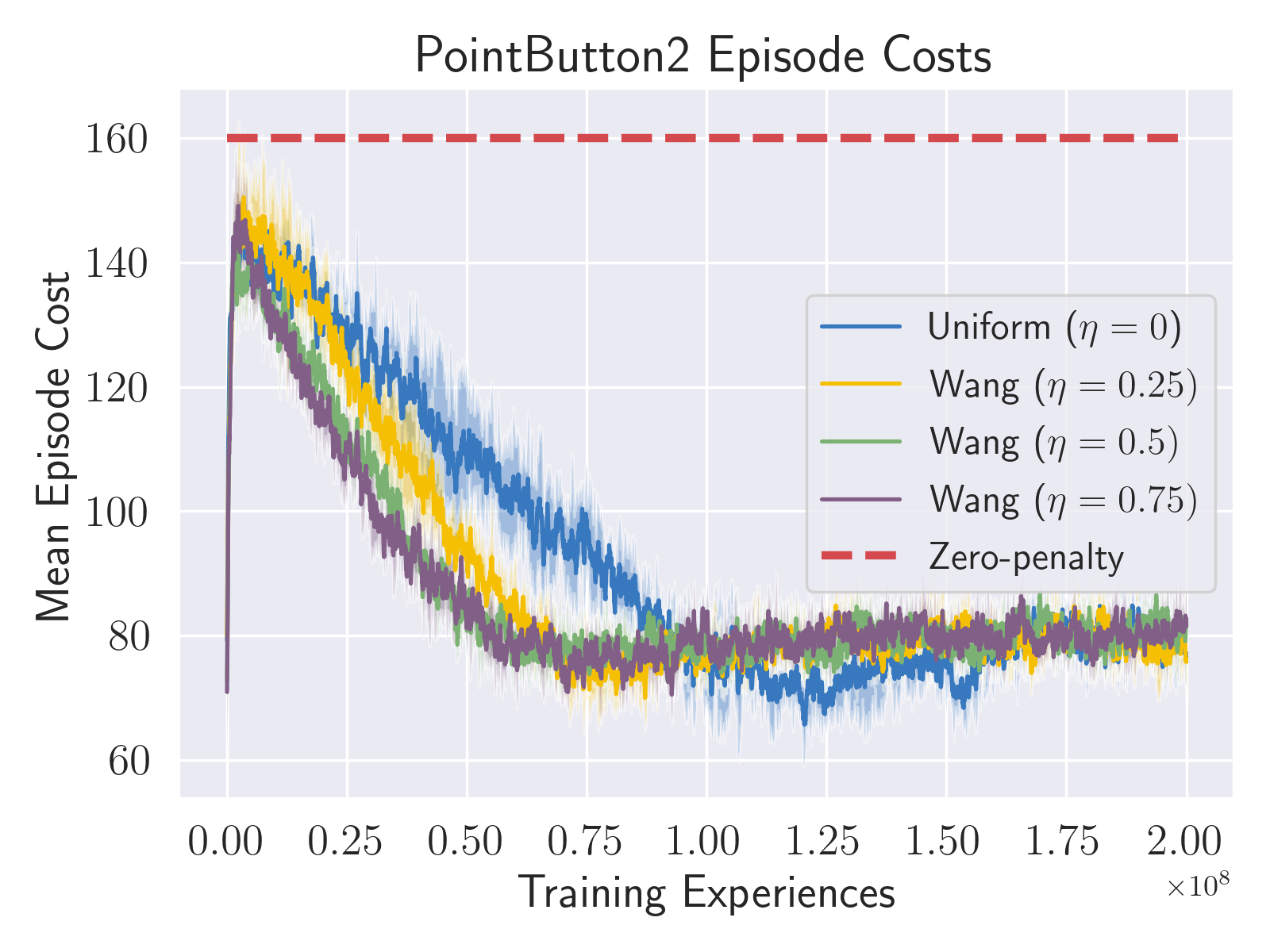}
     \caption{Effect of tuning $\eta$ on constrained learning in one environment (PointButton2).  Top row: policy distributions become wider with increasing $\eta$, while entropy considering control bounds depends on the locations of the control actions of the optimized policies.  Bottom row: episode rewards are similar over a fairly broad range of $\eta > 0$, with cost reaching the target level more rapidly with increasing $\eta$ in that range.}
    \label{constr_eta2}
\end{figure}

\subsection{A.10 Pilot Application to Multi-Task Learning}\label{meta_app}
Here we provide initial results of a pilot application of our unconstrained risk-sensitive approach to multi-task learning.  Therein, the set of tasks were considered as one large, ``joint'' MDP and our approach was applied without further modification.  Episodes collected from different tasks were ranked in terms of their total reward for the purpose of applying weights; no normalization across tasks was used.  In this experiment, we applied the ``pessimistic'' Wang weighting with $\eta=0.75$.  Differing from the Safety Gym configuration, variance was taken to be a function of state and shared computation layers with the network for policy mean.

Performance on the Multi-Task 10 (MT10) benchmark is given in Figure \ref{metaworld_fig}.  MT10 includes 50 flavors of each of 10 task families (i.e. 500 distinct MDPs).  We observed significant performance gains relative to standard, multi-task PPO.  The final observed success rate of the risk-sensitive approach was found to be comparable to that of multi-task Soft Actor-Critic \cite{metaworld}.

While further experimentation and validation are required, this result does indicate potential for the use of our method in the context of multi-task and meta-learning.  In particular, one may evaluate its ability to positively impact the distribution of outcomes over different tasks and episodes.

\begin{figure}
    \includegraphics[width=0.234\textwidth]{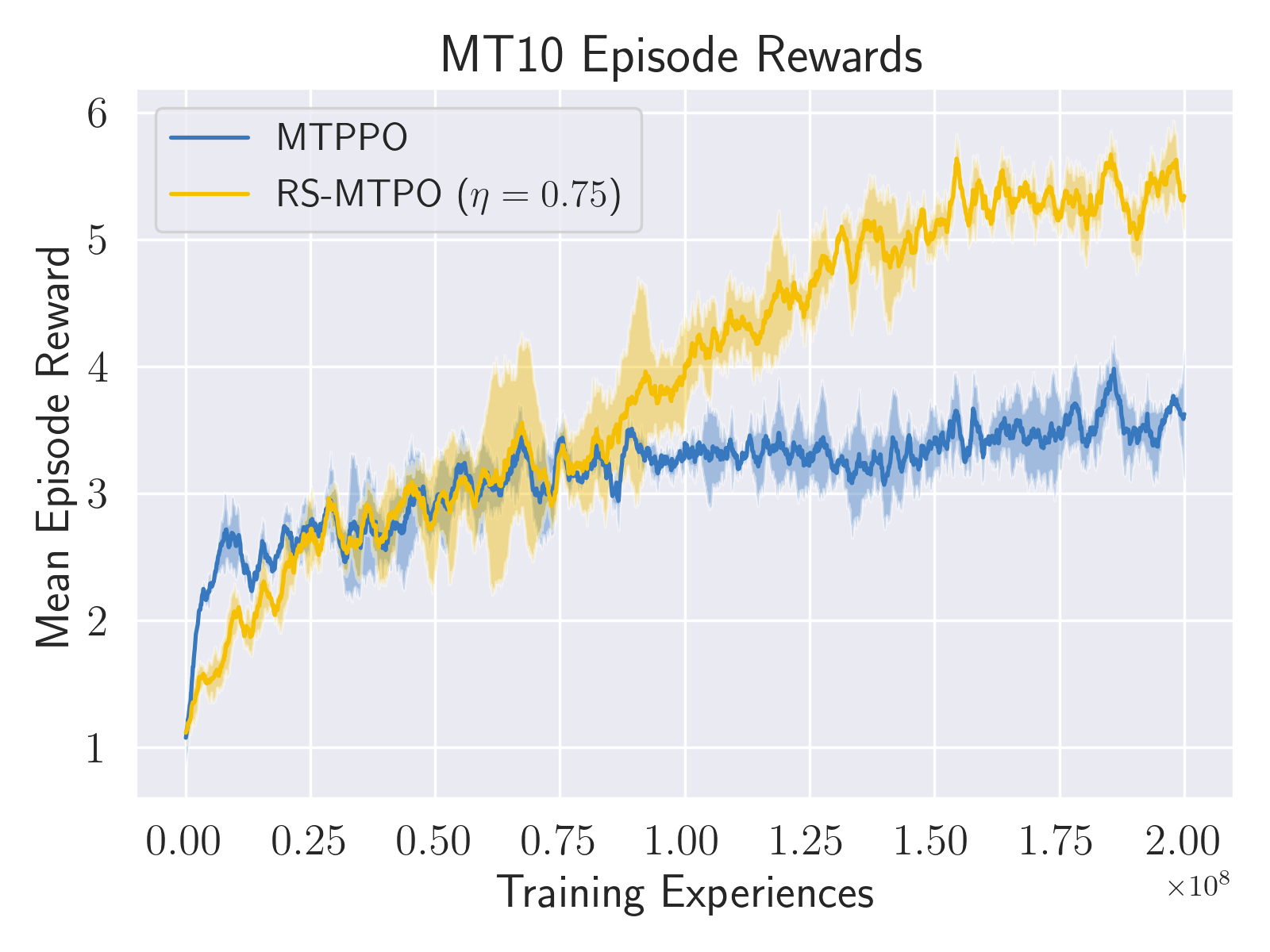}
    \includegraphics[width=0.234\textwidth]{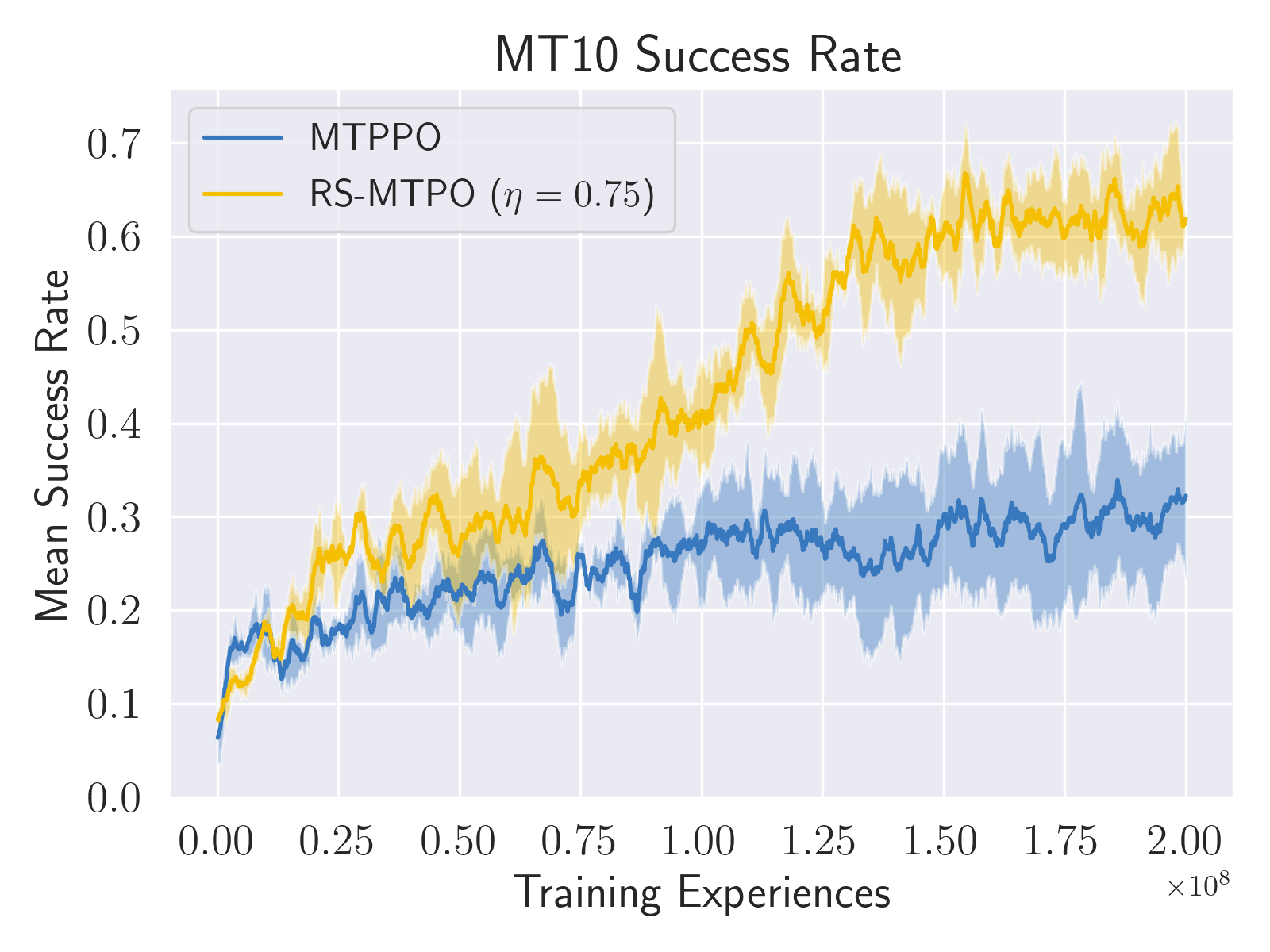}
     \caption{Pilot analysis of our unconstrained risk-sensitive approach (yellow) on MT10 compared to standard multi-task PPO (blue).  The success plot (right) may be compared to Figure 15 of \cite{metaworld}.  Results were generated using 3 random seeds.}
    \label{metaworld_fig}
\end{figure}

\subsection{Source Code}
The code and configurations used to produce these results are posted at https://github.com/JHU-APL-ISC-Deep-RL/risk-sensitive.

\end{document}